\documentclass[twoside,11pt]{article}

\usepackage{blindtext}

\usepackage{thm-restate}
\usepackage{jmlr2e}
\usepackage{graphicx}
\usepackage{wrapfig}
\usepackage{xcolor}
\usepackage{color, colortbl}	
\definecolor{Gray}{gray}{0.9}
\usepackage[ruled,vlined,linesnumbered]{algorithm2e}
\usepackage{amsmath}
\usepackage{amssymb}
\usepackage{commath}
\usepackage[font=small]{caption}
\usepackage{subcaption}
\usepackage{enumerate}
\usepackage{pifont}
\usepackage{amsfonts}
\usepackage{multirow}
\usepackage{makecell}
\usepackage{booktabs} 
\usepackage{dsfont}
\usepackage{bm}

\usepackage{bbding}
\usepackage{siunitx}
\usepackage{tcolorbox}
\usepackage{tikz}
\newcommand{\tikzcircle}[2][cyan,fill=cyan]{\tikz[baseline=-0.7ex]\draw[#1,radius=#2] (0,0) circle ;}

\newcommand{\tikzline}[2][cyan, fill=cyan]{\tikz[baseline=-0.7ex]\draw[-,color=#2,thick] (0,0) -- (0.3,0);}
\SetKwRepeat{Do}{do}{while}

\newcommand{\nn}{f}
\newcommand{\nnunder}{\underline{f}}
\newcommand{\indim}{d}
\newcommand{\outdim}{m}
\newcommand{\inpoint}{x}
\newcommand{\outpoint}{y}

\newcommand{\indomain}{\mathcal{C}}

\newcommand{\intervar}{z^{(i)}_{j}}

\newcommand{\spec}{g}
\newcommand{\net}{f}
\newcommand{\specnum}{K}

\newcommand{\preimage}{\nn^{-1}}
\newcommand{\polytopeset}{\mathcal{T}}
\newcommand{\polytope}{T}
\newcommand{\polytopenum}{D}

\newcommand{\threshold}{v}
\newcommand{\inset}{I}
\newcommand{\outset}{O}
\newcommand{\proportion}{p}

\newcommand{\numsamples}{N}
\newcommand{\nnslope}{\bm{\alpha}}
\newcommand{\nnslopesingle}{\alpha}

\newcommand{\weight}{\textbf{W}}
\newcommand{\bias}{\textbf{b}}
\newcommand{\preact}{h}
\newcommand{\postact}{a}
\newcommand{\numlayers}{L}
\newcommand{\relu}{\textnormal{ReLU}}

\newcommand{\lowerweight}{\underline{\textbf{A}}}
\newcommand{\lowerbias}{\underline{\textbf{b}}}
\newcommand{\lowerweightsingle}{\underline{a}}
\newcommand{\lowerbiassingle}{\underline{b}}
\newcommand{\upperweight}{\overline{\textbf{A}}}
\newcommand{\upperbias}{\overline{\textbf{b}}}
\newcommand{\upperweightsingle}{\overline{a}}
\newcommand{\upperbiassingle}{\overline{b}}
\newcommand{\concretelower}{\textbf{l}}
\newcommand{\concreteupper}{\textbf{u}}
\newcommand{\concretelowersingle}{l}
\newcommand{\concreteuppersingle}{u}

\newcommand{\featurelowerbound}{\underline{x}}
\newcommand{\featureupperbound}{\overline{x}}

\newcommand{\actneuron}{\hat{\textbf{z}}}
\newcommand{\preactneuron}{\textbf{z}}
\newcommand{\dualvar}{\beta^{(i)}_{j}}
\newcommand{\nndual}{\bm{\beta}}

\newcommand{\lincon}{\psi}
\newcommand{\boxcon}{\phi}
\newcommand{\fol}{A}
\newcommand{\folunder}{\alpha}
\newcommand{\folover}{\alpha}

\newcommand{\volume}{\textnormal{vol}}
\newcommand{\cov}{\textnormal{cov}}

\newcommand{\algo}{QV}

\newcommand{\linconw}{c}
\newcommand{\linconb}{d}

\newcommand{\maxiterations}{R}
\newcommand{\targetcov}{r}

\newcommand{\rev}[1]{{\color{black}{#1}}} 
 
\newcommand{\xy}[1]{{\color{black}{#1}}} 
\newcommand{\bwc}[1]{{\color{black} (Benjie) {#1}}}
\newcommand{\bw}[1]{{\color{black} {#1}}}

\usepackage{lastpage}
\jmlrheading{25}{2024}{1-\pageref{LastPage}}{7/24; Revised X/24}{X/24}{21-0000}{Xiyue Zhang, Benjie Wang, Marta Kwiatkowska and Huan Zhang}

\ShortHeadings{Preimage Approximation for Neural Networks}{Zhang, Wang, Kwiatkowska and Zhang}
\firstpageno{1}

\begin{document}

\title{PREMAP: A Unifying PREiMage APproximation Framework for Neural Networks}

\author{\name Xiyue Zhang\thanks{Equal contribution} \footnotemark[2] \email xiyue.zhang@cs.ox.ac.uk \\
        \name Benjie Wang\footnotemark[1] \email benjie.wang@cs.ox.ac.uk \\
        \name Marta Kwiatkowska \email marta.kwiatkowska@cs.ox.ac.uk \\
       \addr Department of Computer Science\\
       University of Oxford\\
       Oxford, OX1 3QD, UK
       \AND
       \name Huan Zhang \email huan@huan-zhang.com \\
       \addr Department of Electrical and Computer Engineering\\
       University of Illinois Urbana–Champaign\\
       Urbana, IL 61801, USA}

\editor{My editor}

\maketitle

\footnotetext[2]{Current address: School of Computer Science, University of Bristol, UK (Correspondence to xiyue.zhang@bristol.ac.uk)}
\begin{abstract}
Most methods for neural network verification focus on bounding the \emph{image}, i.e., set of outputs for a given input set. This can be used to, for example, check the robustness of neural network predictions to bounded perturbations of an input.
However, verifying properties concerning the \emph{preimage}, i.e., the set of inputs satisfying an output property, requires abstractions in the input space.
We present a general framework for preimage abstraction that produces under- and over-approximations of any polyhedral output set. Our framework employs cheap parameterised linear relaxations of the neural network, together with an anytime refinement procedure that iteratively partitions the input region by splitting on input features and neurons. 
The effectiveness of our approach relies on carefully designed heuristics and optimization objectives to achieve rapid improvements in the approximation volume. 
We evaluate our method on a range of tasks, demonstrating significant improvement in efficiency 
and scalability to high-input-dimensional image classification tasks compared to state-of-the-art techniques.
Further, we showcase the application to quantitative verification and robustness analysis, presenting a sound and complete algorithm for the former and providing sound quantitative results for the latter. 
\end{abstract}


\section{Introduction}
\label{sec:introduction}
Despite the remarkable empirical success of neural networks, ensuring their safety against potentially adversarial behaviour, especially when using them as decision-making components in autonomous systems~\citep{Bojarski16AutoControl,Codevilla18,Yun17Robotics}, is an important and challenging task. 
Towards this aim, various approaches have been developed for the verification of neural networks, with extensive effort devoted, in particular, to the problem of \emph{local robustness verification}\bw{, which focuses on deciding  
the presence or absence of adversarial examples~\citep{szegedy2013intriguing,Biggio13} within an $\epsilon$-perturbation neighbourhood} \citep{huang2017safety,katz2017reluplex,zhang2018crown,bunel2018unified,tjeng2019evaluating,singh2019deeppoly,xu2020automated,xu2021fast,wang2021beta}.

While local robustness verification is useful for certifying that a neural network has the same prediction in a neighbourhood of an input, 
it does not provide finer-grained information on the behaviour of the network in the input domain. 
An alternative and more general approach for neural network analysis is to construct the \emph{preimage} abstraction of its predictions \citep{Matoba20Exact,Dathathri19Inverse}. 
Given a set of outputs, the preimage is defined as the set of all inputs mapped by the neural network to that output set. \bw{For example, given a particular action for a neural network controller (e.g., drive left), the preimage captures the set of percepts (e.g., car positions) that cause the neural network to take this action.} By characterising the preimage 
symbolically in an abstract representation, e.g., polyhedra, one can perform more complex analysis for a wider class of properties beyond local robustness, such as computing the \emph{proportion} of inputs satisfying a property 
\citep{webb2018statistical,mangal2019probabilistic}, \bw{or performing downstream reasoning tasks.} 

Unfortunately, exact preimage generation \citep{Matoba20Exact} is intractable at scale, as it requires splitting into input subregions where the neural network is linear. Each such subregion corresponds to a set of determined activation patterns of the nonlinear neurons, the number of which grows exponentially with the number of neurons. Therefore, we focus on 
the problem of \emph{preimage approximation}, that is, constructing symbolic abstractions for the preimage. 
In this work, we propose PREMAP, a general framework for preimage approximation that computes under-approximations and over-approximations represented as disjoint unions of polytopes (DUP). 

Our method leverages recent progress in local robustness verification, which uses parameterised linear relaxations of neural networks together with divide-and-conquer refinement strategies to analyse the input space in an efficient and GPU-friendly manner \citep{zhang2018crown,wang2021beta}. We observe that, unlike robustness verification, where the goal is to determine the behaviour of the neural network at the \emph{worst-case} point in the input space (and thus verify or falsify the property), in preimage approximation we instead aim to 
minimize the \textit{overall} difference in volume between the approximation and the (intractable) exact preimage.
Thus, we design a methodology that focuses on effectively optimising this new volume-based objective, while maintaining the GPU parallelism, efficiency, and flexibility drawn from the state-of-the-art robustness verifiers.

\xy{
In more detail, this paper \bw{makes} the following novel contributions:
\begin{enumerate}
\item \bw{the first} unifying framework \bw{capable of efficiently} generating symbolic under- and over-approximations of the preimage abstraction of any polyhedron output set;
\item an \bw{efficient and anytime} \bw{preimage} refinement algorithm, \bw{which iteratively partitions the input region into subregions using input and/or intermediate (ReLU) splitting (hyper)planes;}
\item \bw{carefully-designed heuristics for selecting input features and neurons to split on, which (i) take advantage of GPU parallelism for efficient evaluation; and (ii) significantly improve approximation quality compared to na\"ive baselines;}
\item \bw{a novel differentiable optimisation objective for improving preimage approximation precision, with respect to (i) convex bounding parameters of nonlinear neurons and (ii) Lagrange multipliers for neuron splitting constraints;}
\item \bw{empirical evaluation of preimage approximation on a range of datasets, and an application to the problem of quantitative verification;}

\item \bw{a publicly-available software implementation of our preimage approximation framework~\citep{premap2025}.}
\end{enumerate}

This work significantly \bw{extends} the preliminary version in \cite{zhang23preimage} 
\bw{in the following ways:}
(i) 
\bw{introducing an \emph{over-approximation} algorithm within the framework, with accompanying empirical results;}
(ii) \bw{improving the refinement procedure through new heuristics for selecting input features to split on,}
using only 49.7\% (avg.) computation time of the prior method to achieve the same precision \bw{(Sections \ref{sec:branching},  \ref{sec:smoothed_input_splitting_exp})}; 
(iii) 
\bw{introducing Lagrangian relaxation to enforce neuron splitting constraints,} enabling further optimisation of the approximations with precision gains of up to 58.6\% (avg.) \bw{for a MNIST preimage approximation task (Sections \ref{sec:dual}, \ref{sec:relu_split_exp})}; 
and (iv) \bw{an extended empirical evaluation of the framework.}
} 


\rev{The paper is organized as follows. 
We present related works in Section \ref{sec:related}.
Section \ref{sec:pre} introduces the notation and preliminary definitions of neural networks,  linear relaxation  
and polyhedra representations.}
In Section \ref{sec:problem_definition}, we present the formulation of the problems studied, namely preimage approximation and quantitative analysis of neural networks. Our preimage approximation method is provided in Section \ref{sec:method}, together with the application to quantitative verification of neural networks and proofs of soundness and completeness. 
In Section \ref{sec:evaluation}, we present the experimental evaluation of our approach and demonstrate its effectiveness and scalability compared to the state-of-the-art techniques, and applications in quantitative verification and robustness analysis. 
We conclude the paper in Section \ref{sec:conclusion}.


\section{Related Work}
\label{sec:related}
Our paper is related to a series of works on 
robustness verification of neural networks.
To address the scalability issues with
\textit{complete} verifiers \citep{huang2017safety,katz2017reluplex,tjeng2019evaluating} based on constraint solving,
convex relaxation \citep{salman2019convex} has been used for developing highly efficient \textit{incomplete} verification methods \citep{zhang2018crown,wong2018provable,singh2019deeppoly,xu2020automated}.
Later works employed the branch-and-bound (BaB) framework \citep{bunel2018unified,Bunel20BaB} to achieve completeness, using incomplete methods for the bounding procedure \citep{xu2021fast,wang2021beta,Ferrari22MultiBaB}. In this work, we adapt convex relaxation for efficient preimage approximation. Further, our divide-and-conquer procedure is analogous to BaB, \xy{but focuses on maximising covered volume for under-approximation (resp. minimising for over-approximation) rather than maximising or minimising a function value}.

There are also works
that have sought to define a weaker notion of local robustness known as \textit{statistical robustness} \citep{webb2018statistical,mangal2019probabilistic,wang2021statistically}, which requires that a proportion of points under some perturbation distribution around an input point are classified in the same way. 
Verification of statistical robustness is typically achieved by sampling and statistical guarantees \citep{webb2018statistical,baluta2021quantitative,tit2021corruption,yang2021quantitative}.
In this paper, we apply our symbolic approximation approach to quantitative analysis of neural networks, while providing \textit{exact quantitative} rather than 
\textit{statistical} evaluation \citep{Webb19Stat}. \rev{In particular, similarly to \citet{xiang2020reachable,rober2023backward}, while we employ sampling in order to guide our divide-and-conquer procedure, the guarantees obtained are exact.}
 
Another line of related works considers deriving exact or approximate abstractions of neural networks,
 which are applied for explanation~\citep{sotoudeh2021syrenn}, verification~\citep{elboher2020abstraction,pulina2010abstraction}, reachability analysis~\citep{prabhakar2019abstraction}, and preimage approximation~\citep{Dathathri19Inverse,ProveBound}.
\cite{Dathathri19Inverse} leverages symbolic interpolants \citep{Albarghouthi13interpolants} for preimage approximations, facing exponential complexity in the number of hidden neurons. 
\cite{ProveBound} considers the preimage over-approximation problem via inverse bound propagation, but their approach cannot be directly extended to the under-approximation setting. They also do not consider any strategic branching and refinement methodologies like those in our unified framework.
Our anytime algorithm, which combines convex relaxation with principled splitting strategies for refinement, is applicable for both under- and over-approximations. 

\rev{
In the context of analysis of systems with neural network controllers, the backward reachability problem is to compute the set of states for which a system's trajectories can reach a particular target region within a finite time horizon. Prior works have explored both exact computation of this set \citep{vincent2021reachable} as well as guaranteed over-approximation \citep{rober2022backward,rober2023backward,zhang2023backward,zhang2024hybrid,ProveBound}. Empirically, when applied to neural network controllers, we find that our approach performs competitively with the state-of-the-art method of \citet{ProveBound}. 
}

\section{Preliminaries}
\label{sec:pre}
We use 
$\net: \mathbb{R}^{\indim} \to \mathbb{R}^{\outdim}$
to denote a feed-forward neural network. 
\xy{For layer $i$, we use $\weight^{(i)}$ to denote the weight matrix, $\bias^{(i)}$ 
the bias, 
$\preactneuron^{(i)}$ the pre-activation neurons, and $\actneuron^{(i)}$ 
the post-activation neurons, such that we have $\preactneuron^{(i)} = \weight^{(i)} \actneuron^{(i-1)} + \bias^{(i)}$. 
We use $\preact^{(i)}(x)$ to denote the function from input to pre-activation neurons, and $\postact^{(i)}(x)$ the function from input to 
the post-activation neurons, i.e.,  $\preactneuron^{(i)}=\preact^{(i)}(x)$ and $\actneuron^{(i)}=\postact^{(i)}(x)$.} 
In this paper, we focus on ReLU neural networks with $\postact^{(i)}(\inpoint)=\relu(\preact^{(i)}(\inpoint))$, 
where $\relu(\preact) := \max(\preact, 0)$ is applied element-wise.
However, our method can be generalized to other activation functions \bw{that can be} bounded by linear 
\bw{functions},
\bw{similarly to}~\cite{zhang2018crown}.

\textbf{Linear Relaxation of Neural Networks.}
Nonlinear activation functions lead to the NP-completeness of the neural network verification problem \bw{as} proved in \cite{katz2017reluplex}. 
To address such intractability, linear relaxation is often used to transform the nonconvex constraints into linear programs. 
As shown in Figure~\ref{fig:linear_relaxation}, given  \textit{concrete}  lower and upper bounds $\concretelower^{(i)}\leq \preact^{(i)}(\inpoint) \leq \concreteupper^{(i)}$ on the pre-activation values of layer $i$, there are three cases to consider.
\rev{In the \emph{inactive} ($\concreteuppersingle^{(i)}_j \leq 0$) and \emph{active} ($\concretelowersingle^{(i)}_j \ge 0)$ cases, 
the mapping function from pre-activation to post-activation values becomes linear, with
$\postact^{(i)}_j(\inpoint) = 0$ and $\postact^{(i)}_j(\inpoint) = \preact^{(i)}_j(\inpoint)$ respectively.}
In the \emph{unstable} case, $\postact^{(i)}_j(\inpoint)$ can be bounded by 
\begin{equation}\label{eq:node_bounding}
\nnslopesingle^{(i)}_j \preact^{(i)}_j(\inpoint) \leq \postact^{(i)}_j(\inpoint) \leq 
-\frac{\concreteuppersingle^{(i)}_j\concretelowersingle^{(i)}_j}{\concreteuppersingle^{(i)}_j - \concretelowersingle^{(i)}_j}
    + \frac{\concreteuppersingle^{(i)}_j}{\concreteuppersingle^{(i )}_j - \concretelowersingle^{(i )}_j} \preact^{(i)}_j(\inpoint)
\end{equation}
where $\alpha^{(i)}_j$ is a configurable parameter that produces a valid lower bound for any value in 
$[0,1]$. 
Linear bounds can also be obtained for other non-piecewise linear activation functions by considering the characteristics of the activation function, such as the S-shape activation functions \citep{zhang2018crown, Matt24Sigmoid}.

\begin{figure}[t]
     \centering
     \begin{subfigure}[b]{0.3\textwidth}
         \centering
         \includegraphics[width=\textwidth]{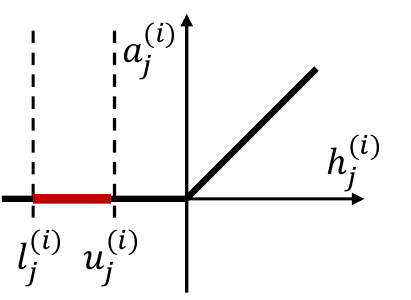}
         \label{fig:negative}
     \end{subfigure}
     \quad
     \begin{subfigure}[b]{0.3\textwidth}
         \centering
         \includegraphics[width=\textwidth]{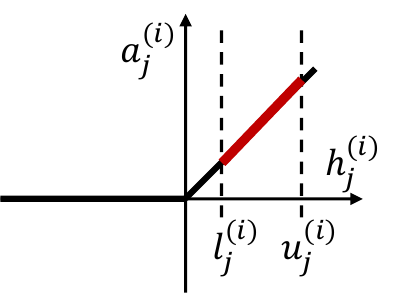}
         \label{fig:positive}
     \end{subfigure}
     \quad
     \begin{subfigure}[b]{0.3\textwidth}
         \centering
         \includegraphics[width=\textwidth]{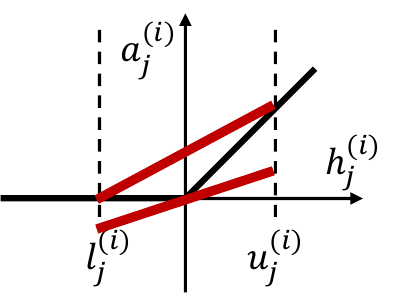}
         \label{fig:unstable}
     \end{subfigure}
        \caption{Linear bounding functions for inactive, active, unstable ReLU neurons.}
        \label{fig:linear_relaxation}
\end{figure}

Linear relaxation can be used to compute linear lower and upper bounds of the form $\lowerweight \inpoint + \lowerbias \leq f(\inpoint) \leq \upperweight \inpoint + \upperbias$ on the output of a neural network, for a given bounded input region $\indomain$.
These methods are known as linear relaxation based perturbation analysis (LiRPA) algorithms \citep{xu2020automated,xu2021fast,singh2019deeppoly}. In particular,
\emph{backward-mode} LiRPA computes linear bounds on 
$f$
by propagating linear 
bounding functions 
backward from the output, layer by layer, to the input layer.

\textbf{Polytope Representations.}
Given an Euclidean space $\mathbb{R}^{\indim}$, \rev{a polyhedron $\polytope$ is defined to be the intersection of a finite number of half spaces}. More formally, suppose we have a set of linear constraints defined by $\lincon_i(\inpoint) := \linconw_i^T \inpoint + \linconb_i \geq 0$ for $i = 1, ... \specnum$, where $\linconw_i \in \mathbb{R}^{\indim}, \linconb_i \in \mathbb{R}$ are constants, and $\inpoint = (\inpoint_1, ..., \inpoint_\indim)$ is a tuple of variables. Then a polyhedron is defined as $\polytope = \{\inpoint \in \mathbb{R}^{\indim}| \bigwedge_{i = 1}^{\specnum} \lincon_i(\inpoint) \}$,
where $\polytope$ consists of all values of $\inpoint$ satisfying the first-order logic (FOL) formula $\folunder(\inpoint) := \bigwedge_{i = 1}^{\specnum} \lincon_i(\inpoint)$.
We use the term polytope to refer to a bounded polyhedron, that is, a polyhedron $\polytope$ such that $\exists R \in \mathbb{R}^{> 0} : \forall \inpoint_1, \inpoint_2 \in \polytope$, $\norm{\inpoint_1 - \inpoint_2}_2 \le R$ holds. 
The abstract domain of polyhedra has been widely used for the verification of neural networks and computer programs as in \cite{singh2019deeppoly,Benoy02Polyhedral,Boutonnet19Polyhedron}. 
An important type of polytope is the hyperrectangle (box), which is a polytope defined by a closed and bounded interval $[\underline{\inpoint_i}, \overline{\inpoint_i}]$ for each dimension, where $\underline{\inpoint_i}, \overline{\inpoint_i} \in \mathbb{Q}$. More formally, using the linear constraints $\boxcon_i := (\inpoint_i \geq \underline{\inpoint_i}) \wedge (\inpoint_i \leq \overline{\inpoint_i})$ for each dimension, \rev{the hyperrectangle takes the form $\indomain = \{\inpoint \in \mathbb{R}^{\indim} | \inpoint \models \bigwedge_{i = 1}^{\indim} \boxcon_i \}$, where $ \inpoint \models \bigwedge_{i = 1}^{\indim} \boxcon_i$ denotes that the input $x$ satisfies the constraints specified by the conjunction of inequalities}.

\begin{figure}
\begin{subfigure}[b]{0.60\linewidth}
\begin{tikzpicture}
\fill[cyan] (1.5, 1.5) ellipse (75pt and 60pt); 
    \draw[] (0, 0) rectangle (3,3) ;
    \node[] at (3.4, 3.2) {$\mathcal{C}$};
    
    \fill[cyan] (5.5,0.5) ellipse (25pt and 25pt); 
    \draw[] (5, 0) rectangle (8,3) ;
    \node[] at (8.4, 3.2) {$\mathcal{C}$};
    
\end{tikzpicture}
\caption{Standard NN Verification: decide whether $f(x) \in \outset, \forall x \in \mathcal{C}$.}
\end{subfigure}
\hfill
\begin{subfigure}[b]{0.34\linewidth}
\centering
\begin{tikzpicture}
\fill[cyan] (0.5, 0.5) ellipse (25pt and 25pt); 
    \draw[] (0, 0) rectangle (3,3) ;
    \node[] at (3.4, 3.2) {$\mathcal{C}$};
    \draw[thick,color=blue] (0, 1.2) -- (1.2, 0);
    \draw[thick,color=red] (0,1.9) -- (2.8,0);
    \draw[thick,color=blue] (0, 0) -- (0, 1.2);
    \draw[thick, color=blue] (0, 0) -- (1.2, 0);
    \draw[thick,color=red] (0, 1.2) -- (0, 1.9);
    \draw[thick, color=red] (1.2, 0) -- (2.8, 0);
\end{tikzpicture}
\caption{NN Preimage Approximation: characterize $f^{-1}(\outset) := \{x \in \mathcal{C} \mid f(x) \in \outset \}$.}
\end{subfigure}
\caption{Illustration of the preimage approximation problem. In contrast to NN robustness verification, where the goal is to answer \texttt{Yes} or \texttt{No} for the statement $f(x) \in \outset, \forall x \in \mathcal{C}$, in preimage approximation the goal is to find a bounding under-approximation \tikzline{blue} and over-approximation \tikzline{red} to the preimage $f^{-1}(\outset)$. \tikzcircle{4pt} indicates the region where $f(x) \in \outset$. 
}\label{fig:preimage_problem}
\end{figure}

\section{Problem Formulation}
\label{sec:problem_definition}

\subsection{Preimage Approximation}

In this work, we are interested in the problem of computing preimages for neural networks. Given a subset $\outset \subset \mathbb{R}^{\outdim}$ of the codomain, the preimage of a function $\nn: \mathbb{R}^{\indim} \to \mathbb{R}^{\outdim}$ is defined to be the set of all inputs $\inpoint \in \mathbb{R}^{\indim}$ that are mapped to an element of $\outset$ by $\nn$. For neural networks in particular, the input is typically restricted to some bounded input region $\indomain \subset \mathbb{R}^{\indim}$. In this work, we restrict the output set $\outset$ to be a polyhedron, and the input set $\indomain$ to be an axis-aligned hyperrectangle region $\indomain \subset \mathbb{R}^{\indim}$, as these are commonly used in neural network verification. 

\rev{As illustrated in Figure~\ref{fig:preimage_problem}, we contrast the preimage approximation problem with the setting of neural network robustness verification, where the goal is to answer a binary question, i.e., whether $f(x) \in \outset$ holds for all $x \in \mathcal{C}$.
In preimage approximation, the objective is to compute the explicit characterisation of the preimage $f^{-1}(\outset)$ for the targeted output set $\outset$.
In particular, we aim to compute bounding under- and over-approximations} (depicted by \tikzline{blue} and \tikzline{red}, respectively) \rev{of the true preimage} (illustrated by \tikzcircle{4pt}) \rev{within the input domain.} 
We now define the notion of a restricted preimage.

\begin{definition}[Restricted Preimage]
Given a neural network $\nn: \mathbb{R}^{\indim} \to \mathbb{R}^{\outdim}$, and an input set $\indomain \subset \mathbb{R}^{\indim}$, the restricted preimage of an output set $\outset \subset \mathbb{R}^{\outdim}$ is defined to be the set $\nn^{-1}_{\indomain}(\outset) := \{\inpoint \in \mathbb{R}^{\indim}| \nn(\inpoint) \in \outset \wedge \inpoint \in \indomain\}$.
\end{definition}

\begin{example}\label{eg:formulation}
To illustrate our problem formulation and approach,
we introduce a vehicle parking task from \cite{Ayala11vehicle} as a running example.  
In this task, there are four parking lots, located in each quadrant of a $2\times 2$ grid $[0,2]^2$, and a neural network with two hidden layers of 10 ReLU neurons $\nn: \mathbb{R}^2 \to \mathbb{R}^4$ is trained to classify which parking lot an input point belongs to.
To analyze the behaviour of the neural network in the input region \xy{$[0, 2] \times [0, 2]$, we set $\indomain = \{\inpoint \in \mathbb{R}^2 | (0 \leq \inpoint_1 \leq 2) \wedge (0 \leq \inpoint_2 \leq 2)\}$. Then the restricted preimage $\nn^{-1}_{\indomain}(\outset)$ of the set $\outset = \{\outpoint \in \mathbb{R}^4 | \bigwedge_{i \in \{2, 3, 4\}} \outpoint_1 - \outpoint_i \geq 0\} $ is the subspace of the region $[0, 2] \times [0, 2]$ that is \textit{labelled} as parking lot $1$ by the neural network.}
\end{example}

We focus on \emph{provable} approximations of the preimage. Given a first-order formula $\fol$, $\folunder$ is an \emph{under-approximation} (resp. \emph{over-approximation}) of $\fol$ if it holds that $\forall \inpoint. \folunder(\inpoint) \implies \fol(\inpoint)$ (resp. $\forall \inpoint. \fol(\inpoint) \implies \folover(\inpoint)$). 
In our context, the restricted preimage is defined by the formula $\fol(\inpoint) = (\nn(\inpoint) \in \outset) \wedge (\inpoint \in \indomain)$, and we restrict to approximations $\folunder$ that take the form of a disjoint union of polytopes (DUP). 
The goal of our method is to generate a DUP approximation $\polytopeset$ that is as tight as possible; that is, we aim to maximize the volume $\volume(\polytopeset)$ of an under-approximation, or minimize the volume $\volume(\polytopeset)$ of an over-approximation.

\begin{definition}[Disjoint Union of Polytopes]
A disjoint union of polytopes (DUP) is a FOL formula $\folunder$ of the form
$    \folunder(\inpoint) := \bigvee_{i = 1}^{\polytopenum} \folunder_i(\inpoint)$, 
where each $\folunder_i$ is a polytope formula 
(conjunction of a finite set of linear half-space constraints), with the property that $\folunder_i \wedge \folunder_j$ is unsatisfiable for any $i \neq j$.
\end{definition}

\subsection{Quantitative Properties} 

One of the most important verification problems for neural networks is that of proving guarantees on the output of a network for a given input set \citep{gehr2018ai2,gopinath2019properties,ruan2018reachability}. This is often expressed as a property of the form $(\inset, \outset)$ such that $\forall \inpoint \in \inset \implies \nn(\inpoint) \in \outset$. We can generalize this to \emph{quantitative} properties:

\begin{definition}[Quantitative Property]
    Given a neural network $\nn: \mathbb{R}^{\indim} \to \mathbb{R}^{\outdim}$, a measurable input set with non-zero measure (volume) $\inset \subseteq \mathbb{R}^{\indim}$, a measurable output set $\outset \subseteq \mathbb{R}^{\outdim}$, and a rational proportion $\proportion \in [0, 1]$, we say that the neural network satisfies the property $(\inset, \outset, \proportion)$ if $\frac{\volume(\nn^{-1}_{\inset}(\outset))}{\volume(\inset)} \geq \proportion$. 
\end{definition}
Note that the restricted preimage of a polyhedron under a neural network is Lebesgue measurable since polyhedra 
(intersection of a finite number of half-spaces) are Borel measurable and NNs are continuous functions.
\rev{
\begin{example}\label{eg:quant_property}
Consider the vehicle parking task, where the goal is to predict where to park among four parking lots. 
Consider the input region $\inset= \{x \in \mathbb{R}^{2}~|~x \in [0,1]^2\}$, representing the first parking lot region, and the output set  $\outset=\{\outpoint \in \mathbb{R}^{4}|\bigwedge_{i = 2}^{4} y_1 - \ y_i \geq 0\}$, which specifies the neural network predicting that the vehicle should park in the first lot, i.e., $y_1$ is the largest score among all decisions.
Let the quantitative proportion be $p=0.9$.
This defines a quantitative property $(\inset, \outset, \proportion)$ asserting that the volume of the generated preimage under-approximation $\nn^{-1}_{\inset}(\outset)$, from which the neural network maps to $\outset$, is at least 90\% of the total volume of $\inset$.
\end{example}
}

Neural network verification algorithms can be 
\bw{characterized by two main properties: \emph{soundness}},
which \bw{states that the algorithm} always returns correct results, and \emph{completeness}, \bw{which states that the algorithm always reaches}
a conclusion on any verification query~\citep{liu2021algorithms}.
We now define the soundness and completeness of verification algorithms for quantitative properties.

\begin{definition}[Soundness]
    A verification algorithm $\algo$ is sound if, whenever $\algo$ outputs \textnormal{True}, the property $(\inset, \outset, \proportion)$ holds.
\end{definition}

\begin{definition}[Completeness]
    A verification algorithm $\algo$\ is complete if (i) $\algo$ never returns \textnormal{Unknown}, and (ii) whenever $\algo$ outputs \textnormal{False}, the property $(\inset, \outset, \proportion)$ does not hold. 
\end{definition}

If the property $(\inset, \outset)$ holds, then the quantitative property $(\inset, \outset, 1)$ holds, while quantitative properties for $0 \leq  \proportion < 1$ provide more information when $(\inset, \outset)$ does \emph{not} hold. 
Most neural network verification methods produce approximations of the \emph{image} of $\inset$ in the output space, which cannot be used to verify quantitative properties. 
Preimage \textit{over-approximations} include 
points outside of the true preimage;
\bw{thus, they cannot be applied for sound quantitative verification}. 
In contrast, preimage \emph{under-approximations} provide a lower bound on the volume of the preimage, allowing us to soundly verify quantitative properties.

\section{Methodology}
\label{sec:method}

\begin{figure}[tbp]
\centering
\includegraphics[width=0.95\textwidth]
     {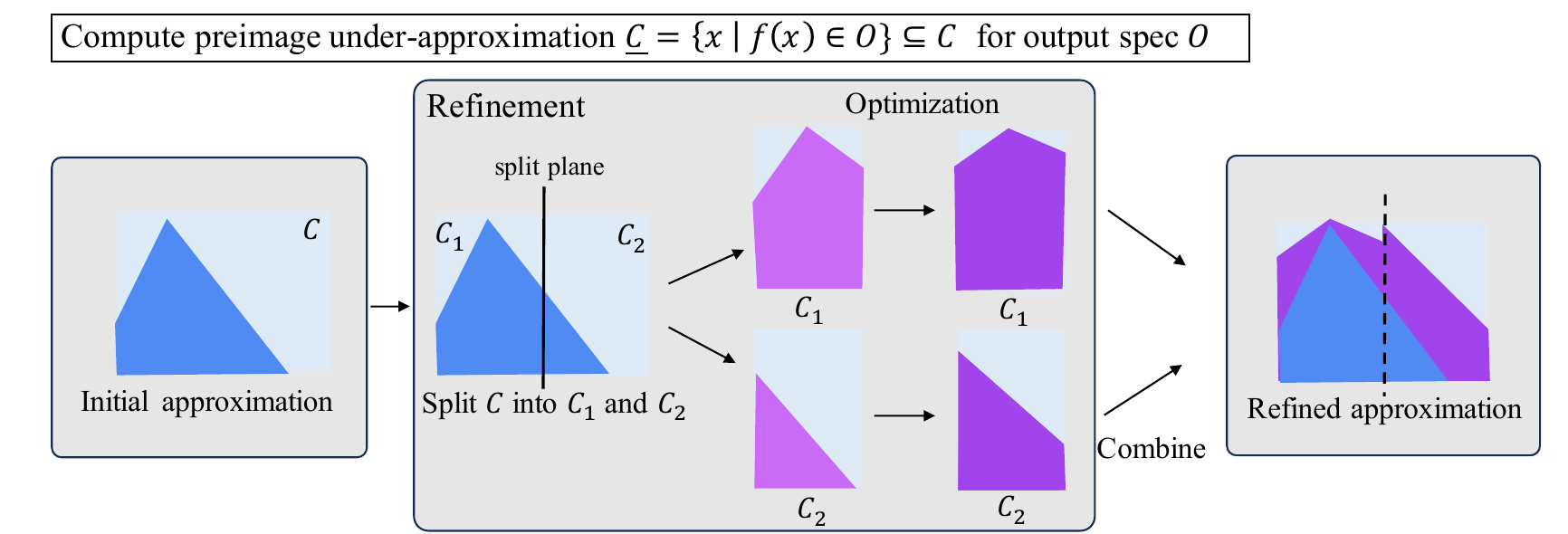}
\caption{\xy{Illustration of the} workflow for preimage \bw{under-approximation (shown in 2D for clarity)}.
\xy{Given a neural network $\net: \mathbb{R}^{\indim} \to \mathbb{R}^{\outdim}$ and output specification $\outset \subset \mathbb{R}^{\outdim}$,}
\bw{our algorithm generates an under-approximation $\underline{\indomain}$ to the preimage in the input region $\indomain$.}
Starting from the input region $\indomain$, the procedure repeatedly splits the selected region into smaller subregions $\mathcal{C}_1$ and $\mathcal{C}_2$ with tighter input bounds, and then optimises the bounding and Lagrangian parameters to increase the volume and thus improve the quality of the under-approximation. 
The refined under-approximation is combined into a union of polytopes.}
\label{fig:workflow}
\end{figure}

\subsection{Overview} 
In this section, we present the main components of our methodology.  
\xy{Figure~\ref{fig:workflow} shows the workflow of our preimage approximation method (using under-approximation as an illustration).
}

In Section \ref{sec:poly_gen}, we introduce how to cheaply and soundly \bw{under-approximate (or over-approximate)} the (restricted) preimage with a single polytope by means of the linear relaxation methods (Algorithm \ref{alg:genunderapprox}), which offer greater scalability than the exact method \citep{Matoba20Exact}.
\xy{
To handle the approximation loss caused by linear relaxation, in Section \ref{sec:branching} we propose an anytime refinement algorithm that improves the approximation by partitioning a (sub)region into subregions with splitting (hyper)planes, with each subregion then being approximated more accurately in parallel. 
In Section \ref{sec:local_opt}, we propose a novel differentiable objective to optimise the bounding parameters of linear relaxation to tighten the polytope approximation.
Next, in Section \ref{sec:dual}, we propose a refinement scheme based on intermediate ReLU splitting planes and derive a preimage optimisation method using Lagrangian relaxation of the splitting constraints.
The main contribution of this paper 
(Algorithm \ref{alg:main}) integrates these four components and is described in Section \ref{sec:overall}.}
Finally, in Section \ref{sec:verif}, we apply our method to quantitative verification (Algorithm \ref{alg:verify}) and prove its soundness and completeness.

To simplify the presentation,
we focus on  computing under-approximations and 
explain the necessary changes to compute over-approximations in highlight boxes throughout.

\begin{algorithm}[tb]
\caption{GenApprox}\label{alg:genunderapprox}
\KwIn{List of subregions $\indomain$, Output set $\outset$, Number of samples $\numsamples$, Boolean $Under$}
\KwOut{\bw{List of polytopes $\textbf{\polytope}$}}

$\textbf{\polytope} = []$\;
\For{\textnormal{subregion} $\indomain_{sub} \in \indomain$ \tcp{Parallel over subregions}}  {
$\inpoint_1, ..., \inpoint_{\numsamples} \leftarrow $ Sample($\indomain_{sub}, \numsamples$)\;
    \If{Under}{
    \bw{$[\underline{\spec_1}(\inpoint, \nnslope_1, \nndual_1), ..., \underline{\spec_{\specnum}}(\inpoint, \nnslope_{\specnum}, \nndual_{\specnum})] \leftarrow$ LinearLowerBound($\indomain_{sub}, \outset$)
    }
    \; \label{algline:linearbound_lower}
    
   Loss$(\nnslope_1, ..., \nnslope_{\specnum}, \nndual_1, ..., \nndual_{\specnum}) \leftarrow - \frac{\text{vol}(\indomain_{sub})}{\numsamples} \sum_{j = 1, ..., \numsamples} \sigma(-\textnormal{LSE}(-\underline{\spec_1}(\inpoint_j, \nnslope_1,\nndual_1), ..., -\underline{\spec_{\specnum}}(\inpoint_j, \nnslope_{\specnum},\nndual_{\specnum})))$\label{algline:loss_under}\;
   $\nnslope_1^*, ..., \nnslope_{\specnum}^* ,\nndual_1^*, ..., \nndual_{\specnum}^* \leftarrow \textnormal{argmin}(\textnormal{Loss}(\nnslope_1, ..., \nnslope_{\specnum},\nndual_1, ..., \nndual_{\specnum}))$\;
    $\textbf{\polytope} = \textnormal{Append}(\textbf{\polytope}, [\underline{\spec_1}(\inpoint, \nnslope_1^*,\nndual_1^*) \geq 0, ..., \underline{\spec_{\specnum}}(\inpoint, \nnslope_{\specnum}^*,\nndual_{\specnum}^*) \geq 0, \inpoint \in \indomain_{sub}])$\label{algline:append_under}
   }
   \Else{
   \bw{$[\overline{\spec_1}(\inpoint, \nnslope_1, \nndual_1), ..., \overline{\spec_{\specnum}}(\inpoint, \nnslope_{\specnum}, \nndual_{\specnum})] \leftarrow$ LinearUpperBound($\indomain_{sub}, \outset$)
    }
    \; \label{algline:linearbound_upper}
    
   Loss$(\nnslope_1, ..., \nnslope_{\specnum}, \nndual_1, ..., \nndual_{\specnum}) \leftarrow \frac{\text{vol}(\indomain_{sub})}{\numsamples} \sum_{j = 1, ..., \numsamples} \sigma(-\textnormal{LSE}(-\overline{\spec_1}(\inpoint_j, \nnslope_1,\nndual_1), ..., -\overline{\spec_{\specnum}}(\inpoint_j, \nnslope_{\specnum},\nndual_{\specnum})))$\label{algline:loss_over}\;
   $\nnslope_1^*, ..., \nnslope_{\specnum}^* ,\nndual_1^*, ..., \nndual_{\specnum}^* \leftarrow \textnormal{argmin}(\textnormal{Loss}(\nnslope_1, ..., \nnslope_{\specnum},\nndual_1, ..., \nndual_{\specnum}))$\;
    $\textbf{\polytope} = \textnormal{Append}(\textbf{\polytope}, [\overline{\spec_1}(\inpoint, \nnslope_1^*,\nndual_1^*) \geq 0, ..., \overline{\spec_{\specnum}}(\inpoint, \nnslope_{\specnum}^*,\nndual_{\specnum}^*) \geq 0, \inpoint \in \indomain_{sub}])$\label{algline:append_over}
   }
    
}

\KwRet{$\textbf{\polytope}$}
\end{algorithm}

\subsection{\bw{Polytope Approximation via Linear Relaxation}}
\label{sec:poly_gen}

We first show how to adapt linear relaxation techniques to efficiently generate valid under-approximations and over-approximations to the restricted preimage for a given input region $\indomain$ as a single polytope. Recall that LiRPA methods enable us to obtain linear lower and upper bounds on the output of a neural network $\nn$, that is, $\lowerweight \inpoint + \lowerbias \leq \net(\inpoint) \leq \upperweight \inpoint + \upperbias$, where the linear coefficients depend on the input region $\indomain$.

Suppose that we are given the input hyperrectangle $\indomain = \{\inpoint \in \mathbb{R}^{\indim} | \inpoint \models \bigwedge_{i = 1}^{\indim} \boxcon_i \}$, and the output polytope specified using the half-space constraints $\lincon_i(\outpoint) = (\linconw_i^{T} \outpoint + \linconb_i \geq 0)$ for $ i = 1, ..., \specnum$ over the output space. Let us first consider generating a guaranteed under-approximation. Given a constraint $\lincon_i$, we append an additional linear layer at the end of the network $\nn$, which maps $\outpoint \mapsto \linconw_i^{T} \outpoint + \linconb_i$, such that the function $\spec_i: \mathbb{R}^{\indim} \to \mathbb{R}$ represented by the new network is $\spec_i(\inpoint) = \linconw_i^{T} \nn(\inpoint) + \linconb_i$. 
\bw{Then, applying LiRPA lower bounding to each $\spec_i$, we obtain a lower bound $\underline{\spec_i}(\inpoint) = \lowerweightsingle_i^T \inpoint + \lowerbiassingle_i$ for each $i$, such that $\underline{\spec_i}(\inpoint) \geq 0 \implies \spec_i(\inpoint) \geq 0$ for $\inpoint \in \indomain$. 
Notice that, for each $i = 1,..., \specnum$, }
\bw{
$\lowerweightsingle_i^T \inpoint + \lowerbiassingle_i \geq 0$ is a half-space constraint in the input space. We conjoin these constraints, along with the restriction to the input region $\indomain$, to obtain a polytope:
\begin{equation}
    \underline{\polytope_{\indomain}}(\outset) := \{\inpoint| \bigwedge_{i =1}^{\specnum} (\underline{\spec_i}(\inpoint)  \geq 0 )\wedge \bigwedge_{i = 1}^{\indim} \boxcon_i(\inpoint) \}
\end{equation}


\begin{tcolorbox}
[width=\linewidth, sharp corners=all, colback=white!95!black, frame empty]

\paragraph{Over-Approximation} Alternatively, to generate a guaranteed over-approximation, we can instead apply LiRPA upper bounding to each $\spec_i$, obtaining upper bounds $\overline{\spec_i}(\inpoint) = \upperweightsingle_i^T \inpoint + \upperbiassingle_i$ for each $i$, such that $\spec_i(\inpoint) \geq 0 \implies \overline{\spec_i}(\inpoint) \geq 0$ for $\inpoint \in \indomain$, and defining the polytope:
\begin{equation}
    \overline{\polytope_{\indomain}}(\outset) := \{\inpoint| \bigwedge_{i =1}^{\specnum} (\overline{\spec_i}(\inpoint)  \geq 0 )\wedge \bigwedge_{i = 1}^{\indim} \boxcon_i(\inpoint) \}
\end{equation}
\end{tcolorbox}

\begin{restatable}{proposition}{propUnder}
$\underline{\polytope_{\indomain}}(\outset), \overline{\polytope_{\indomain}}(\outset)$ are respectively under- and over-approximations to the restricted preimage $\preimage_{\indomain}(\outset)$.
\end{restatable}

\begin{proof}
For the under-approximation, the LiRPA bound $\underline{\spec_i}(\inpoint) \leq \spec_i(\inpoint)$ holds for any $\inpoint \in \indomain$ and $i = 1, ..., \specnum$, and so we have $\bigwedge_{i =1}^{\specnum} (\underline{\spec_i}(\inpoint)  \geq 0 )\wedge \inpoint \in \indomain \implies \bigwedge_{i =1}^{\specnum} (\spec_i(\inpoint)  \geq 0 )\wedge \inpoint \in \indomain$, i.e., $\underline{\polytope_{\indomain}}(\outset)$ is an under-approximation to $\preimage_{\indomain}(\outset)$. Similarly, for the over-approximation, $\spec_i(\inpoint) \leq \overline{\spec_i}(\inpoint)$ holds for any $\inpoint \in \indomain$ and $i = 1, ..., \specnum$, and so $ \bigwedge_{i =1}^{\specnum} (\spec_i(\inpoint)  \geq 0 )\wedge \inpoint \in \indomain \implies \bigwedge_{i =1}^{\specnum} (\overline{\spec_i}(\inpoint)  \geq 0 )\wedge \inpoint \in \indomain$, i.e. $\overline{\polytope_{\indomain}}(\outset)$ is an over-approximation to $\preimage_{\indomain}(\outset)$.

\end{proof}

}
\begin{figure}[t]
\centering
\includegraphics[width=0.6\textwidth]{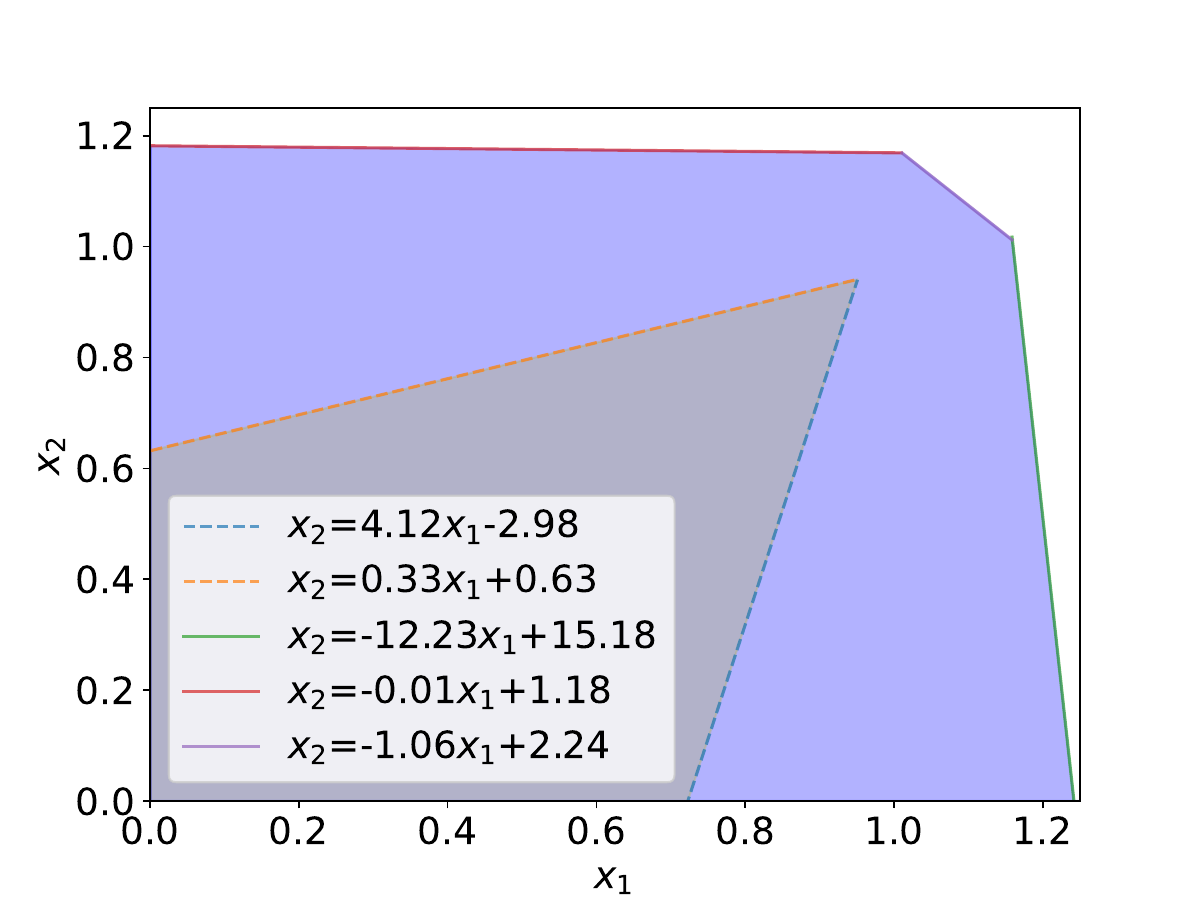}
\caption{Illustration of initial preimage under- and over-approximation for output specification $\wedge_{i \in \{2,3,4\}} (y_1 - y_i \geq 0)$ in the vehicle parking task (for more details see Example \ref{eg:underapprox}). The under-approximation is the polytope in grey, bounded by dotted half-planes, and the over-approximation is the polytope in blue, bounded by solid half-planes. 
\rev{The ground-truth preimage for the output specification is the rectangular region $[0,1]\times [0,1]$.}}
\label{fig:init_under_over}
\end{figure}
\begin{example}\label{eg:underapprox}
Returning to Example \ref{eg:formulation}, the output constraints (for $i = 2, 3, 4$) are given by $\lincon_i = (\outpoint_1 - \outpoint_i \geq 0) = (\linconw_i^{T} \outpoint + \linconb_i \geq 0)$,  where $\linconw_i := e_1 - e_i$ (we use $e_i$ to denote the $i^{\text{th}}$ standard basis vector) and $\linconb_i := 0$. \bw{Applying LiRPA bounding, we obtain the linear lower bounds $\underline{\spec_2}(\inpoint) = -4.12 \inpoint_1 + \inpoint_2 + 2.98 \geq 0; \underline{\spec_3}(\inpoint) = 0.33 \inpoint_1 - \inpoint_2 +0.63 \geq 0$; and $\underline{\spec_4}(\inpoint)$ (not shown).
The intersection of these constraints, shown in Figure \ref{fig:init_under_over} (region in grey), represents an under-approximation to the preimage.}
\xy{Similarly, we can obtain linear upper bounds $\overline{\spec_2}(\inpoint) = -12.23 \inpoint_1 - \inpoint_2 + 15.18 \geq 0; \overline{\spec_3}(\inpoint) = -0.01 \inpoint_1 - \inpoint_2 + 1.18 \geq 0$; and $\overline{\spec_4}(\inpoint)= -1.06 \inpoint_1 - \inpoint_2 + 2.24 \geq 0$; the intersection of those constraints represents an over-approximation to the preimage, as shown in Figure \ref{fig:init_under_over} (region in blue).}
\end{example}

We generate the linear bounds in parallel over the output polyhedron constraints $i = 1, ..., \specnum$ using the \textit{backward mode} LiRPA \citep{zhang2018crown}, and store the resulting \bw{approximating polytope as a list of constraints.}
This highly efficient procedure is used as a sub-routine \texttt{LinearBounds}  when generating either \bw{preimage under-approximations or over-approximations} in Algorithm \ref{alg:genunderapprox}
(Lines \ref{algline:linearbound_lower}, \ref{algline:linearbound_upper}).

\begin{algorithm}[htb]
\small
\caption{Preimage Approximation}\label{alg:main}
\KwIn{Neural network $f$, Input region $\indomain$, Output region $\outset$, Volume \bw{threshold $\threshold$}, Maximum iterations $\maxiterations$, Number of samples $\numsamples$, Boolean $Under$, Boolean $SplitOnInput$ }
\KwOut{Disjoint union of polytopes $\polytopeset_{\textnormal{Dom}}$}
\bw{$\polytope$ $\leftarrow$ GenApprox($\indomain$, $\outset$, $\numsamples$, $Under$)} \tcp*{Initial preimage polytope}\label{algline:initial_polytope}
\bw{Dom $\leftarrow \{(\indomain, \polytope, \text{CalcPriority}(\polytope, Under))\}$ \tcp*{Priority queue}}\label{algline:init_queue}
\bw{\tcp{$\polytopeset_{\textnormal{Dom}}$ is the union of the under/over-approximating polytopes in Dom}}
\While{\bw{$((\textnormal{Under} \textnormal{\textbf{ and }} \textnormal{EstimateVolume}(\polytopeset_{\textnormal{Dom}}) < \threshold ) $ \textnormal{\textbf{or}} $(\neg\textnormal{Under \textbf{and} }\textnormal{EstimateVolume}(\polytopeset_{\textnormal{Dom}}) > \threshold))$ \textnormal{\textbf{and}} $\textnormal{Iterations}\leq \maxiterations$ }\label{algline:while}}{ 
   \bw{ $\indomain_{\text{sub}}, \polytope, \text{Priority}$ $\leftarrow$ Pop(Dom) }\label{algline:refine_start}\tcp*{Subregion with highest priority}
    \If{SplitOnInput}
    {
    $id$ $\leftarrow$ SelectInputFeature($\text{Feature}_{I}, Under$) \tcp*{$\text{Feature}_{I}$ is the set of input features/dimensions}\label{algline:input_selection}
  }
  \Else{
    $id \leftarrow$ SelectReLUNode($\text{Node}_{Z}, Under$)\tcp*{$\text{Node}_{Z}$ is the set of unstable ReLU nodes}\label{algline:relu_selection}
  }
[$\indomain_{sub}^{l}$,$\indomain_{sub}^{u}$] 
 $\leftarrow$ SplitOnNode($\indomain_{sub}$, $id$)\label{algline:split_node_id}\tcp*{Split on the selected node}
  \bw{$\polytope^{l}, \polytope^{u}$$\leftarrow$ GenApprox([$\indomain_{sub}^{l}$,$\indomain_{sub}^{u}$],  $\outset$, $\numsamples, Under$) }\tcp*{Generate preimage}\label{algline:subdomain_polytope}
    \bw{Dom $\leftarrow$ Dom $\cup$ \{($\indomain_{sub}^{l}, \polytope^{l}$,$\text{CalcPriority}(\polytope^{l}, Under)$),  ($\indomain_{sub}^{u}, \polytope^{r}$,$\text{CalcPriority}(\polytope^{r}, Under)$)\}\label{algline:refine_end}}\tcp*{Disjoint polytope} 
    } \label{algline:update_approximation}
\KwRet{\bw{$\polytopeset_{\textnormal{Dom}}$}}
\end{algorithm}

\subsection{Global Branching and Refinement} \label{sec:branching}

\xy{As LiRPA performs crude linear relaxation,
the resulting bounds can be quite loose, even with optimisation over bounding parameters (as we will see in Section \ref{sec:local_opt}), meaning that the (single) 
polytope under-approximation 
or over-approximation 
is unlikely to \bw{be a good approximation to the preimage by itself}.
}
To address this challenge, we employ a divide-and-conquer approach that
iteratively refines our approximation of the preimage. Starting from the initial region $\indomain$ at the root, our method generates a tree by iteratively partitioning a subregion $\indomain_{sub}$ represented at a leaf node into two smaller subregions $\indomain_{sub}^{l}, \indomain_{sub}^{u}$, which are then attached as children to that leaf node. In this way, the subregions represented by all leaves of the tree are disjoint, such that their union is the initial region $\indomain$.

\bw{In order to under-approximate (resp. over-approximate) the preimage}, for each leaf subregion $\indomain_{sub}$ we compute, using  LiRPA bounds, 
\bw{an associated polytope that under-approximates (resp. over-approximates)} the preimage in $\indomain_{sub}$. Thus, irrespective of the number of refinements performed, the union of the \bw{under-approximating} polytopes \bw{(resp. over-approximating)} corresponding to all leaves forms an \emph{anytime} DUP under-approximation \bw{(resp. over-approximation)} $\polytopeset$ to the preimage in the original region $\indomain$.
The process of refining the subregions  continues until an appropriate termination criterion is met.

Unfortunately, even with a moderate number of input dimensions or unstable ReLU nodes, na\"ively splitting along all input- or ReLU-planes quickly becomes computationally intractable. For example, splitting a $\indim$-dimensional hyperrectangle using
bisections along each dimension results in $2^d$ subdomains to approximate. It thus becomes crucial to prioritise the subregions to split, as well as improve the efficiency of the splitting procedure itself.
We describe these in turn.

\paragraph{Subregion Selection.}
We propose a subregion selection strategy that prioritises splitting subregions \bw{with the largest} difference in volume
between the exact preimage $\preimage_{\indomain_{sub}}(\outset)$ and the (already computed) polytope approximation $\polytope_{\indomain_{sub}}(\outset)$
on that subdomain: this indicates ``how much improvement" can be achieved on this subdomain and is implemented as the \texttt{CalcPriority} function in Algorithm \ref{alg:main}.  
Unfortunately, computing the volume of a polytope exactly is a computationally expensive task, requiring specialised tools \citep{Augustin22polyvolume}. 
To overcome this, we employ Monte Carlo estimation of volume computation by sampling $\numsamples$ points $\inpoint_1, ..., \inpoint_\numsamples$ uniformly from the input subdomain $\indomain_{sub}$.
For an under-approximation, we have: 
\begin{align}
\textnormal{Priority}(\indomain_{sub}) &:= \frac{\text{vol}(\indomain_{sub})}{\numsamples} \times \left(\sum_{i = 1}^{\numsamples}\mathds{1}_{\inpoint_i \in \preimage_{\indomain_{sub}}(\outset)} - \sum_{i = 1}^{\numsamples}\mathds{1}_{\inpoint_i \in \underline{\polytope_{\indomain_{sub}}}(\outset)} \right)  \\
&\approx 
\volume(\preimage_{\indomain_{sub}}(\outset)) -  \volume(\underline{\polytope_{\indomain_{sub}}}(\outset)) \label{eqnline:priority_under_volume}
\end{align}



This measures the 
gap between the polytope under-approximation and the optimal approximation, namely, the preimage itself. 

\begin{tcolorbox}
[width=\linewidth, sharp corners=all, colback=white!95!black, frame empty]

\paragraph{Over-Approximation} Similarly, in the case of an over-approximation, we define:
\begin{align}
\textnormal{Priority}(\indomain_{sub}) &:= \frac{\text{vol}(\indomain_{sub})}{\numsamples} \times \left(\sum_{i = 1}^{\numsamples}\mathds{1}_{\inpoint_i \in \overline{\polytope_{\indomain_{sub}}}(\outset)} - \sum_{i = 1}^{\numsamples}\mathds{1}_{\inpoint_i \in \preimage_{\indomain_{sub}}(\outset)} \right)  \\
&\approx 
\volume(\overline{\polytope_{\indomain_{sub}}}(\outset)) - 
\volume(\preimage_{\indomain_{sub}}(\outset)) \label{eqnline:priority_over_volume}
\end{align}
\end{tcolorbox}

We then choose the leaf subdomain with the maximum priority.
This leaf subdomain is then partitioned into two subregions $\indomain_{sub}^l, \indomain_{sub}^u$, \bw{each of which we then approximate with polytopes $\polytope_{\indomain_{sub}^l}(\outset), \polytope_{\indomain_{sub}^u}(\outset)$.}
\rev{Tighter intermediate concrete bounds, and thus tighter linear bounding functions, can often be computed on the partitioned subregions.
While such locally improved bounds do not guarantee a global improvement in preimage volume, 
the polytope approximation on each subregion is typically
refined compared to that on the original subregion. 
In our framework, we design a greedy input splitting strategy (see below) that selects the partition leading to the greatest improvement in preimage volume. Additionally, we perform optimisation (Section \ref{sec:local_opt} and \ref{sec:dual}) over bounding parameters within each subregion to further enhance the tightness of the polytope approximation.}

\rev{Notice that, although we approximate the volumes by sampling, this does not affect the \emph{deterministic} volume guarantees provided by our method, as the Priority is simply a heuristic used to guide the algorithm.}
\rev{In the rest of this subsection, we consider how to split a leaf subregion into two subregions to optimise the volume of the preimage approximation. In particular, we propose two approaches: \emph{input splitting} and \emph{ReLU splitting}.}

\paragraph{Input Splitting.}
Given a subregion (hyperrectangle) defined by lower and upper bounds $\inpoint_i \in [\featurelowerbound_i, \featureupperbound_i]$ for all dimensions $i = 1, ..., \indim$, input splitting partitions it
into two subregions by 
cutting along some feature $i$. 
This splitting procedure will produce two subregions
that are similar to the original subregion, but have
updated bounds $[\featurelowerbound_i, \frac{\featurelowerbound_i + \featureupperbound_i}{2}], [\frac{\featurelowerbound_i + \featureupperbound_i}{2}, \featureupperbound_i]$ for feature $i$ instead. 
A commonly-adopted splitting heuristic is to select the dimension with the longest edge~\citep{Bunel20BaB}, that is, to select feature $i$ with the largest range: $\arg \max_i (\featureupperbound_i-\featurelowerbound_i)$.
However, this method does not perform well in terms of per-iteration volume improvement of the preimage approximation.

\bw{Thus, we propose to greedily select a dimension instead according to a volume-aware heuristic.}
Specifically, for each feature, \bw{we generate approximating polytopes $\polytope^l, \polytope^r$ for the two subregions resulting from the split, and choose the feature that maximises the following priority metric. In the case of under-approximation, when $\polytope^l$ consists of linear lower bounds $\underline{\spec_1}^l, \ldots, \underline{\spec_\specnum}^l$ and $\polytope^{r}$ consists of linear lower bounds $\underline{\spec_1}^r, \ldots, \underline{\spec_\specnum}^r$, we define:
}
\bw{
\begin{equation} \label{eqn:priority_under}
    \text{InputPriority}(\polytope^{l}, \polytope^{r}) := \frac{\text{vol}(\indomain_{sub})}{\numsamples} \left( \sum_{j = 1}^{\numsamples} \sigma\left(\min_{i = 1, ... \specnum} \underline{\spec_i^l}(\inpoint_j)  \right) + \sum_{j = 1}^{\numsamples} \sigma\left(\min_{i = 1, ... \specnum} \underline{\spec_i^r}(\inpoint_j)  \right)  \right)
\end{equation}
where $\sigma$ is the sigmoid function $\sigma(y) := \frac{1}{1 + e^{-y}}$. Intuitively, this is an approximation to the (total) volume of the under-approximating polytopes (e.g., $\inpoint_j$ is in the polytope $\polytope_l$ iff $\min_{i = 1, ... \specnum} \underline{\spec_i^l}(\inpoint_j) > 0$); we should prefer to split on input features that maximise the total volume. However, we found  empirically that, in early iterations of the refinement, the under-approximation could often be empty (as the set $\{\min_{i=1, ..., \specnum} \underline{\spec_i^l}(\inpoint_j) > 0 \}$ for all $i$ lies outside the subregion $\indomain_{sub}$), leading to zero priority for all features. For this reason, we propose to instead use the smooth sigmoid function to measure ``how close" the constraints are to being satisfied for the sampled points, in order to provide signal for the best feature to split on.

\begin{tcolorbox}
[width=\linewidth, sharp corners=all, colback=white!95!black, frame empty]

\paragraph{Over-Approximation} Similarly, if we are generating an over-approximation, then we prioritise according to the following (i.e., minimising volume):
\begin{equation} \label{eqn:priority_over}
    \text{InputPriority}(\polytope^{l}, \polytope^{r}) := -\frac{\text{vol}(\indomain_{sub})}{\numsamples} \left( \sum_{j = 1}^{\numsamples} \sigma\left(\min_{i = 1, ... \specnum} \overline{\spec_i^l}(\inpoint_j)  \right) + \sum_{j = 1}^{\numsamples} \sigma\left(\min_{i = 1, ... \specnum} \overline{\spec_i^r}(\inpoint_j)  \right)  \right)
\end{equation}
\end{tcolorbox}

}

\begin{figure}[t]
     \centering
     \begin{subfigure}[b]{0.32\textwidth}
         \centering
         \includegraphics[width=\textwidth]{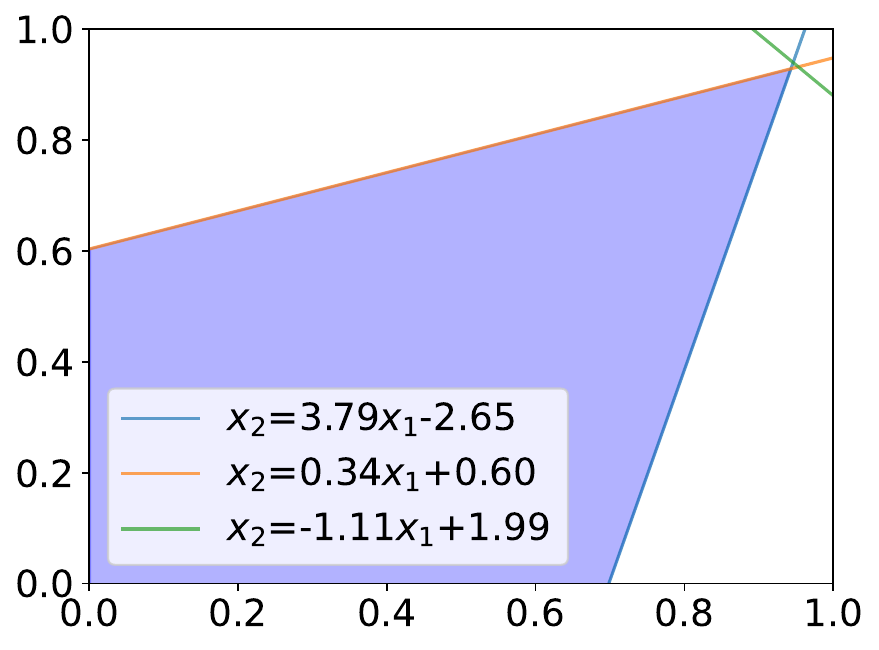}
         \caption{Initial under-approximation}
         \label{fig:init}
     \end{subfigure}
     \begin{subfigure}[b]{0.32\textwidth}
         \centering
         \includegraphics[width=\textwidth]{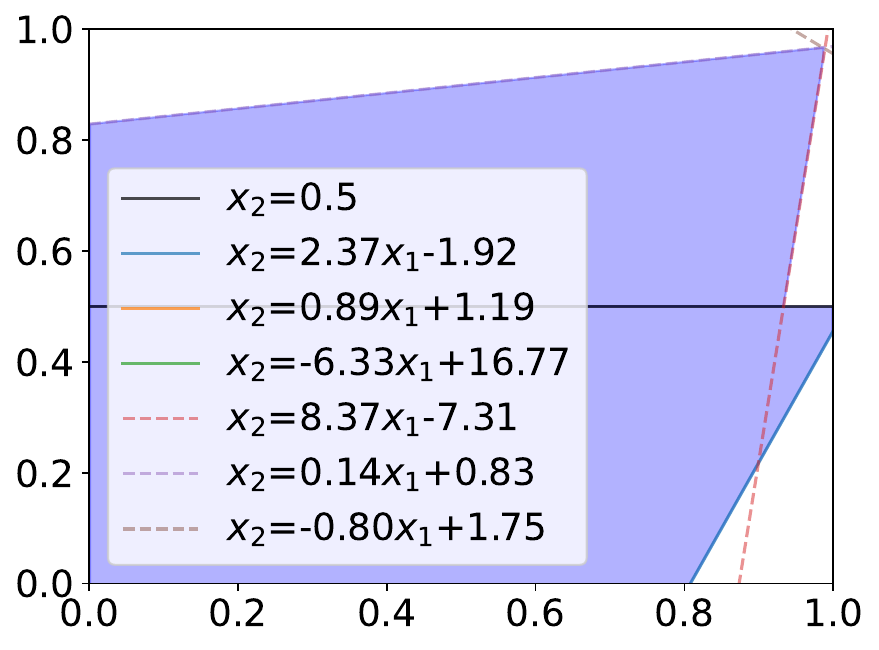}
         \caption{Input split}
         \label{fig:after_branch}
     \end{subfigure}
     \begin{subfigure}[b]{0.32\textwidth}
         \centering
         \includegraphics[width=\textwidth]{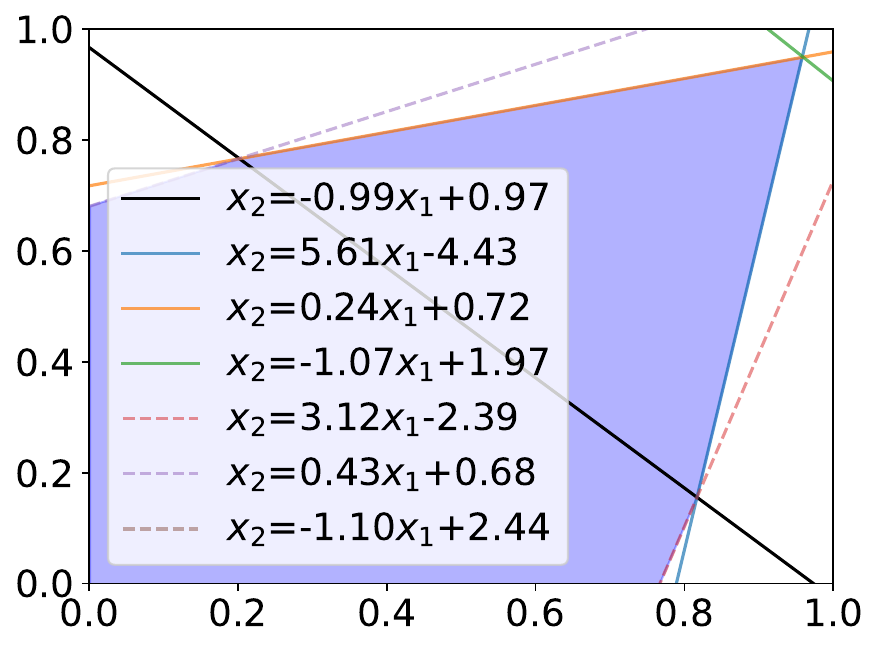}
         \caption{ReLU split}
         \label{fig:relu_split}
     \end{subfigure}
        \caption{Refinement of the initial preimage under-approximation with input and ReLU splitting. Figure \ref{fig:after_branch} and \ref{fig:relu_split} display the refined preimage, i.e., larger volume, after adding input and ReLU splitting planes, where the dotted and solid bounding planes are used to form the polytope on each subregion, respectively.}
        \label{fig:global_branching}
\end{figure}
\begin{figure}[t]
     \centering
     \begin{subfigure}[b]{0.32\textwidth}
         \centering
         \includegraphics[width=\textwidth]{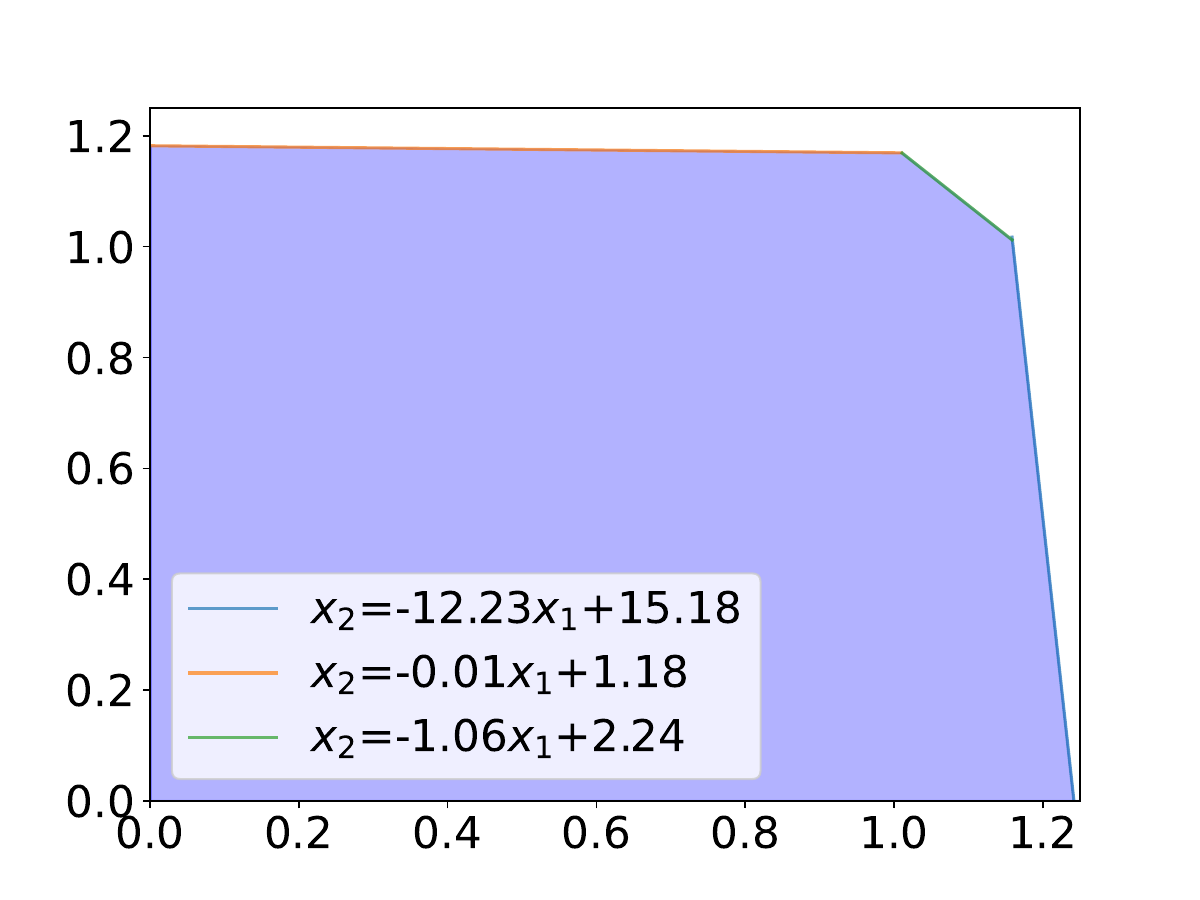}
         \caption{Initial over-approximation}
         \label{fig:init_over}
     \end{subfigure}
     \begin{subfigure}[b]{0.32\textwidth}
         \centering
         \includegraphics[width=\textwidth]{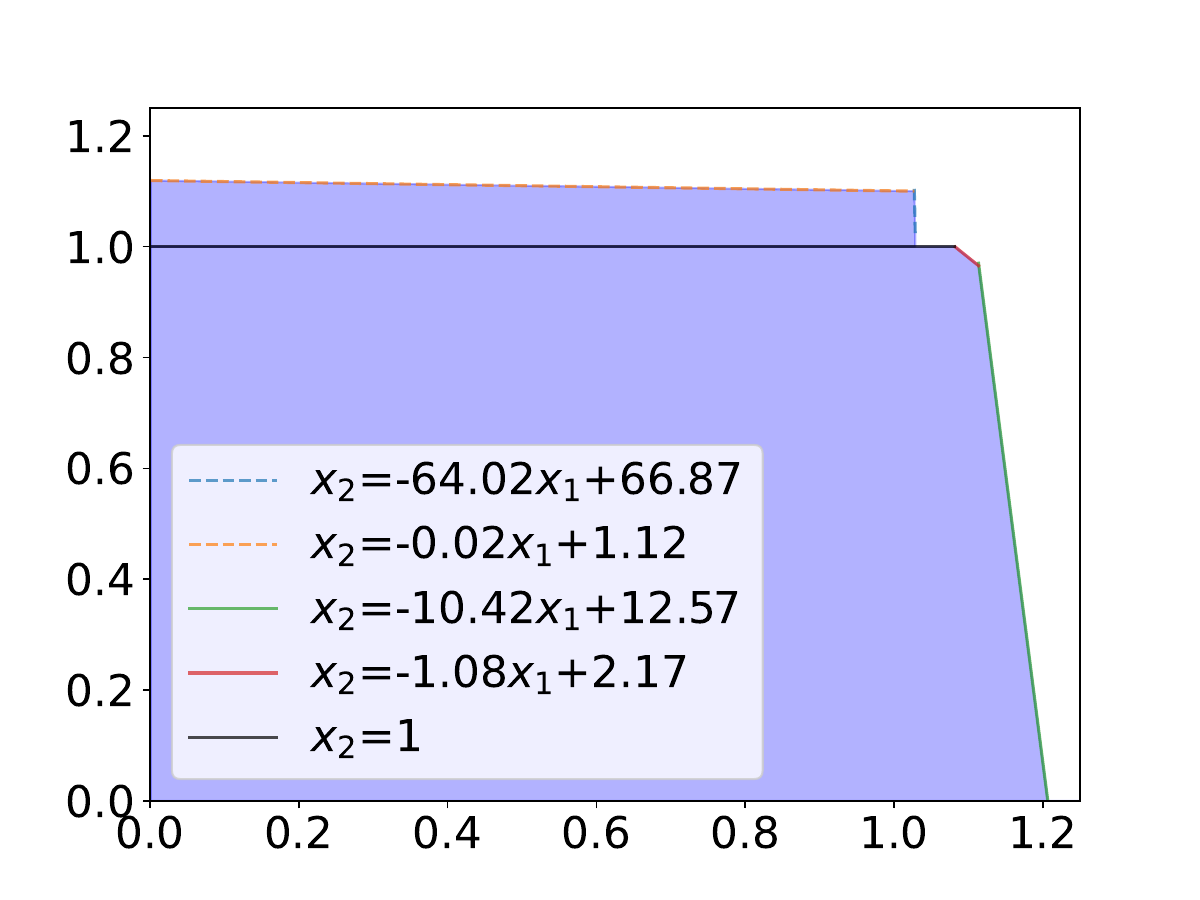}
         \caption{Input split}
         \label{fig:over_input_split}
     \end{subfigure}
     \begin{subfigure}[b]{0.32\textwidth}
         \centering
         \includegraphics[width=\textwidth]{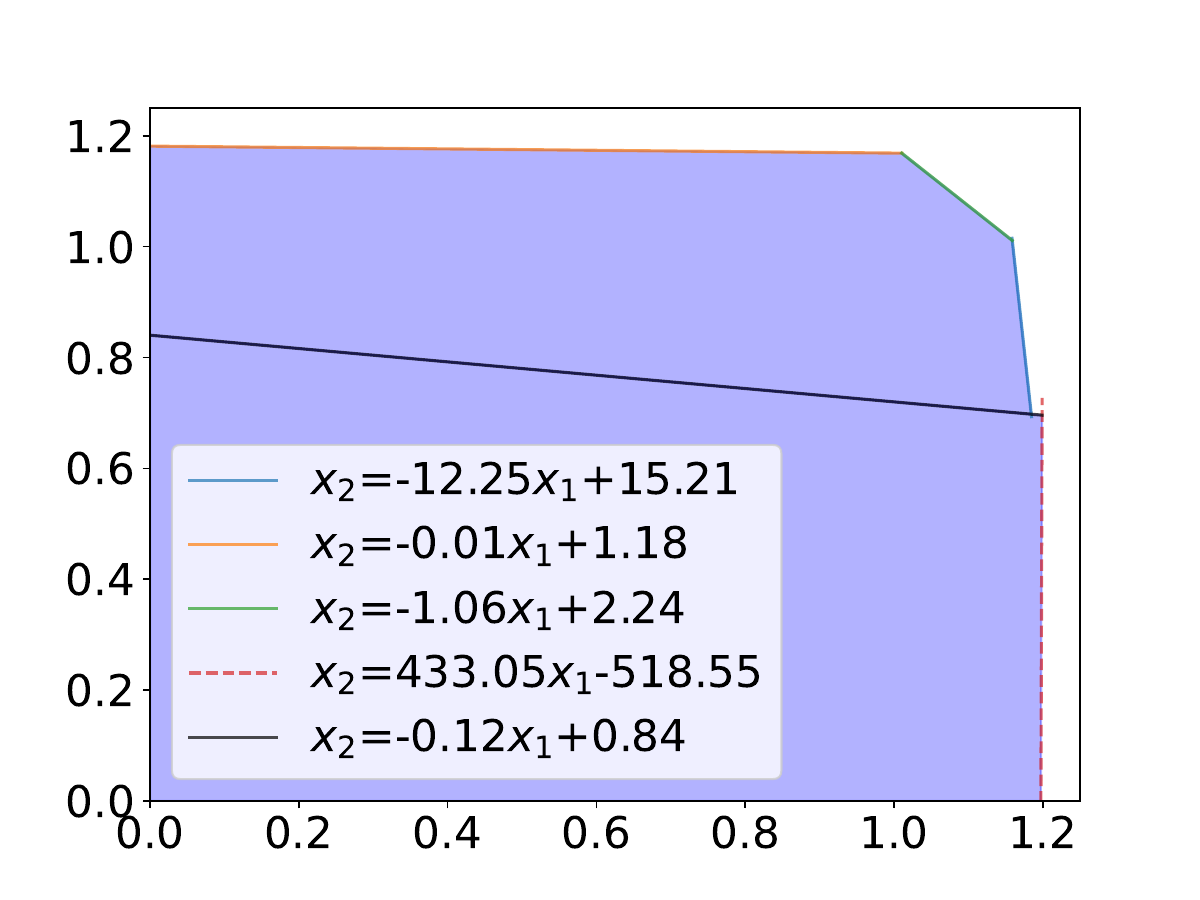}
         \caption{ReLU split}
         \label{fig:over_relu_split}
     \end{subfigure}
        \caption{
        Refinement of the initial preimage over-approximation with input and ReLU splitting. Figure \ref{fig:over_input_split} and \ref{fig:over_relu_split} display the refined preimage, i.e., smaller volume, after adding input and ReLU splitting planes, where the dotted and solid bounding planes are used to form the polytope on each subregion, respectively.}
        \label{fig:global_branching_over}
\end{figure}


\begin{example}
\label{eg:branching}
We revisit 
Example~\ref{eg:formulation}.
\xy{Figure \ref{fig:init} shows the initial polytope under-approximation computed on the input region $\indomain$ before refinement, where each solid line represents the bounding plane for each output specification ($\outpoint_1 - \outpoint_i \geq 0 $). 
Figure \ref{fig:after_branch} depicts the refined approximation by splitting the input region along the vertical axis, where the solid and dashed lines represent the 
bounding planes for the two resulting subregions. It can be seen that the total volume of the under-approximation has improved significantly.} \bw{Similarly, in Figure \ref{fig:init_over}, we show the initial polytope over-approximation before refinement, and in Figure \ref{fig:over_input_split} 
the improved over-approximation after greedy input splitting.}
\end{example}

\textbf{Intermediate ReLU Splitting.}
Refinement through splitting on input features is adequate for low-dimensional input problems such as reinforcement learning agents. However, it may be infeasible to generate sufficiently fine subregions for high-dimensional domains. 
We thus propose an algorithm for ReLU neural networks that uses intermediate ReLU splitting for preimage refinement.
After determining a subregion for refinement, we partition the subregion based upon the pre-activation value of an intermediate unstable neuron $\intervar=0$.
As a result, the original subregion $\indomain_{sub}$ is split into two new subregions
$\indomain^{+}_{\intervar}=\{\inpoint \in \indomain_{sub}~|~\intervar=\preact^{(i)}_j(\inpoint) \ge 0\}$
and $\indomain^{-}_{\intervar}=\{\inpoint \in \indomain_{sub}~|~\intervar=\preact^{(i)}_j(\inpoint) < 0\}$. Note that to obtain the polytope approximation, we can utilise linear lower/upper bounds on $\preact^{(i)}_j(\inpoint)$ as an approximation to the subregion boundary.

\rev{In this procedure, the order of splitting unstable ReLU neurons can greatly influence the quality and efficiency of the refinement. 
Existing heuristic methods for ReLU prioritisation select ReLU nodes that lead to greatest improvement in the final bound (maximum or minimum value) of the neural network $f$ over the input domain~\citep{Bunel20BaB}, e.g., improving $\min_{\inpoint \in \indomain} \nnunder(x)$.
These methods focus on optimising the worst-case output bounds at specific input points such as $x^{\ast}=\arg \min_{\inpoint \in \indomain} \nnunder(x)$.
However, these methods are not well-suited for preimage analysis, 
where our aim is instead to refine the preimage approximation to the exact preimage. Specifically, the objective is to minimise the volume of under-approximations and maximise volumes of over-approximations. Prioritising ReLU nodes based on local worst-case bound improvements may therefore fail to improve the overall precision of the preimage approximation.
}
To this end, 
we compute (an estimate of) the volume difference 
between the split subregions
$\lvert \volume(\indomain^{+}_{\intervar})-\volume(\indomain^{-}_{\intervar})\rvert$, using a single forward pass for a set of sampled data points from the input domain; note that this is bounded above by the total subregion volume $\volume(\indomain_{sub})$.
We then propose to select the ReLU node that minimises this difference.
Intuitively, this choice results in balanced subdomains after splitting. 

\rev{A key advantage of ReLU splitting is that it allows us to replace unstable neuron bounds with precise bounds. For an unstable ReLU neuron $\postact^{(i)}_j(\inpoint) = \max(0, \preact^{(i)}_j(\inpoint))$, we use linear relaxation to bound the post-activation value (as in Equation~\ref{eq:node_bounding}), which yields linear lower and upper bounds of the form $\underline{c} \preact^{(i)}_j(\inpoint) + \underline{d} \leq \postact^{(i)}_j(\inpoint) \leq \overline{c} \preact^{(i)}_j(\inpoint) + \overline{d}$, where $\underline{c}$, $\overline{c}$ denote the slopes and $\underline{d}$, $\overline{d}$ denote the intercepts of the lower and upper bounding functions, respectively.
When a split is introduced, the neuron becomes stable in each subdomain, and the exact linear function $\postact^{(i)}_j(\inpoint) = \preact^{(i)}_j(\inpoint)$ and $\postact^{(i)}_j(\inpoint) = 0$ can be used in place of its linear relaxation, as shown in Figure \ref{fig:linear_relaxation} (unstable to stable).} 
This can typically tighten the approximation on each subdomain as the linear relaxation errors for this unstable neuron are removed for each subdomain and substituted with the exact symbolic function for backward propagation.

\begin{example}
\label{eg:reluSplit}
We now apply our algorithm with ReLU splitting to the 
problem in Example~\ref{eg:formulation}.
Figure \ref{fig:relu_split} shows the refined preimage polytope by adding the splitting plane (black solid line) along the direction of a selected unstable ReLU node. 
\xy{Compared with Figure \ref{fig:init}, we can see that the volume of the approximation increased.} \bw{Similarly, in Figure \ref{fig:over_relu_split}, we show the improved over-approximation after ReLU splitting, compared to the initial over-approximation \ref{fig:init_over}.}

\end{example}

\xy{\textbf{Combining Preimage Polytopes.}
As the final step, we combine the refined symbolic approximations on each subregion to compute the disjoint polytope union for the desired preimage of the output property.
Note that the input splitting (hyper)planes naturally yield disjoint subregions. 
We can directly compute the final disjoint polytope union by combining the preimage polytopes of each subregion, where the splitting planes serve as part of the constraints that form the preimage polytope, e.g., two disjoint polytopes with the splitting constraints $x_2-0.5\geq 0$ and $-x_2+0.5\geq 0$, respectively, partitioned by $x_2=0.5$ in Figure \ref{fig:after_branch}.
In the case of ReLU splitting,
as each ReLU neuron represents a complex non-linear function with respect to the input, we cannot directly add the constraints introduced by ReLU splitting to the polytope representation.
Instead, we compute the linear upper or lower bounding functions of the non-linear constraint represented by the ReLU neuron, i.e., $\underline{\preact^{(i)}_j}(\inpoint) \le \preact^{(i)}_j(\inpoint) \le \overline{\preact^{(i)}_j}(\inpoint)$.
The constraints introduced by the linear bounding functions, i.e., $\underline{\preact^{(i)}_j}(\inpoint) \ge 0$  and $-\underline{\preact^{(i)}_j}(\inpoint) \ge 0$, can then be added to form disjoint polytopes. For instance, as shown in Figure \ref{fig:relu_split}, two disjoint polytopes are formed with the additional splitting constraints $-0.99x_1-x_2+0.97 \ge 0$ and $0.99x_1+x_2-0.97 \ge 0$, respectively,
partitioned by the linear splitting plane $x_2=-0.99x_1+0.97$ (exact linear function of the selected ReLU neuron in this case). 
In fact, any linear function between the linear upper and lower bounding functions of the ReLU neuron serves as a valid splitting (hyper)plane to form disjoint polytopes.
}

\rev{\textbf{Input vs ReLU Splitting.} 
Input and ReLU splitting are alternative strategies for splitting into smaller subregious that users can employ for different application scenarios, but not both simultaneously (Lines \ref{algline:refine_start}-\ref{algline:relu_selection} in Algorithm \ref{alg:main}). The choice of which to use should primarily be guided by the dimensionality of the problem; we found in our experiments that input splitting is quite effective for low-dimensional problems (see Section \ref{sec:inputVSrelu}), while ReLU splitting is necessary to scale to high dimensions. This is consistent with findings from LiRPA-based verification tools such as $\alpha$-$\beta$-CROWN \citep{wang2021beta} and our comparison experiments in Section \ref{sec:inputVSrelu}.}

\subsection{Local Optimization} 
\label{sec:local_opt}

One of the key components behind the effectiveness of LiRPA-based bounds is the ability to efficiently improve the 
tightness of the bounding function by optimising the relaxation parameters $\nnslope$ via projected gradient descent. 
In the context of local robustness verification,  
the goal is to optimise the 
\emph{concrete} (scalar) lower or upper bounds 
over the (sub)region $\indomain_{sub}$ \citep{xu2020automated}, i.e.,
$\min_{\inpoint \in \indomain_{sub}} \lowerweight(\nnslope) \inpoint + \lowerbias(\nnslope)$ in the case of lower bounds, where we explicitly note
the dependence of the linear coefficients on $\nnslope$.
In our case, we are instead interested in optimising $\nnslope$ to refine
the polytope 
\bw{approximation, that is, increase the volume of under-approximations and decrease the volume of over-approximations (to the exact preimage).}

As before, we employ statistical estimation; we sample $\numsamples$ points $\inpoint_1, ..., \inpoint_\numsamples$ uniformly from the input domain $\indomain_{sub}$
then employ Monte Carlo estimation for the volume of the approximating polytope. \bw{In the case of under-approximation, we have: } 
\bw{
\begin{align}\label{eq:stat_est}
&\widehat{\volume}(\underline{\polytope_{\indomain_{sub}, \nnslope}}(\outset)) = \frac{\text{vol}(\indomain_{sub})}{\numsamples} \times \sum_{i = 1}^{\numsamples}\mathds{1}_{\inpoint_i \in \underline{\polytope_{\indomain_{sub}, \nnslope}}(\outset)}
\end{align}}
%
where we highlight the dependence of  \bw{$\underline{\polytope_{\indomain_{sub}}}(\outset) = \{\inpoint| \bigwedge_{i =1}^{\specnum} \underline{\spec_i}(\inpoint, \nnslope_i)  \geq 0 \wedge \bigwedge_{i = 1}^{\indim} \boxcon_i(\inpoint) \}$ on $\nnslope = (\nnslope_1, ..., \nnslope_\specnum)$,  
and
$\nnslope_i$ are the $\nnslopesingle$-parameters for the linear relaxation of the neural network $\spec_i$ corresponding to the $i^{\textnormal{th}}$ half-space constraint in $\outset$.}
However,
this is still non-differentiable w.r.t.\ $\nnslope$ \rev{due to the indicator function}. We now show how to derive a differentiable relaxation, which is amenable to gradient-based optimization: 
\bw{
\begin{align}
    \widehat{\volume}(\underline{\polytope_{\indomain_{sub}, \nnslope}}(\outset)) &= \frac{\text{vol}(\indomain_{sub})}{\numsamples} \sum_{j = 1}^{\numsamples}\mathds{1}_{\inpoint_j \in \underline{\polytope_{\indomain_{sub}, \nnslope}}(\outset)} 
    = \frac{\text{vol}(\indomain_{sub})}{\numsamples} \sum_{j = 1}^{\numsamples} \mathds{1}_{\min_{i = 1, ... \specnum} \underline{\spec_i}(\inpoint_j, \nnslope_i) \geq 0}\\
    & \approx \frac{\text{vol}(\indomain_{sub})}{\numsamples} \sum_{j = 1}^{\numsamples} \sigma\left(\min_{i = 1, ... \specnum} \underline{\spec_i}(\inpoint_j, \nnslope_i)\right) \\
    & \approx \frac{\text{vol}(\indomain_{sub})}{\numsamples} \sum_{j = 1}^{\numsamples} \sigma\left(-\textnormal{LSE}( -\underline{\spec_1}(\inpoint_j, \nnslope_1), ..., -\underline{\spec_\specnum}(\inpoint_j, \nnslope_\specnum))\right) \label{eqnline:diff_approximation_alpha}
\end{align}
}
\bw{As before, we use a sigmoid relaxation to approximate the volume. However, the minimum function is still non-differentiable.
Thus}, we approximate the minimum over specifications using the log-sum-exp (LSE) function. The log-sum-exp function is defined by $LSE(y_1, ..., y_{\specnum}) := \log(\sum_{i = 1, ..., \specnum} e^{y_i})$, and is a differentiable approximation to the maximum function; we employ it to approximate the minimisation by adding the appropriate sign changes. The final expression is now a differentiable function of $\nnslope$.


Then the goal is to maximise the volume of the under-approximation with respect to $\nnslope$: 
\begin{equation}
    \textnormal{Loss}(\nnslope) = 
    - \widehat{\volume}(\underline{\polytope_{\indomain_{sub}, \nnslope}}(\outset))
\end{equation}
We employ this as the loss function in Algorithm \ref{alg:genunderapprox} (Line \ref{algline:loss_under}) for generating a polytope approximation, and optimise volume using projected gradient descent. 

\begin{tcolorbox}
[width=\linewidth, sharp corners=all, colback=white!95!black, frame empty]

\paragraph{Over-Approximation} In the case of an over-approximation (Line \ref{algline:loss_over} of Algorithm \ref{alg:genunderapprox}), we instead aim to minimise the volume of the approximation:
\begin{equation}
    \textnormal{Loss}(\nnslope) = 
    \widehat{\volume}(\overline{\polytope_{\indomain_{sub}, \nnslope}}(\outset))
\end{equation}

\end{tcolorbox}

\begin{figure}[t]
     \centering
     \begin{subfigure}[b]{0.4\textwidth}
         \centering
         \includegraphics[width=\textwidth]{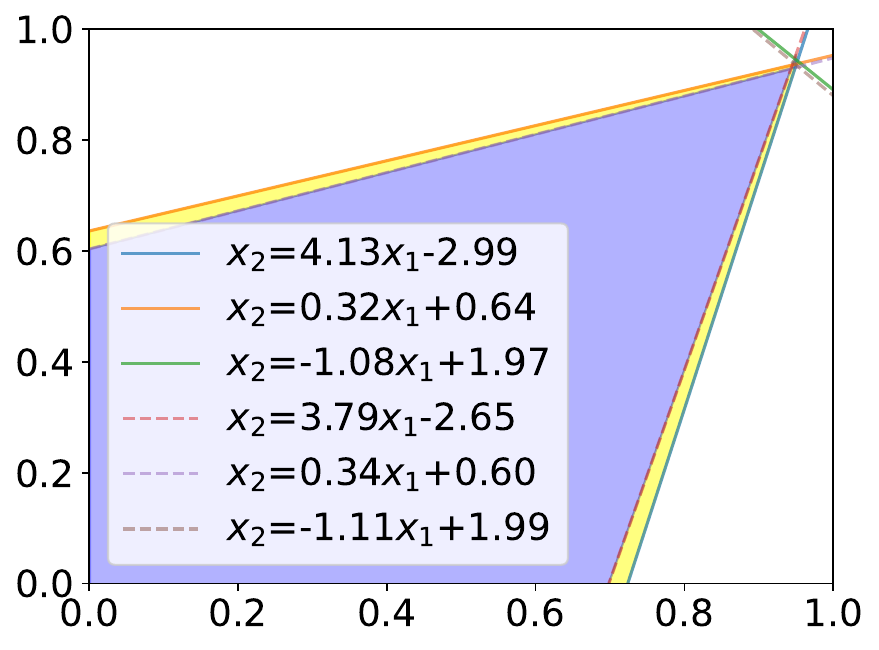}
         \caption{Optimised under-approximation}
         \label{fig:after_optim_under}
     \end{subfigure}
     \begin{subfigure}[b]{0.4\textwidth}
         \centering
         \includegraphics[width=\textwidth]{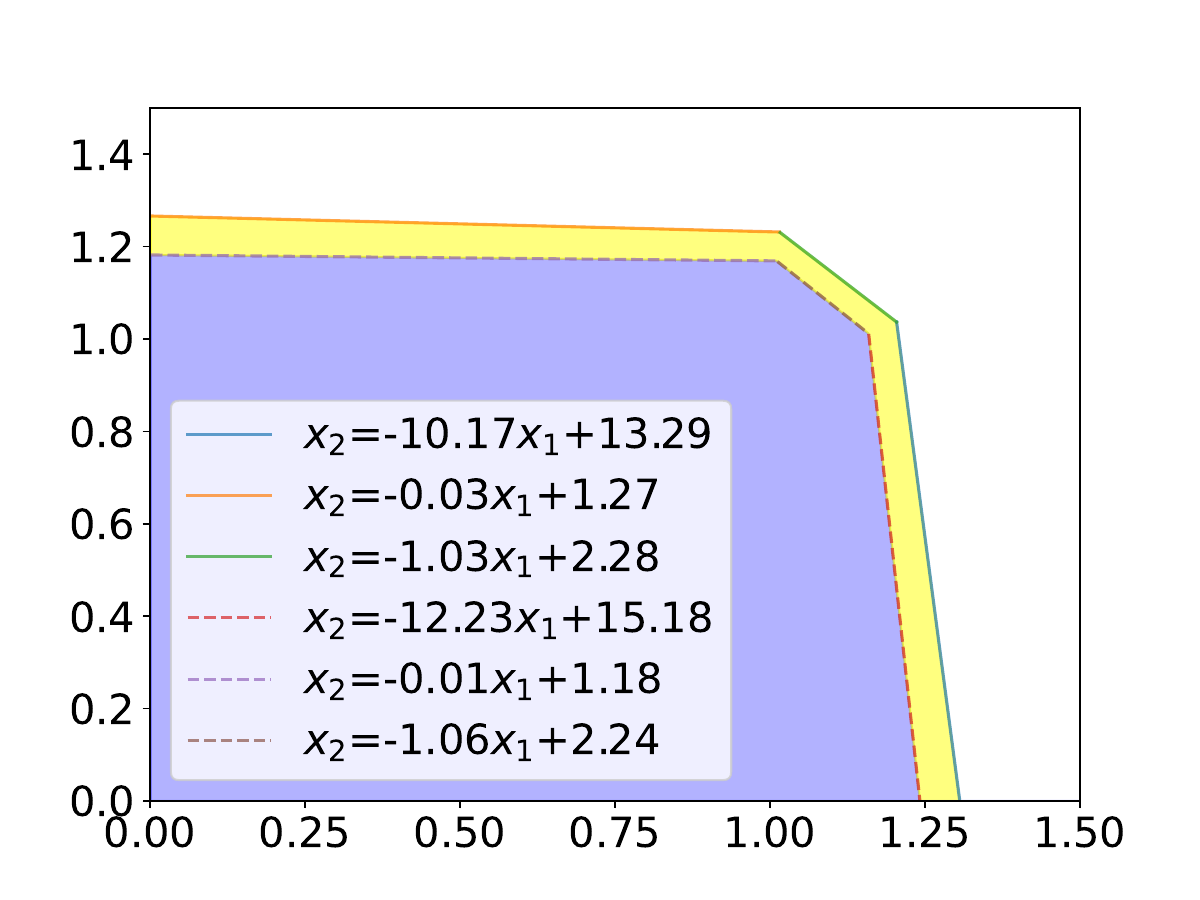}
         \caption{Optimised over-approximation}
         \label{fig:after_optim_over}
     \end{subfigure}
        \caption{\rev{Local optimisation for preimage under- and over-approximations. In Figure \ref{fig:after_optim_under}, the blue polytope represents the preimage under-approximation before optimisation and the yellow region illustrates the expanded polytope volume after optimisation.
        In Figure \ref{fig:after_optim_over}, the blue polytope represents the over-approximation before optimisation. The yellow region shows the reduced polytope volume after optimisation.}}
        \label{fig:local_optim_under_over}
\end{figure}
\begin{example}
\label{eg:optim}
    We revisit 
    Example~\ref{eg:formulation}.
    \rev{Figure \ref{fig:after_optim_under} shows the computed under-approximations before and after local optimisation.} 
    We can see that the bounding planes for all three specifications are optimised, \bw{such that the volume of the approximation has increased. \rev{Similarly, in Figure \ref{fig:after_optim_over} we show the over-approximations before and after optimisation;} it can be seen that the volume of the over-approximation has decreased. 
    }
\end{example}

\xy{\subsection{Optimisation of Lagrangian Relaxation }\label{sec:dual}

\bw{Previously, in Section \ref{sec:branching}, we proposed a preimage refinement method that adds intermediate ReLU splitting planes to tighten the bounds of a selected individual neuron.}
\bw{However, intermediate bounds for other neurons are not updated based on the newly added splitting constraint.} 
In the following, 
we first discuss the impact of stabilising an intermediate ReLU neuron from two different perspectives.
We then present an optimisation approach leveraging Lagrangian relaxation to enforce the splitting constraint on refining the preimage.

\textbf{Effect of Stabilisation of Intermediate Neurons.}
Our previous approach of \cite{zhang23preimage} exploits one level of bound tightening after ReLU splitting: the substitution of relaxation functions with exact linear functions for the individual neuron. 
Specifically, assume an intermediate (unstable) neuron $\intervar$ ($=\preact^{(i)}_j(\inpoint)$) is selected to split the input (sub)region $\indomain$ into two subregions $\indomain^{+}_{\intervar}=\{\inpoint \in \indomain~|~\intervar \ge 0\}$
and $\indomain^{-}_{\intervar}=\{\inpoint \in \indomain~|~\intervar < 0\}$.
For each subregion, the linear bounding functions of the nonlinear activation function $\postact^{(i)}_j(\intervar)$, as shown in Figure \ref{fig:linear_relaxation} (unstable mode), are then substituted with the exact ones, eliminating relaxation errors on the particular neuron.
Another effect, potentially more impactful, is the bound tightening of every other intermediate neuron.
Intuitively, one can tighten the intermediate bounds \bw{of (and thus stabilise)} the other unstable neurons, since we are \bw{restricted to} a smaller input region with the added splitting plane.
A straightforward solution to enforce the effect of the splitting constraint is to call a regular LP solver to compute the new lower and upper bounds for every intermediate ReLU neuron under the splitting constraints.
Naturally, this is computationally expensive ($2N$ LP calls where $N$ is the number of ReLU neurons).


\textbf{Refinement with Optimisation of Lagrangian Relaxation.} 
In order to derive tighter preimage approximations without explicitly introducing LP solver calls, we propose to \bw{adapt}
Lagrangian optimisation techniques \bw{\citep{wang2021beta} to preimage generation.}

\bw{Consider first the case of generating under-approximations}.
Without loss of generality, we focus on preimage generation for the $k$-th output specification constraint, $\underline{\spec_k}(\inpoint) = \lowerweightsingle_k^T \inpoint + \lowerbiassingle_k$. We will drop the subscript $k$ for simplicity.

Consider the subregion where we have $\intervar \le 0$.
To tighten the bounding plane $\underline{\spec}(x)$ of the preimage under the splitting constraint $\intervar \le 0$, we introduce the Lagrange multiplier, parameterized as $\dualvar (\ge 0)$, to enforce its effect throughout the neural network. 
\rev{Specifically, let us write 
 $g(\bm{z}^{(i)})$ to denote the neural network function mapping from the pre-activation neurons $\bm{z}^{(i)}$ of layer $i$ to the output.  In the typical backward LiRPA propagation through layer $i$, we have:}
\begin{equation}
    g(\bm{z}^{(i)}) \geq \bm{\lowerweight}^{(i)} \bm{z}^{(i)}  + \lowerbias^{(i)}
\end{equation}
Now, we add the splitting constraint using a Lagrange multiplier and obtain a Lagrangian relaxation of the original problem as follows:
\begin{equation}\label{eq:lagrangian}
    g(\bm{z}^{(i)}) \geq \max_{\dualvar \ge 0} \,\bm{\lowerweight}^{(i)} \bm{z}^{(i)}  + \lowerbias^{(i)} + \dualvar\intervar 
\end{equation}
Note that 
$\max_{\dualvar \ge 0} \dualvar\intervar = 0$, and thus Equation \ref{eq:lagrangian} holds in the universally quantified region. 
For the other case where $\intervar \ge 0$, we can obtain a sound lower bound similarly by changing the sign for the additional splitting constraint:
\begin{equation}\label{eq:lagrangian_pos}
    g(\bm{z}^{(i)}) \geq \max_{\dualvar \ge 0} \,\bm{\lowerweight}^{(i)} \bm{z}^{(i)}  + \lowerbias^{(i)} - \dualvar\intervar 
\end{equation}}

We then propagate this backwards through the network to obtain a valid lower bound with respect to the input layer $x$:
\begin{equation}\label{eq:lagrangian_input}
    g(x) \geq \underline{\spec}(x) = \max_{\nndual \ge 0} \,\lowerweight(\nnslope, \nndual) \inpoint + \lowerbias(\nnslope, \nndual)
\end{equation}
Here, we explicitly note the dependence of the linear coefficients on $\nndual$, which denotes the vector of $\dualvar$ introduced for all split neurons. \rev{These coefficients can be computed using a single (modified) LiRPA backward pass, where we add the Lagrange multipliers in each step as in Equations \ref{eq:lagrangian}, \ref{eq:lagrangian_pos}.}
Once we obtain the bounding plane for each half-space constraint in $\outset$, the preimage polytope can be formulated as $\polytope_{\indomain}(\outset) = \{\inpoint| \bigwedge_{i =1}^{\specnum} \underline{\spec_i}(\inpoint, \nnslope_i, \nndual_i)  \geq 0 \wedge \bigwedge_{i = 1}^{\indim} \boxcon_i(\inpoint) \}$.

Similarly to the optimisation over relaxation parameters $\nnslope$, we can then optimise $\bm{\beta}$ to maximise the preimage volume.
Our differentiable preimage volume estimate is given by:
\begin{equation}
\widehat{\volume}(\polytope_{\indomain_{sub}, \nnslope, \nndual}(\outset)) = 
\frac{\text{vol}(\indomain)}{\numsamples} \sum_{j = 1}^{\numsamples} \sigma\left(-\textnormal{LSE}( -\underline{\spec_1}(\inpoint_j, \nnslope_1, \nndual_1), ..., -\underline{\spec_\specnum}(\inpoint_j, \nnslope_\specnum, \nndual_\specnum))\right)
\end{equation}
where we have added the dependence on the Lagrange multipliers $\nndual$ to Equation \ref{eqnline:diff_approximation_alpha}. Intuitively, the additional splitting constraint enforced by the Lagrangian relaxation reduces the input space for maximising the preimage volume, which allows a tighter preimage bounding plane for the subregion.
In the case where all $\bm{\beta}$ coefficients are zero, this corresponds precisely to the 
previous standard LiRPA bound with $\nnslope$ parameters from Section \ref{sec:local_opt}. We then maximise the volume estimate of the under-approximation with the following loss function  in Algorithm \ref{alg:genunderapprox} (Line \ref{algline:loss_under}):
\begin{equation}
    \textnormal{Loss}(\nnslope, \nndual) = 
    - \widehat{\volume}(\underline{\polytope_{\indomain_{sub}, \nnslope, \nndual}}(\outset))
\end{equation}

\begin{tcolorbox}
[width=\linewidth, sharp corners=all, colback=white!95!black, frame empty]

\paragraph{Over-Approximation} In the case of an over-approximation (Line \ref{algline:loss_over} of Algorithm \ref{alg:genunderapprox} ), we instead aim to minimise the volume of the approximation:
\begin{equation}
    \textnormal{Loss}(\nnslope, \nndual) = 
    \widehat{\volume}(\overline{\polytope_{\indomain_{sub}, \nnslope, \nndual}}(\outset))
\end{equation}

\end{tcolorbox}

\begin{figure}[t]
     \centering
         \includegraphics[width=0.4\textwidth]{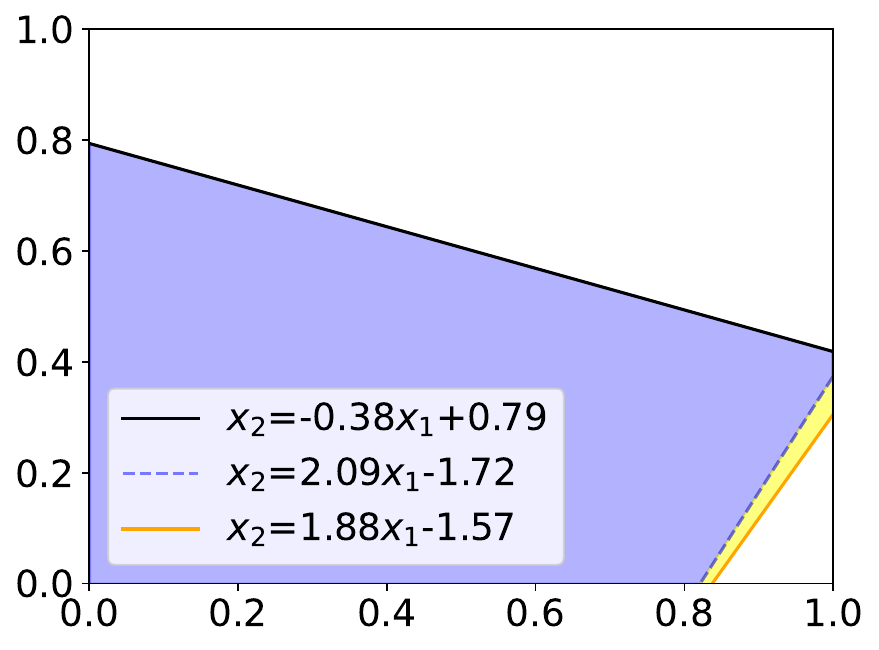}
        \caption{Optimisation of Lagrangian relaxation for preimage refinement. The preimage polytope in blue represents the under-approximation before Lagrangian optimisation and the yellow region displays the expanded polytope after optimisation.
        Details in Example \ref{eg:dual_optim}.
        }
        \label{fig:dual_optim}
\end{figure}

\begin{example}
\label{eg:dual_optim}
We now apply our  
optimisation method over Lagrangian relaxation to
Example~\ref{eg:formulation}.
\rev{Figure \ref{fig:dual_optim} shows the preimage polytope before and after Lagrangian optimisation, respectively, where the splitting plane of the selected unstable ReLU node is marked with a black solid line.
Note that the blue polytope before Lagrangian optimisation in Figure \ref{fig:dual_optim} is computed
by removing the relaxation errors of the selected unstable ReLU node, where the symbolic upper/lower bounding functions are substituted with the exact linear functions.
The preimage is further refined by enforcing the added splitting constraint $\intervar \le 0$ for one subdomain throughout the neuron network, which allows tighter preimage approximation (vs the blue polytope tightened via stabilizing a single neuron).}
\end{example}

\subsection{Overall Algorithm}\label{sec:overall}

Our overall preimage approximation method is summarised in Algorithm \ref{alg:main}. It takes as input a neural network $f$, input region $\indomain$, output region $\outset$, target polytope volume threshold $v$ (a proxy for approximation precision), maximum number of iterations $\maxiterations$, number of samples $\numsamples$ for statistical estimation, and Boolean variables indicating (i) whether to return an under-approximation or over-approximation and (ii) whether to use input or ReLU splitting, and returns a disjoint polytope union 
\bw{$\polytopeset_{\text{Dom}}$ representing a guaranteed under-approximation (or over-approximation) to the preimage.}

The algorithm initiates and maintains a priority queue of (sub)regions according to Equation~\ref{eqn:priority_under}.
\xy{The \textit{initialisation} step (Lines \ref{algline:initial_polytope}-\ref{algline:init_queue}) generates an initial polytope approximation of the whole region using Algorithm \ref{alg:genunderapprox} (Sections \ref{sec:poly_gen}, \ref{sec:local_opt}, \ref{sec:dual})}\bw{, with priority calculated (\texttt{CalcPriority}) according to Equations \ref{eqnline:priority_under_volume}, \ref{eqnline:priority_over_volume}}. Then, the \textit{preimage refinement} loop (Lines \ref{algline:while}-\ref{algline:refine_end}) partitions a subregion in each iteration, with the preimage restricted to the child subregions then being re-approximated (Line \ref{algline:split_node_id}-\ref{algline:subdomain_polytope}). In each iteration, we choose the region to split (Line \ref{algline:refine_start}) and the splitting plane to cut on (Line \ref{algline:input_selection} for input split and Line \ref{algline:relu_selection} for ReLU split), as
detailed in Section \ref{sec:branching}. The preimage subregion queue is then updated by computing the priorities for each subregion by approximating their volume (Line \ref{algline:update_approximation}). The loop terminates and the approximation is returned when the target volume \bw{threshold $v$} or maximum iteration limit $R$ is reached. 



\subsection{Quantitative Verification} \label{sec:verif}

\xy{We now show how to use our efficient preimage under-approximation method (Algorithm \ref{alg:main}) to verify a given quantitative property $(\inset, \outset, \proportion)$, where $\outset$ is a polyhedron, $\inset$ a polytope and $p$ the desired proportion threshold, summarised in Algorithm \ref{alg:verify}. 
Note that preimage over-approximation cannot be applied for sound quantitative verification as the approximation may contain false regions outside the true preimage. 
To simplify, assume that $\inset$ is a hyperrectangle, so that we can take $\indomain = \inset$.  
We discuss the case of general polytopes at the end of this section.}
\begin{algorithm}[tb]
\caption{Quantitative Verification}\label{alg:verify}
\KwIn{Neural network $f$, Property $(\inset, \outset, \proportion)$, Maximum iterations $\maxiterations$}
\KwOut{Verification result $\in \{\textnormal{True, False, Unknown}\}$}

$\volume(\inset) \gets \textnormal{EstimateVolume}(\inset)$\;
$\indomain \gets \textnormal{OuterBox}(\inset)$ \label{algline:outerbox}\tcp*{For general polytopes $\inset$}
$\polytopeset \gets \textnormal{InitialRun}(f, \indomain, \outset)$\;
\While{$\textnormal{Iterations} \leq \maxiterations$} {
    $\polytopeset \gets \textnormal{Refine}(f, \polytopeset, \indomain,\outset)$\;
    \If{$\textnormal{EstimateVolume}(\polytopeset) \geq \proportion \times \volume(\inset)$}{
    \KwRet{\textnormal{True}}
    }
    \If{\textnormal{AllReLUSplit}}{
    \KwRet{\textnormal{False}}
    }
}
\KwRet{\textnormal{Unknown}}
\end{algorithm}
We utilise Algorithm \ref{alg:main} by setting the \rev{volume threshold $v$} to $\proportion \times \volume(\inset)$, such that we have $\frac{\volume(\polytopeset)}{\volume(\inset)} \geq \proportion$ if the algorithm terminates before reaching the maximum number of iterations. 
If the final preimage polytope volume $\volume(\polytopeset) \geq \proportion \times \volume(\inset)$, then the property is verified. Otherwise, we continue running the preimage refinement. 
If the refinement loop has stabilised all ReLU neurons and the volume threshold is still not achieved, the property is falsified.

In Algorithm \ref{alg:verify}, \texttt{InitialRun} generates an initial under-approximation to the preimage as in Lines \ref{algline:initial_polytope}-\ref{algline:init_queue} of Algorithm 1, and \texttt{Refine} performs one iteration of approximation refinement (Lines \ref{algline:refine_start}-\ref{algline:refine_end}). Termination occurs when we have verified or falsified the quantitative property, or when the maximum number of iterations has been exceeded.


\begin{restatable}{proposition}{propSound}\label{prop:sound}
    Algorithm \ref{alg:verify} is sound for quantitative verification with input splitting.
\end{restatable}

\begin{restatable}{proposition}{propComplete}\label{prop:complete}
    Algorithm \ref{alg:verify} is sound and complete for quantitative verification on piecewise linear neural networks with ReLU splitting.
\end{restatable}

\xy{Proofs of the propositions are presented in Appendix \ref{app:proofs}.}

\xy{\textbf{General Input Polytopes.}
Previously we detailed how to use our preimage under-approximation method to verify quantitative properties $(\inset, \outset, \proportion)$, where $\inset$ is a hyperrectangle. 
We now discuss how to extend our method for a general polytope $\inset = \{x \in \mathbb{R}^d | \bigwedge_{i=1}^{K_{in}} c_i^Tx + d_i \geq 0\}$, (where $\psi_i$ are half-space constraints). \rev{Intuitively, we can handle these general input polytopes (i) by finding a bounding hyperrectangle $\indomain$; and (ii) by encoding the polytope half-planes as additional constraints defining the preimage. Assuming that the output polytope $\outset$ is given as $g_i(x) \geq 0$ for $i = 1, \ldots, K_{out}$, then we have:
\begin{equation}
\nn^{-1}_{\indomain}(\outset) \cap I := \left\{\inpoint \in \mathbb{R}^{\indim} \Big\vert \bigwedge_{i=1}^{K_{out}} g_i(x) \geq 0 \wedge \inpoint \in \indomain \wedge \bigwedge_{i=1}^{K_{in}} c_i^Tx + d_i \geq 0 \right\}
\end{equation}
}

Firstly, in Line \ref{algline:outerbox} of Algorithm \ref{alg:verify},  we derive a hyperrectangle $\indomain$ such that $\inset \subseteq \indomain$, by converting the polytope $\inset$ into its \textit{V-representation} \citep{grunbaum2003polytope}, that is, a list of the vertices (extreme points) of the polytope, which can be computed as in \cite{avis1991pivoting,Barber96Qhull}. 
Once we have a V-representation, obtaining a bounding box (hyperrectangle) can be achieved simply by computing the minimum and maximum value $\underline{\inpoint_i}, \overline{\inpoint_i}$ of each dimension among all vertices.  

\rev{Once we have this input hyperrectangle $\indomain$, we can then run the preimage refinement as usual, but with the modification that, when defining the polytopes and restricted preimages, we must additionally include the polytope constraints from $\inset$. 
Practically, this means that, during every call to \texttt{EstimateVolume} in Algorithm \ref{alg:verify}, we add these polytope constraints, and in Line \ref{algline:append_under} of Algorithm \ref{alg:genunderapprox} we add the polytope constraints from $\inset$, in addition to those derived from the output $\outset$ and the box constraints from $\indomain_{sub}$. The output will be a DUP under/over-approximation of the preimage intersected with the polytope $I$.}}

\section{Experiments}
\label{sec:evaluation}
We have implemented our approach as a tool~\citep{premap2025}
for preimage approximation for 
polyhedral
output sets/specifications. 
In this section, 
we report on experimental evaluation of the proposed approach, 
and demonstrate its 
effectiveness in approximation generation 
and the application to
quantitative analysis of neural networks.

\subsection{Benchmark and Evaluation Metric}\label{sec:eval_metric}
We evaluate PREMAP on a benchmark of reinforcement learning and image classification tasks. 
Besides the vehicle parking task of \cite{Ayala11vehicle} shown in the running example,
we \bw{consider the following tasks:}
(1) aircraft collision avoidance system (VCAS) from \cite{Julian19nncontrol} with 9 feed-forward neural networks (FNNs); \rev{(2) neural network controllers from \cite{vnn22} for three reinforcement learning tasks (Cartpole, Lunarlander, and Dubinsrejoin) as in \cite{Brockman16}; and (3) the neural network  for MNIST classification from VNN-COMP 2022~\citep{vnncomp2022}.}
Details of the benchmark tasks and neural networks  are summarised in Appendix \ref{app:exp_setup}.

\textbf{Evaluation Metric.}
To evaluate the quality of the preimage approximation, 
we define the \textit{coverage ratio} to be the ratio of volume covered \bw{by the approximation} to the volume of the exact preimage,
i.e., 
    $\cov(\polytopeset, \preimage_{\indomain}(\outset)) := \frac{\volume(\polytopeset)}{\volume(\preimage_{\indomain}(\outset))}$.
Note that this is a normalised measure for assessing the quality of the approximation, 
as \bw{used} in Algorithm~\ref{alg:verify} when comparing with target coverage proportion $\proportion$ for termination of the refinement loop.
In practice, we use Monte Carlo estimation to compute $\volume(\preimage_{\indomain}(\outset))$ 
as $\widehat{\volume}(\preimage_{\indomain}(\outset)) = \text{vol}(\indomain)\times \frac{1}{\numsamples} \sum_{i=1}^{\numsamples} \mathds{1}_{\nn(\inpoint_i) \in \outset}$, where $\inpoint_1, ... \inpoint_{\numsamples}$ are samples from $\indomain$.
In Algorithm \ref{alg:main}, 
the target volume  (stopping criterion) is set as $\threshold = \targetcov \times \widehat{\volume}(\preimage_{\indomain}(\outset)$, where $\targetcov$ is the \textit{target coverage ratio}. \rev{Thus, in our experimental results, the coverage ratio reported will be greater than (resp. less than) the target coverage ratio for an under-approximation (resp. over-approximation), unless the maximum number of iterations is reached.}

\subsection{Evaluation}\label{sec:eval}
\subsubsection{Effectiveness on Preimage Approximation with Input Split}
\newcolumntype{g}{>{\columncolor{Gray}}c}
\begin{table}[tb]
\caption{Performance comparison on preimage generation\bw{, for four different specifications on the vehicle parking task}. Over-approximation results are highlighted
with \textit{grey} background in subcolumns labelled by \textbf{ox},  whereas under-approximation is shown with \textit{white} background in subcolumns labelled by \textbf{ux}.}\label{tab:vehicle}
\centering
\scriptsize
\begin{tabular}{c|cc| c g |c g | c g | c g | c g} 
 \toprule
        \multirow{3}{*}{\makecell{\textbf{Vehicle}\\ \textbf{Parking}}} 
        & \multicolumn{2}{c|}{\textbf{Exact}}
         & \multicolumn{4}{c|}{\textbf{Invprop}} 
        & \multicolumn{6}{c}{\textbf{PREMAP}}\\
\cmidrule{2-13}
        &\textbf{\#Poly} & \textbf{Time(s)}
          & \multicolumn{2}{c}{\textbf{Time(s)}}
        & \multicolumn{2}{c|}{\textbf{Cov}}  
         &\multicolumn{2}{c}{\textbf{\#Poly}} 
         & \multicolumn{2}{c}{\textbf{Time(s)}}
        & \multicolumn{2}{c}{\textbf{Cov}} \\   
\cmidrule{2-13}
        & \textbf{exact} & \textbf{exact} & \textbf{ux} & \textbf{ox} & \textbf{ux} & \textbf{ox} &\textbf{ux} & \textbf{ox}& \textbf{ux} & \textbf{ox} & \textbf{ux} & \textbf{ox} \\
\midrule
 $\wedge_{i \in \{2,3,4\}} y_1 \ge y_i$
& 10 & 3110.979
& 2.642 & \textbf{0.907} & 0.921 & \textbf{1.043}
& \textbf{4}  & \textbf{4} & \textbf{1.116} & 1.121  & \textbf{0.957} & 1.092 \\ 
 $\wedge_{i \in \{1,3,4\}} y_2 \ge y_i$
& 20 & 3196.561
& 2.242 & \textbf{0.793} & 0.895 & \textbf{1.051}
& \textbf{4}  & \textbf{4} & \textbf{1.235} & 1.336 & \textbf{0.948} & 1.074\\ 
 $\wedge_{i \in \{1,2,4\}} y_3 \ge y_i$
& 7 & 3184.298
& 2.325 & \textbf{0.865} & 0.906 & \textbf{1.083}
& \textbf{3} & \textbf{4} & \textbf{1.074} & 1.129 & \textbf{0.952} & 1.098\\ 
 $\wedge_{i \in \{1,2,3\}} y_4 \ge y_i$
& 15 & 3206.998
& 2.402 & \textbf{0.793} & 0.915 & \textbf{1.058}
& \textbf{3}  & \textbf{3} & \textbf{1.055} & 1.004 & \textbf{0.922} & 1.061\\ 
\bottomrule
\end{tabular}
\end{table}

We apply Algorithm~\ref{alg:main}
with input splitting to 
the 
\bw{preimage approximation problem} for
low-dimensional reinforcement learning tasks. 
For comparison, we also run the exact preimage generation method (Exact) from \cite{Matoba20Exact} and the preimage over-approximation method (Invprop) from \cite{ProveBound,invprop}.

\newcolumntype{g}{>{\columncolor{Gray}}c}
\begin{table}[tb]
\caption{Performance comparison on preimage generation (average performance) \bw{on vehicle parking and VCAS}, with
over-approximation shown in \textit{grey} background (subcolumns labelled by \textbf{ox}) and under-approximation in \textit{white} background (subcolumns labelled by \textbf{ux}).}\label{tab:compare_baseline}
\centering
\scriptsize
\begin{tabular}{c|c c|c g c g |c g c g c g} 
\toprule
\multirow{3}{*}{\textbf{Tasks}}  
        & \multicolumn{2}{c|}{\textbf{Exact}}
        & \multicolumn{4}{c|}{\textbf{Invprop}}
         & \multicolumn{6}{c}{\textbf{PREMAP}} \\
\cmidrule{2-13}
         & \textbf{\#Poly} & \textbf{Time(s)} 
        & \multicolumn{2}{c}{\textbf{Time(s)}} & \multicolumn{2}{c|}{\textbf{Cov}} 
        & \multicolumn{2}{c}{\textbf{\#Poly}} & \multicolumn{2}{c}{\textbf{Time(s)}} & \multicolumn{2}{c}{\textbf{Cov}} \\ 
\cmidrule{2-13}
        & \textbf{exact} & \textbf{exact} & \textbf{ux} & \textbf{ox} & \textbf{ux} & \textbf{ox} &\textbf{ux} & \textbf{ox}& \textbf{ux} & \textbf{ox} & \textbf{ux} & \textbf{ox} \\
\midrule
Vehicle  & 13 & 3174.709 & 2.403 & \textbf{0.840} & 0.909 & \textbf{1.059} & \textbf{4} & \textbf{4} & \textbf{1.120} & 1.148 & \textbf{0.945} & 1.081\\ 
\midrule
VCAS 
& 131 & 6363.272 
 & - & - & - & -
& \textbf{15} & \textbf{1} & \textbf{10.775} & \textbf{1.045} & \textbf{0.908}& \textbf{1.041}\\
\bottomrule
\end{tabular}
\end{table}
\textit{Vehicle Parking \& VCAS.}
Table~\ref{tab:vehicle} and~\ref{tab:compare_baseline} present the comparison results with state-of-the-art exact and approximate preimage generation methods.
In the table, we show the number of polytopes (\#Poly) in the preimage, computation time (Time(s)), and
the approximate coverage ratio (Cov) when the preimage approximation algorithm terminates with target coverage of 0.90 (the larger, the better) for under-approximation and 1.10 (the lower, the better) for over-approximation.
Note that the \textit{Exact} method computes the exact preimage (i.e., coverage ratio 1.0), while PREMAP computes the under- and over-approximation of the exact preimage.
The results for over-approximation are highlighted with \textit{grey} background, whereas under-approximation is shown with \textit{white} background.
\bw{Invprop only supports computing over-approximations natively; thus,}
we adapt it to produce an under-approximation by computing over-approximations for the complement of each output constraint; note that the resulting approximation is then the \textit{complement} of a union of polytopes, rather than a DUP.

Compared with the exact method, 
our approach yields \textit{orders-of-magnitude} improvement in efficiency (see Table \ref{tab:vehicle} and Table \ref{tab:compare_baseline}). 
It can also characterise the preimage with much fewer (and also disjoint) polytopes, achieving an average reduction of 69.2\% for vehicle parking (both under- and over-approximation) and 88.5\% (under-approximation) and 99.2\% (over-approximation) for VCAS.
Compared with Invprop, our method produces comparable results in terms of time and approximation coverage on the 2D vehicle parking task. 
\rev{In Table \ref{tab:compare_baseline}, Invprop provides a tighter over-approximation and lower runtime for the vehicle parking task.
This advantage can be attributed to Invprop's exploitation of output constraints for iterative refinement of intermediate bounds, which improves precision and reduces the number of splitting rounds needed. 
However, the iterative refinement procedure utilising output constraints in Invprop is quite expensive. It essentially has a quadratic dependence on network depth when all intermediate bounds are refined because refining the bounds of each intermediate layer requires bound propagation to the input layer. In contrast, our framework utilises output constraints with a linear dependence with respect to the number of layers.
As a result, Invprop's scalability to deeper neural networks remains a limitation, especially in more complex tasks or architectures with greater depth.}
%
\begin{figure}[t]
\centering
\includegraphics[width=0.7\textwidth]{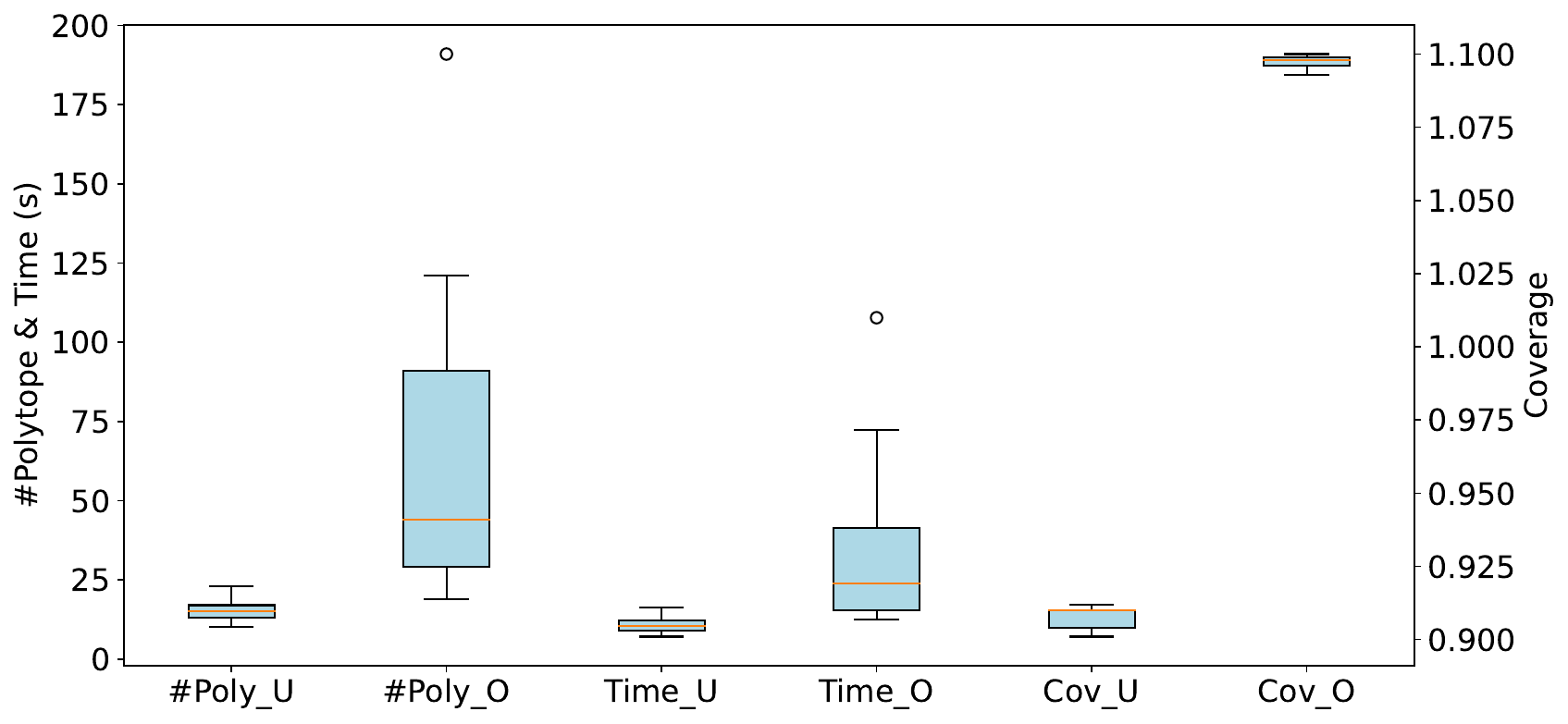}
\caption{Preimage approximation \bw{results over the 9 VCAS neural networks}.
Results for under-approximation are indicated with \textit{metric}\_U and over-approximation with \textit{metric}\_O. 
The scale for both the number of polytopes and time is indicated on the vertical axis on the left.
} 
\label{fig:box_vcas}
\end{figure}

\bw{While Table \ref{tab:compare_baseline} shows average performance on VCAS}, Figure~\ref{fig:box_vcas} plots more detailed results of our method for the nine neural networks in the VCAS task in terms of the number of polytopes ($y$-axis on the left), time cost ($y$-axis on the left) and approximation coverage ($y$-axis on the right) for both under- (indicated with \textit{metric}\_U) and over-approximation (indicated with \textit{metric}\_O).
As shown in the figure, our method is able to reach the targeted approximation coverage (0.90 for under-approximation and 1.10 for over-approximation) for all networks.
The median number of polytopes for the preimage under-approximation for property $\outset = \{y\in \mathbb{R}^9~|\wedge_{i \in [1,9]}~y_1 \ge y_i\}$ is  15 and the median time cost is 10.492s. 
The over-approximation shows higher variability in the number of generated polytopes and computation time for property $\outset = \{y\in \mathbb{R}^9~|\wedge_{i \in [1,9]\setminus 3}~y_3 \ge y_i\}$,
with the maximum reaching 191 polytopes and computation time  of 107.758s for VCAS model 3. 


\newcolumntype{g}{>{\columncolor{Gray}}c}
\begin{table}[t]
\caption{Performance of preimage approximation for reinforcement learning tasks, with
over-approximation shown in \textit{grey} background (marked in subcolumns \textbf{ox}) and under-approximation in \textit{white} background (marked in subcolumns \textbf{ux}).}\label{tab:rl-tasks}
\centering
\scriptsize
\begin{tabular}{cc|c|c g|c g|c g} 
\toprule
\multirow{2}{*}{\textbf{Task}} & \multirow{2}{*}{\textbf{Property}} & \multirow{2}{*}{\textbf{Config}} & \multicolumn{2}{c|}{\textbf{\#Poly}} & \multicolumn{2}{c|}{\textbf{Cov}} & \multicolumn{2}{c}{\textbf{Time(s)}} \\ 
\cmidrule{4-9}
& & & \textbf{ux} & \textbf{ox} & \textbf{ux} & \textbf{ox} & \textbf{ux} & \textbf{ox}\\
\midrule
\multirow{3}{*}{\makecell{Cartpole\\ (FNN $2 \times 64$)}} & \multirow{3}{*}{$\{y\in \mathbb{R}^2|~ y_1 \ge y_2\}$} & $\dot{\theta} \in [-2,-1]$ & 25 & 1 & 0.766 & 1.213 & 13.337 & 2.149\\ 
& & $\dot{\theta} \in [-2,-0.5]$ & 42 & 8 & 0.750 & 1.242 & 19.732 & 5.778\\ 
& & $\dot{\theta} \in [-2,0]$ & 66 & 22 & 0.755 & 1.246 & 30.563 & 11.476\\ 
\midrule
\multirow{3}{*}{\makecell{Lunarlander \\ (FNN $2 \times 64$)}}
 & 
\multirow{3}{*}{$\{y\in \mathbb{R}^4| \wedge_{i \in \{1,3,4\}}  y_2 \ge y_i\}$} & $\dot{v} \in [-1, 0]$ & 18 & 1 & 0.754 & 1.068 & 14.453 & 2.381 \\ 
& & $\dot{v} \in [-2, 0]$ & 67 & 23 & 0.751 & 1.246  & 48.455 & 19.210\\ 
& & $\dot{v} \in [-4, 0]$ & 97 & 90 & 0.751 & 1.249 & 76.234 & 72.285 \\
\midrule
\multirow{3}{*}{\makecell{Dubinsrejoin\\ (FNN $2 \times 256$)}}
& 
\multirow{3}{*}{\makecell{$\{y\in \mathbb{R}^8|\wedge_{i \in [2,4]} ~ y_1 \ge y_i $ \\
$\quad \quad \quad \,\bigwedge \wedge_{i \in [6,8]} ~y_5 \ge y_i\}$}} & $x_v \in [-0.1,0.1]$ & 211 & 20 & 0.751 & 1.242 & 182.821 & 18.666\\ 
& & $x_v \in [-0.2,0.2]$ & 409 & 23 & 0.750 & 1.241 & 323.839 & 24.788\\
& & $x_v \in [-0.3,0.3]$ & 677 & 43 & 0.750 & 1.244 & 589.939 & 41.502\\
\bottomrule
\end{tabular}
\end{table}

\begin{figure}[t]
\centering
 \begin{subfigure}[b]{0.49\textwidth}
     \centering
     \includegraphics[width=\textwidth]{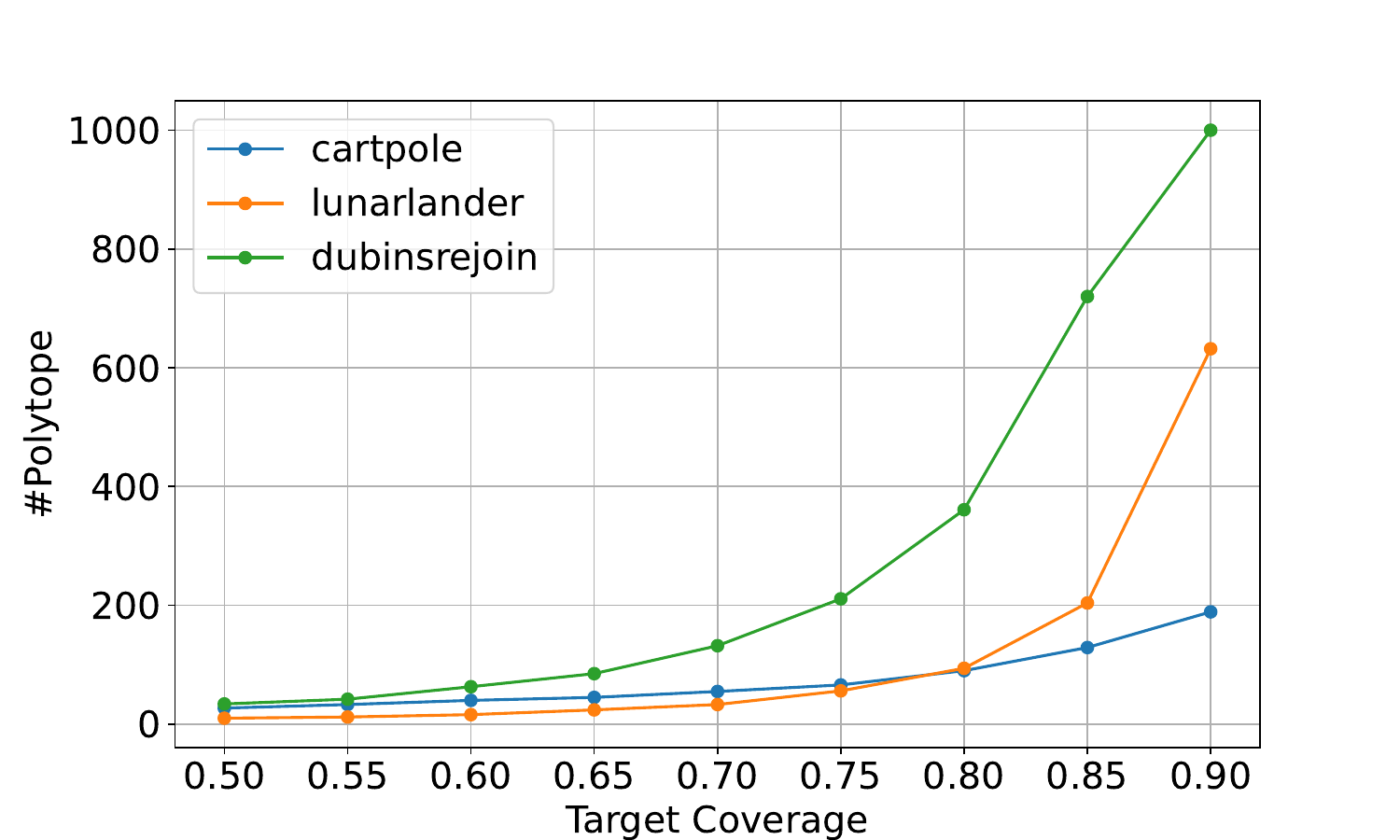}
 \end{subfigure}
 \begin{subfigure}[b]{0.49\textwidth}
     \centering
     \includegraphics[width=\textwidth]{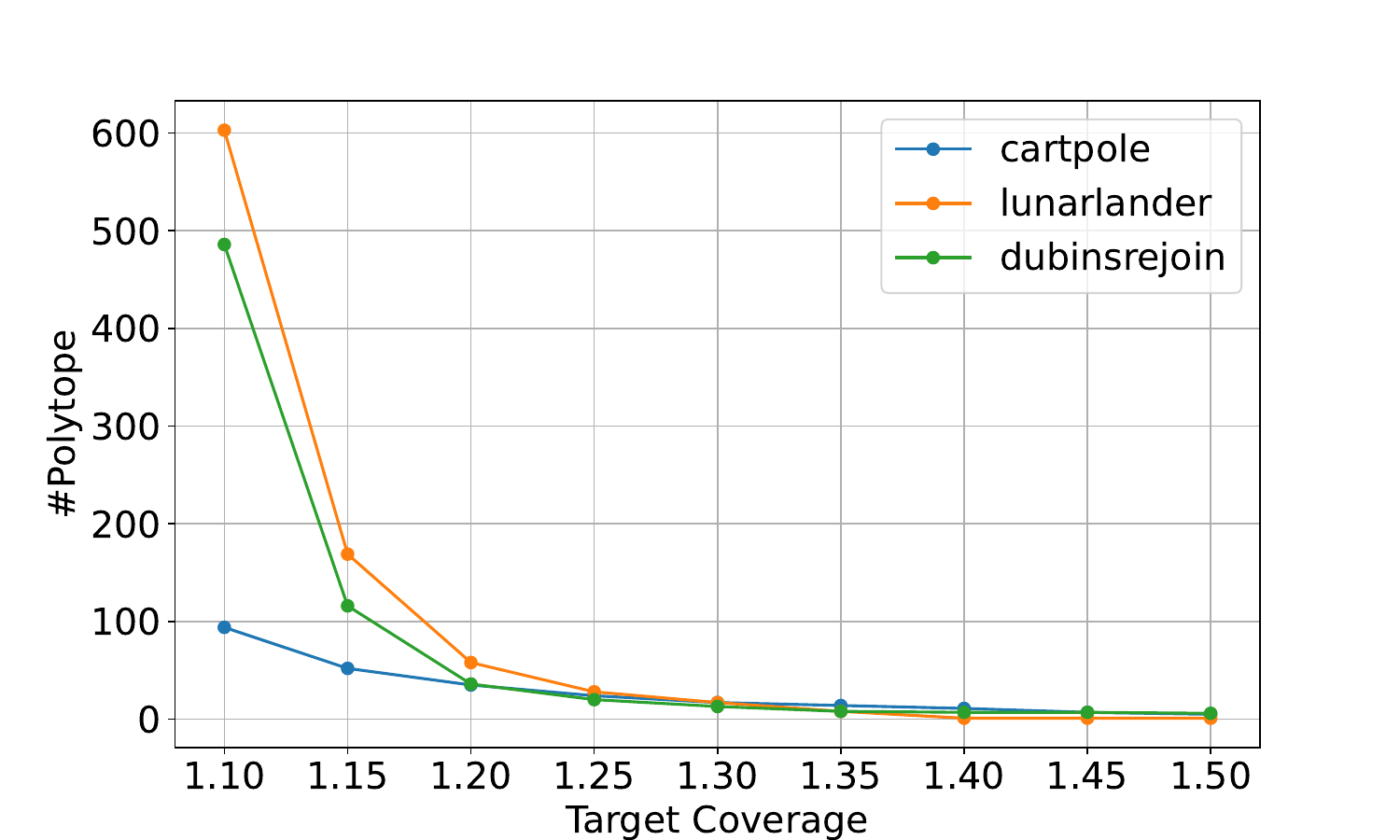}
 \end{subfigure}
 \caption{Number of polytopes in preimage approximation for a range of target coverage values. Left: under-approximation. Right: over-approximation.}\label{fig:poly_target_coverage}
\end{figure}

\textit{Neural Network Controllers.}
In this experiment, we consider preimage approximation for neural network controllers in reinforcement learning tasks. 
Note that the \textit{Exact} method in \cite{Matoba20Exact}  is unable to deal with neural networks of these sizes and Invprop in \cite{ProveBound}
\xy{is not capable of characterising the preimage under-approximation in the form of disjoint polytopes.}
Table~\ref{tab:rl-tasks} summarises the experimental results 
obtained by our method, where the columns for over-approximations are marked with grey background and under-approximations marked with a white background.

We evaluate 
Algorithm~\ref{alg:main} (with input splitting) with respect to \bw{a range of} different
configurations of the input region (e.g., angular velocity $\dot{\theta}$ for Cartpole). 
For comparison, we set the same target coverage ratio for different input region sizes (0.75 for under-approximation and 1.25 for over-approximation) and an iteration limit of 1000.
In Table \ref{tab:rl-tasks}, we see that our method \bw{successfully} generates preimage approximations for all configurations, reaching the targeted approximation coverage. 
Empirically, for the same coverage ratio, our method requires a number of polytopes and time roughly linear in the input region size for the preimage under-approximation.
For over-approximations, the bounding constraints of the input region are added as additional constraints to form the polytope approximation on each subregion,
which \bw{affects} 
the linear trend in the number of polytopes and computation time as the input region size increases.
For example, the constraint brought by the input configuration of $-2 \le \dot{\theta} \le -1$ for Cartpole, together with the single polytope over-approximation, already reaches the target coverage, while for a larger input region of $-2 \le \dot{\theta} \le 0$, 22 preimage polytopes are needed together with input bounding constraints for each input subregion.
\rev{Table \ref{tab:rl-tasks} also shows that the number of polytopes, and consequently the runtime cost, required for computing over-approximations is lower than for under-approximations. This difference can be attributed to how half-plane constraints are used to form the polytope. For over-approximation, the input bounds of each subregion act as additional half-plane constraints that further tighten the feasible set, resulting in fewer polytopes that reach the over-approximation precision. In contrast, the under-approximation imposes more strict constraints than the subregion bounds, as it is guaranteed to be included in the subregion. This often requires a finer splitting of the input space to achieve the target precision, resulting in a greater number of polytopes and increased computational effort.}

\xy{In Figure \ref{fig:poly_target_coverage}, we show the number of polytopes needed to reach different target coverage \bw{ratios} for both under-approximation (left) and over-approximation (right).
Our evaluation results indicate that the \bw{number of} refinement iterations taken is influenced by the number of output constraints and the size of the neural network. 
For instance, the neural network controller for Cartpole, which has a single output constraint, shows a roughly linear
increase in the number of polytopes as the target coverage increases for under-approximation (resp. decreases for over-approximation).
In contrast, accommodating multiple output constraints for larger neural networks, e.g., Dubinsrejoin, requires a significant increase in refinement iterations as the target coverage approaches 1.}
 
\begin{figure}[t]
\centering
 \begin{subfigure}[b]{0.45\textwidth}
     \centering
     \includegraphics[width=\textwidth]{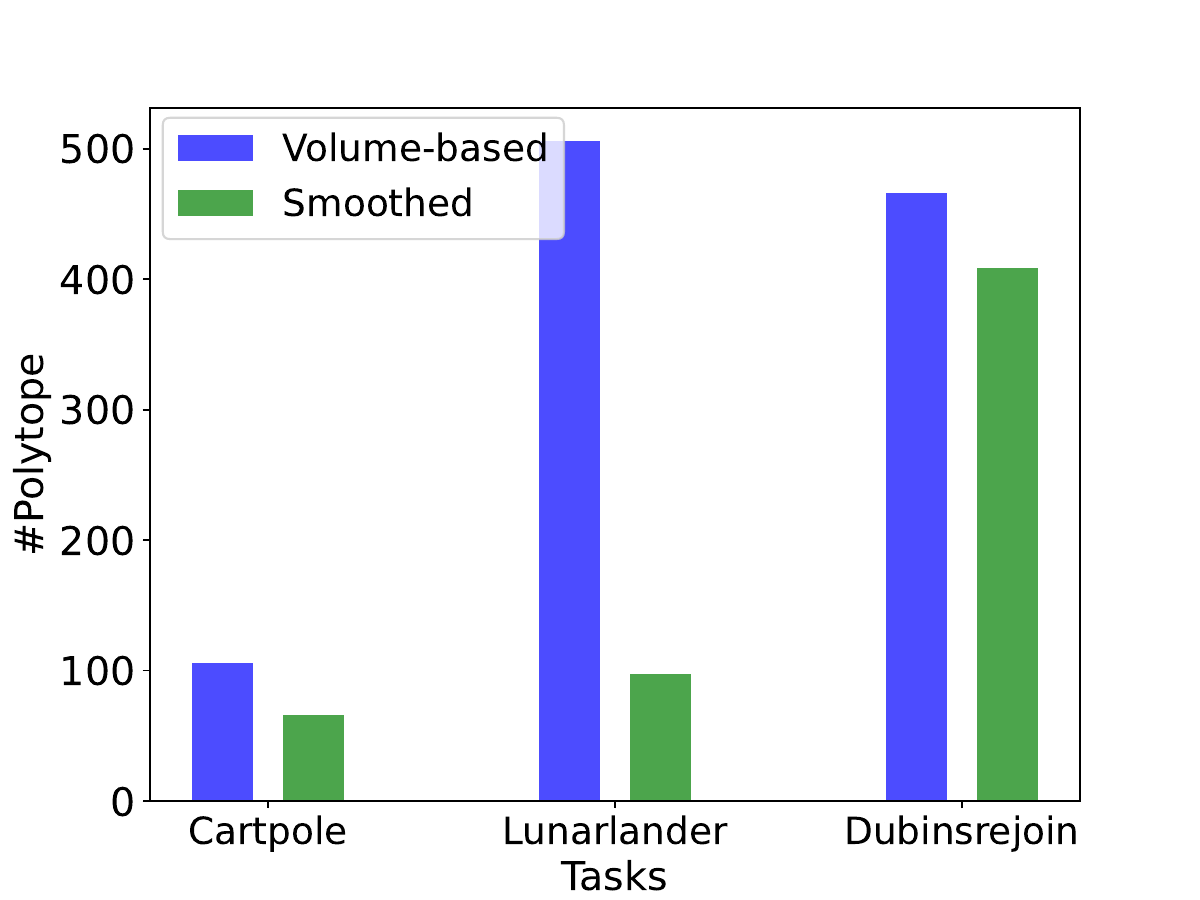}
     \caption{Number of polytopes}
     \label{fig:split_polytope_under}
 \end{subfigure}
 \begin{subfigure}[b]{0.45\textwidth}
     \centering
     \includegraphics[width=\textwidth]{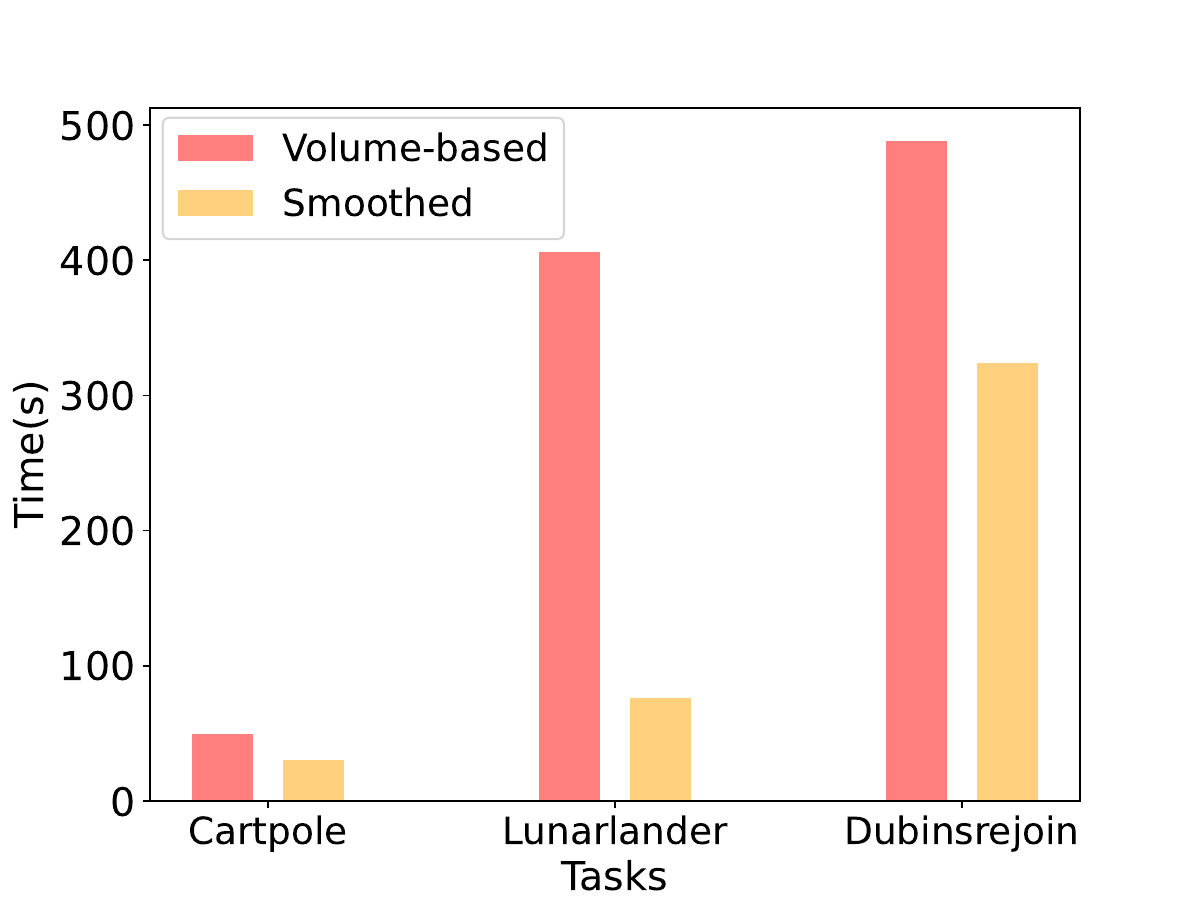}
     \caption{Time cost}
     \label{fig:split_time_under}
 \end{subfigure}
 \caption{Effectiveness of smooth splitting for preimage under-approximation. \rev{Comparison results with Volume-based method} from \cite{zhang23preimage}.}\label{fig:bar_under}
\end{figure}

\begin{figure}[t]
\centering
 \begin{subfigure}[b]{0.45\textwidth}
     \centering
     \includegraphics[width=\textwidth]{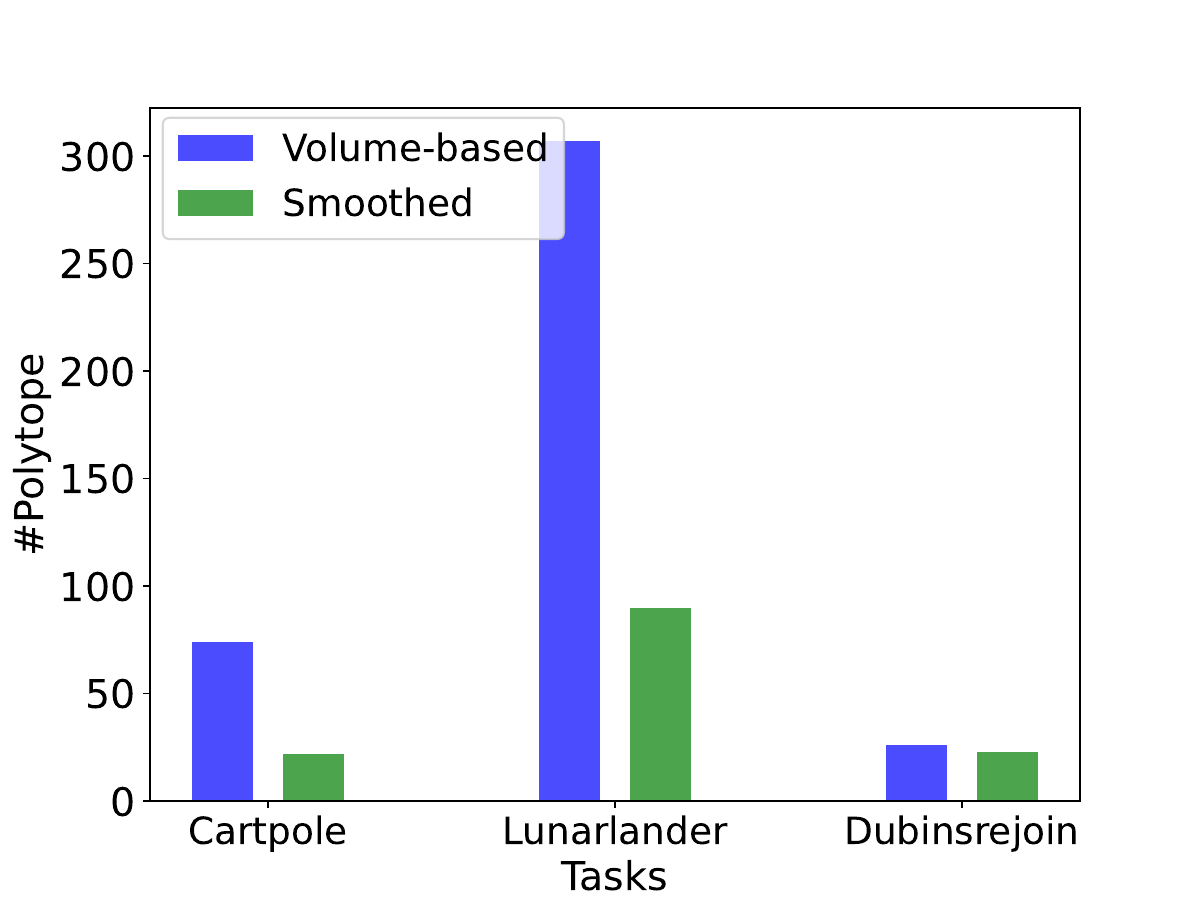}
     \caption{Number of polytopes}
     \label{fig:split_polytope_over}
 \end{subfigure}
 \begin{subfigure}[b]{0.45\textwidth}
     \centering
     \includegraphics[width=\textwidth]{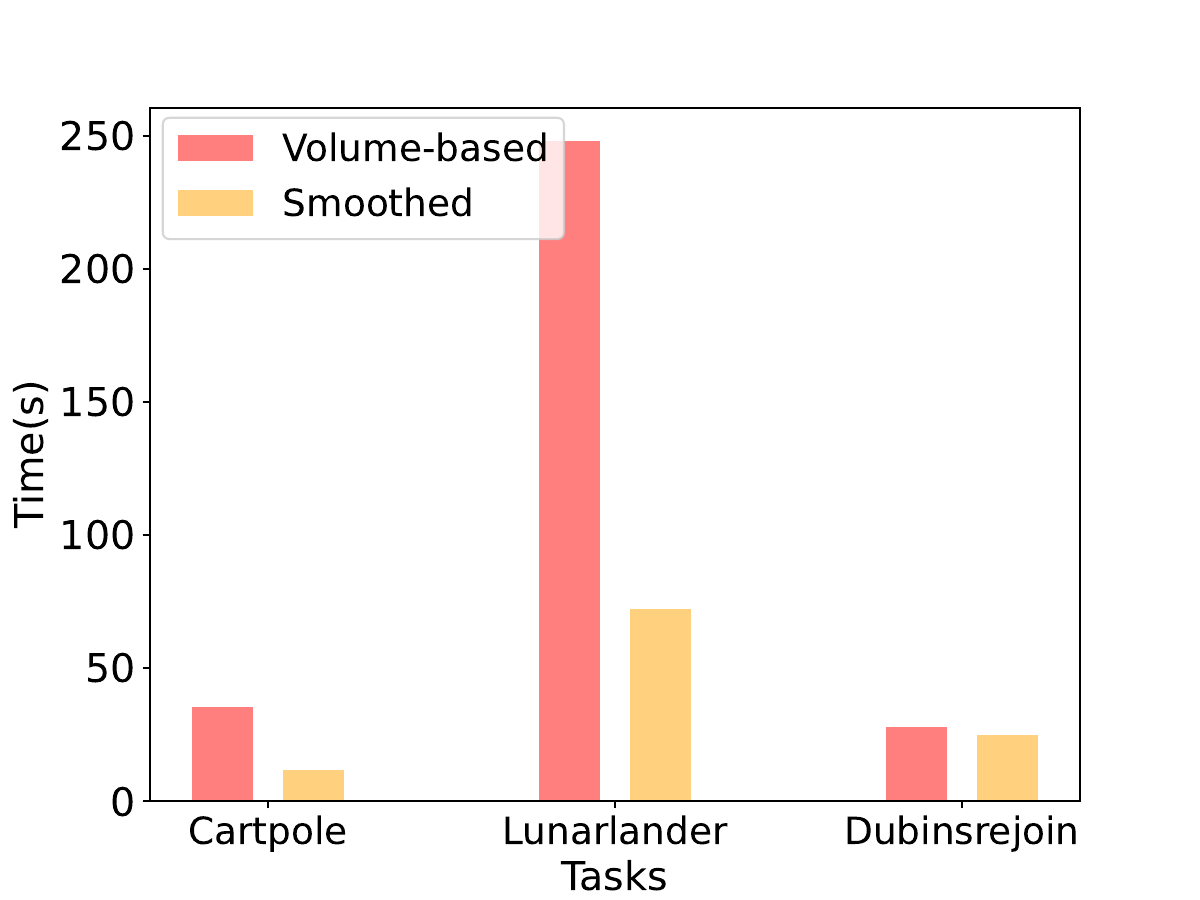}
     \caption{Time cost}
     \label{fig:split_time_over}
 \end{subfigure}
 \caption{Effectiveness of smooth splitting for preimage over-approximation. \rev{Comparison results with Volume-based method} from \cite{zhang23preimage}.} \label{fig:bar_over}
\end{figure}


\subsubsection{Effectiveness of Smoothed Input Splitting} \label{sec:smoothed_input_splitting_exp}
We now analyse the effectiveness of the smoothed splitting method \bw{described} in Section \ref{sec:branching} (Equation \ref{eqn:priority_under} and \ref{eqn:priority_over}), \rev{in comparison to a volume-guided splitting method in \cite{zhang23preimage} that chooses the input feature leading to the greatest improvement in approximation volume.} 
From Figures \ref{fig:bar_under} and \ref{fig:bar_over}, we observe that the smoothed splitting method 
requires \bw{significantly} fewer refinement iterations for all reinforcement learning controllers to achieve the target coverage, thus reducing the number of polytopes and computation time, than the volume-guided splitting method.
More specifically, the smoothed splitting method achieves an average reduction of 43.6\% in the number of polytopes and 51.0\% in computation time for under-approximation across the neural network controllers, up to 80.8\%/81.2\% reduction 
for the Lunarlander task. 
Similar improvements in computation efficiency and size of polytope union are also achieved for over-approximations, with an average reduction of 50.8\%/49.6\% across all reinforcement learning tasks.

\bw{Recall that the smoothed input splitting heuristic relaxes the volume-based heuristic, such that, for each sampled input point in the input region, we take into account not only \emph{whether} the point lies in the polytope approximation, but also \emph{how far away} the point is from the approximation.} 
\bw{This is particularly crucial in early iterations, where the approximation may be too loose; for example, an under-approximation may have no overlap with the input region (thus zero volume). Therefore, computing the approximation volume (after splitting on each input feature) provides very little signal.}
\bw{In such cases, the smoothed splitting heuristic is able to capture promising input features that, while not immediately improving the approximation volume, can bring the preimage bounding planes closer to the exact preimage, which is beneficial for future iterations.} 

\subsubsection{Effectiveness of Preimage Approximation with ReLU Split} \label{sec:relu_split_exp}

\newcolumntype{g}{>{\columncolor{Gray}}c}
\begin{table}[t]
\scriptsize
    \caption{Preimage under-approximation refinement with ReLU split ($L_{\infty}$ attack). Results with Lagrangian optimisation are marked in grey background in columns \textit{w/ LagOpt}.}
    \label{tab:image_linf_relu}
    \centering
    \begin{tabular}{c|c | g | c |  g | c | g}
\toprule
\multirow{2}{*}         {\makecell{$L_{\infty}$ \textbf{attack} \\ (FNN $6 \times 100$)}}
        & \multicolumn{2}{c|}{\textbf{\#Poly}}
         & \multicolumn{2}{c|}{\textbf{Cov}}
         & \multicolumn{2}{c} {\textbf{Time(s)}} \\
\cmidrule{2-7}
& \textbf{w/o  } & \textbf{w/ LagOpt} & \textbf{w/o  } & \textbf{w/ LagOpt} & \textbf{w/o } & \textbf{w/ LagOpt} 
\\
        \midrule
          0.06
        & 2 & 2
        & 1.0
        & 1.0
        & 3.183
        & 3.237
        \\
         0.07
        & 247
        & 40
        & 0.752
        & 0.756
         & 130.746
         & 29.019
         \\
         0.08
        & 522
        & 290
        & 0.751
        & 0.751
        & 305.867
        & 218.455
        \\
         0.09
         & 733
         & 563
         & 0.165
         & 0.751
         & 507.116
         & 365.552
         \\
        \bottomrule
    \end{tabular}
\end{table}

\newcolumntype{g}{>{\columncolor{Gray}}c}
\begin{table}[t]
\scriptsize
    \caption{Preimage under-approximation refinement with ReLU split (patch attack). Results with Lagrangian optimisation are marked in grey background in columns \textit{w/ LagOpt}.}
    \label{tab:image_patch_relu}
    \centering
    \begin{tabular}{c|c | g | c |  g | c | g}
\toprule
\multirow{2}{*}         {\makecell{\textbf{Patch attack}\\ (FNN $6 \times 100$)}}
        & \multicolumn{2}{c|}{\textbf{\#Poly}}
         & \multicolumn{2}{c|}{\textbf{Cov}}
         & \multicolumn{2}{c} {\textbf{Time(s)}} \\
\cmidrule{2-7}
& \textbf{w/o  } & \textbf{w/ LagOpt} & \textbf{w/o  } & \textbf{w/ LagOpt} & \textbf{w/o } & \textbf{w/ LagOpt} 
\\
        \midrule
$3 \times 3$(center)
        & 1
        & 1
        & 1.0
        & 1.0
        & 2.611
        & 2.637
         \\
$4 \times 4$(center)
        & 678
        & 678
        & 0.382
        & 0.427
         & 455.988
         & 514.272
        \\
$7 \times 7$(corner)
         & 7 
         & 7
         & 0.842
         & 0.861
         & 6.065
         & 6.217
         \\
$8 \times 7$(corner)
        & 956
        & 954
        & 0.033
        & 0.214
        & 488.849
        & 676.666
        \\
        \bottomrule
    \end{tabular}
\end{table}

In this subsection, we evaluate the scalability of Algorithm \ref{alg:main} with ReLU splitting
by applying it to MNIST image classifiers. 
\bw{In particular, we consider input regions defined by bounded perturbations to a given MNIST image.}
\xy{Table \ref{tab:image_linf_relu} and \ref{tab:image_patch_relu} summarise the evaluation results for
two types of image \bw{perturbations commonly considered in the adversarial robustness literature} ($L_{\infty}$ and patch attack, respectively).
For $L_{\infty}$ attacks, bounded perturbation noise is applied to all image pixels.
The patch attack applies only to a smaller patch area \bw{of $n \times m$ pixels} but allows arbitrary perturbations covering the whole valid range $[0,1]$. The task is then to produce a DUP approximation of the \bw{subset of the} perturbation region that is guaranteed to be classified correctly.

For $L_{\infty}$ attack, we evaluate our method over perturbations of increasing size, from $0.06$ to $0.09$. It is worth noting that for this size of preimage, e.g., from $0.06$ to $0.07$, the \emph{volume} of the input region increases by tens of orders of magnitude due to the high dimensionality, \bw{making effective preimage approximation significantly more challenging}.
\bw{Table \ref{tab:image_linf_relu} shows that} 
our approach (Algorithm \ref{alg:main}) without Lagrangian optimisation (marked in columns \textit{w/o}) is able to generate a preimage under-approximation that achieves the targeted coverage of $0.75$ for $L_{\infty}$ noise up to 0.08.  The fact that the number of polytopes and computation time remain manageable is due to the effectiveness of ReLU splitting.
In Table \ref{tab:image_patch_relu}, for the patch attack, we observe that the number of polytopes and time required increase sharply when increasing the patch size for both the centre and corner area of the image, \bw{suggesting that the model is more sensitive to larger local perturbations}.
\bw{It is also interesting that our method can generate preimage approximations for larger patches in the corner as opposed to the centre of the image; }we hypothesize this is due to the greater influence of central pixels on the neural network output, and correspondingly a greater number of unstable neurons over the input perturbation space.

Table \ref{tab:image_patch_relu_over} shows the preimage refinement results for over-approximations in the context of patch attack. 
The results of our approach (Algorithm \ref{alg:main}) without Lagrangian optimisation are summarised in columns \textit{w/o}. 
As shown in the table, our refinement method can effectively tighten the over-approximation to the targeted coverage of 1.25 for different attack configurations.
For patch size $10 \times 10$ (centre) and $16 \times 15$ (corner), \bw{we found that} the perturbation region is a trivial over-approximation itself for the target coverage \bw{of 1.25}; thus, we demonstrate the results with a target coverage of 1.1 and 1.05.
Similarly to 
under-approximations, a patch attack in the centre with a smaller patch size requires more refinement iterations than the patch attack in the corner, demonstrating a greater influence of central pixels. 

\noindent\textbf{Effectiveness of Lagrangian Optimisation.} 
The results of evaluation of
our approach (Algorithm \ref{alg:main}) \bw{for under-approximation} with Lagrangian optimisation are shown in Table \ref{tab:image_linf_relu} and \ref{tab:image_patch_relu} (marked in columns \textit{w/ LagOpt} with grey background).
For $L_{\infty}$ attack, the refinement method with Lagrangian optimisation generates preimage approximations that achieve the
target coverage of 0.75 for all perturbation settings, including perturbation noise $0.09$ where 
the refinement \textit{without} Lagrangian optimisation fails (0.751 vs 0.165 in Table \ref{tab:image_linf_relu}). 
The new refinement method also leads to a significant reduction in the number of polytopes and computation cost. 
For the patch attack, the refinement method with Lagrangian optimisation effectively improves the preimage approximation precision for all configuration settings. 
Since the patch attack allows arbitrary perturbations covering the whole valid range $[0,1]$, it leads to a rapid increase in the number of unstable neurons and exhausts the iteration limit when increasing the patch size. Nonetheless, the resulting preimage approximation coverage obtained with Lagrangian optimisation shows better per-iteration precision improvement, \bw{while introducing marginal computation overhead compared to the previous method.}

\newcolumntype{g}{>{\columncolor{Gray}}c}
\begin{table}[t]
    \caption{Preimage over-approximation refinement with ReLU split (patch attack). Results with Lagrangian optimisation are marked in grey background in columns \textit{w/ LagOpt}.}
    \scriptsize
    \label{tab:image_patch_relu_over}
    \centering
    \begin{tabular}{c|c | g |  c | g |   c | g}
\toprule
\multirow{2}{*}         {\makecell{\textbf{Patch attack}\\ (FNN $6 \times 100$)}}
        & \multicolumn{2}{c|}{\textbf{\#Poly}}
         & \multicolumn{2}{c|}{\textbf{Cov}}
         & \multicolumn{2}{c} {\textbf{Time(s)}} \\
\cmidrule{2-7}
& \textbf{w/o} & \textbf{w/ LagOpt} & \textbf{w/o} & \textbf{w/ LagOpt} & \textbf{w/o} & \textbf{w/ LagOpt} 
\\
        \midrule
$10 \times 10$(center)
        & 387
        & 387
        & 1.099
        & 1.099
        & 261.826
        & 281.916
         \\
$11 \times 11$(center)
        & 317
        & 317
        & 1.249
        & 1.249
         & 192.954
         & 212.735
        \\
$16 \times 15$(corner)
        & 616
        & 616
        & 1.050
        & 1.050
         & 328.589
         & 350.092
        \\
$16 \times 16$(corner)
        & 285
        & 285
        & 1.249
        & 1.249
         & 165.250
         & 175.605
        \\
        \bottomrule
    \end{tabular}
\end{table}


Columns \textit{w/ LagOpt} in Table \ref{tab:image_patch_relu_over} summarises the over-approximation results with Lagrangian optimisation.
In this case, we introduce 
the Lagrange multipliers with the opposite signs to the under-approximation to guarantee the validity of the symbolic over-approximation. 
\bw{Intriguingly, in contrast to under-approximation, we find that the optimised $\beta$ parameters are almost always close to 0, meaning that the results are similar to not using Lagrangian optimisation.}
\bw{We hypothesize that, for over-approximations, the objective function is relatively flat in the vicinity of $0$, which makes the parameters difficult to optimise.}
}

\begin{table}[t]
    \caption{Comparison with a robustness verifier.}
    \label{tab:compare_verifier}
    \centering
    \scriptsize
    \begin{tabular}{c|cc|ccc}
        \toprule
        \multirow{2}{*}{\textbf{Task}}
        &\multicolumn{2}{c|}{\textbf{$\alpha,\beta$-CROWN}}
         & \multicolumn{3}{c}{\textbf{PREMAP}} \\
        \cmidrule{2-6}
        & Result
        & Time
        & Cov(\%)
        & \#Poly
        & Time\\
        \midrule
        Cartpole ($\dot{\theta} \in [-1.642,-1.546]$) 
         & yes
         & 3.349
        & 100.0
        & 1
         & 1.137
        \\        
         \midrule
        Cartpole ($\dot{\theta} \in [-1.642,0]$)
         &no
         & 6.927
        & 94.9
        & 2
         & 3.632
        \\
        \midrule
        MNIST ($L_{\infty}$ 0.026)
        & yes
         & 3.415
        & 100.0
        & 1
         & 2.649
         \\
         \midrule
        MNIST ($L_{\infty}$ 0.04)
        & unknown
         & 267.139
        & 100.0
        & 2
         & 3.019
         \\
        \bottomrule
    \end{tabular}
\end{table}
\noindent\textbf{Comparison with Robustness Verifiers.}
We now illustrate empirically the utility of preimage computation in robustness analysis compared to robustness verifiers. Table \ref{tab:compare_verifier} shows comparison results with $\alpha,\beta$-CROWN,  winner of the VNN competition~\citep{vnn22}.
We set the tasks according to the problem instances from VNN-COMP 2022 for local robustness verification (localised perturbation regions).
For Cartpole, $\alpha,\beta$-CROWN can provide a verification guarantee (yes/no or safe/unsafe) for both problem instances. However, in the case where the robustness property does not hold, our 
method 
explicitly generates a preimage 
under-approximation in the form of a disjoint polytope union (which guarantees the satisfaction of the output properties), and
covers $94.9\%$ of the exact preimage. 
For MNIST, while the smaller perturbation region is successfully verified, $\alpha,\beta$-CROWN with tightened intermediate bounds by MIP solvers returns unknown with a timeout of 300s for the larger region. 
In comparison,
our algorithm provides a concrete union of polytopes where the input is guaranteed to be correctly classified, which we find covers 100$\%$ of the input region (up to sampling error). Note also, 
as shown in Table \ref{tab:image_linf_relu}, our algorithm can produce non-trivial under-approximations for input regions far larger than $\alpha, \beta$-CROWN can verify.

\begin{table*}[ht]
\centering
\scriptsize
\caption{Performance comparison on preimage generation for reinforcement learning tasks. Under-approximation results are in the column labelled by Under-Approx, whereas over-approximation is in the column labelled
by Over-Approx.}
\label{tab:comparison_inv_premap}
\begin{tabular}{cc|ccc|ccc}
\toprule
\multirow{2}{*}{\textbf{Task}} & \multirow{2}{*}{\textbf{Method}} & \multicolumn{3}{c|}{\textbf{Under-Approx}} & \multicolumn{3}{c}{\textbf{Over-Approx}} \\
\cmidrule(lr){3-5} \cmidrule(lr){6-8}
 & & \textbf{\#Poly} & \textbf{Cov} & \textbf{Time(s)} & \textbf{\#Poly} & \textbf{Cov} & \textbf{Time(s)} \\
\midrule
\multirow{2}{*}{Cartpole} 
 & Invprop & 81 & 0.760 & 19.835 & 196 & \textbf{1.119} & 224.718 \\
 & PREMAP & \textbf{25} & \textbf{0.766} & \textbf{13.337} & \textbf{8} & 1.242 & \textbf{5.778}  \\
\midrule
\multirow{2}{*}{Lunarlander} 
 & Invprop & 256 & 0.756 & 325.206 & 221 & 1.378 & 358.386 \\
 & PREMAP & \textbf{12} & \textbf{0.759} & \textbf{8.091} & \textbf{90} & \textbf{1.247} & \textbf{50.154} \\
\midrule
\multirow{2}{*}{Dubinsrejoin} 
 & Invprop & 295 & 0.441 & 360.69 & 205 & 1.876 & 361.16 \\
 & PREMAP & \textbf{20} & \textbf{0.765} & \textbf{12.683} & \textbf{13} & \textbf{1.232} & \textbf{8.822} \\
\bottomrule
\end{tabular}
\end{table*}

\rev{
\subsubsection{Comparison with Output Constraint-Enhanced Approach}
The original Invprop method focused on preimage over-approximation in low-dimensional tasks.
In our experiments, we compare our method, PREMAP, with the updated version of Invprop, which is now integrated into the auto-LiRPA framework, enabling both under- and over-approximation, 
which we still refer to as Invprop.
Table \ref{tab:comparison_inv_premap} presents a comparative evaluation of preimage approximation quality on neural network controllers between Invprop and our proposed method, PREMAP.
Note that, for Invprop, output constraints are leveraged to compute tighter intermediate bounds, thereby improving the preimage approximation quality.

The evaluation results show that PREMAP demonstrates superior preimage approximation precision and runtime efficiency across all evaluated neural network controllers. 
The results highlight the importance of effective preimage refinement strategies. Specifically, PREMAP's effectiveness relies on the smart subregion selection and greedy splitting, which leads to better preimage quality improvement than solely applying intermediate bound tightening on subregions.
A promising direction for future work is to explore a hybrid approach that integrates PREMAP's domain selection and splitting method with the bound-tightening capabilities of Invprop. 
Investigating whether this synergy can consistently outperform each method individually could offer valuable insights to advance the effectiveness of preimage analysis techniques.}


\rev{\subsubsection{Evaluation of Input vs ReLU Split}\label{sec:inputVSrelu}
Table~\ref{tab:compare_split} presents  performance comparison between input and ReLU splitting across low-dimensional control tasks.
Both approaches are able to reach the target coverage; however, input splitting generally achieves higher per-iteration precision refinement. It requires fewer iterations, and as a result, fewer preimage polytopes compared to ReLU splitting. 
This advantage is mainly because one input split of the original region into smaller ones can simultaneously stabilise multiple (unstable) ReLU neurons, whereas ReLU splitting addresses one unstable neuron per iteration. 

Notably, the proposed greedy input splitting method further selects the input feature that yields the greatest improvement in approximation precision per split, though at the cost of increased computational overhead in certain cases (e.g., Lunarlander).
For the same reason, for high-dimensional tasks such as the MNIST classification task with 784 input features, the same greedy method requires parallel computation across hundreds of processes. Each instance involves large intermediate tensors, which can easily exceed the GPU memory capacity, and the greedy selection for the optimal split further increases computation time.
}

\begin{table}[t]
    \caption{Preimage under-approximation refinement with Input vs ReLU split.}
    \label{tab:compare_split}
    \centering
    \scriptsize
    \begin{tabular}{cc|cc|cc|cc}
        \toprule
        \multirow{2}{*}{\textbf{Task}}
        & \multirow{2}{*}{\textbf{Config}}
        &\multicolumn{2}{c|}{\textbf{\#Poly}}
         & \multicolumn{2}{c|}{\textbf{Cov(\%)}}
         & \multicolumn{2}{c}{\textbf{Time(s)}} \\
        \cmidrule{3-8}
        &
        & Input
        & ReLU
        & Input
        & ReLU
        & Input
        & ReLU\\
         \midrule
        Cartpole
        & $\dot{\theta} \in [0,0.1]$
         & \textbf{3}
         & 106
        & \textbf{78.0}
        & 75.3
         & \textbf{4.233}
         & 12.613
        \\
        \midrule
        Lunarlander
        & $\dot{v} \in [-0.1,0]$
         & \textbf{36}
         & 134
        & \textbf{75.1}
        & \textbf{75.1}
         & 35.055
         & \textbf{14.239}
        \\
        \midrule
        Dubinsrejoin
        & $x_v\in [-0.1,0]$
         & \textbf{2}
         & 226
         & \textbf{94.5}
        & 75.0
        & \textbf{3.638}
         & 34.895
        \\
        \bottomrule
    \end{tabular}
\end{table}
\subsubsection{Quantitative Verification}

We now demonstrate the application of our preimage under-approximation to quantitative verification of the property $(\inset, \outset, \proportion)$; that is, we aim to check whether $\nn(\inpoint) \in \outset$ for at least proportion $\proportion$ of input values $\inpoint \in \inset$. 
Table \ref{tab:quant} summarises the quantitative verification results, which
leverage the disjointness of our under-approximation, such that we can compute the total volume covered by computing the volume of each individual polytope.


\textit{Vertical Collision Avoidance System.}
In this example, we consider the VCAS system
and a scenario where the two aircraft have negative relative altitude from intruder to ownship ($h \in [-8000, 0]$), the ownship aircraft has a positive climbing rate $\dot{h_A} \in[0,100]$ and the intruder has a stable negative climbing rate $\dot{h_B}=-30$, and time to the loss of horizontal separation is $t \in [0,40]$, which defines the input region $\inset$.
For this scenario, the correct advisory is ``Clear Of Conflict'' (COC).
We apply Algorithm \ref{alg:verify} to verify the quantitative property where $\outset=\{\outpoint \in \mathbb{R}^{9} | \bigwedge_{i = 2}^{9} y_1 - y_i \geq 0\}$ and the 
proportion $p=0.9$, with an iteration limit of 1000. 
The quantitative proportion reached by the generated under-approximation is 90.8\%, which verifies the quantitative property in 5.620s.

\begin{table}[t]
\caption{Quantitative verification results with preimage under-approximation.}\label{tab:quant}
\centering
\scriptsize
\begin{tabular}{c|c|c|c|c} 
\toprule
\textbf{Task} & \textbf{Property} & \textbf{\#Poly}  & \textbf{Time(s)} & \textbf{QuantProp(\%)}\\ 
\midrule
VCAS & $\outset=\{\outpoint \in \mathbb{R}^{9} | \bigwedge_{i = 2}^{9} y_1 - y_i \geq 0\}$ & 6 & 5.620 & 90.8\\ 
\midrule
Cartpole & 
$\outset=\{y|y_1-y_2 \ge 0\}$  & 11 & 12.1 & 90.0\\ 
\midrule
Lunarlander & 
$\outset=\{y\in \mathbb{R}^4| \wedge_{i \in \{1,3,4\}}  y_2 \ge y_i\}$  & 120 & 429.480 & 90.0\\ 
\bottomrule
\end{tabular}
\end{table}

\xy{\textit{Cartpole.}
In the Cartpole problem, the objective is to balance the pole attached to a cart by pushing the cart either left or right.
We consider a scenario where the cart position is to the right of the centre ($x \in [0,1]$), the cart is moving right ($\dot{x} \in [0, 0.5]$), the pole is slightly tilted to the right ($\theta \in [0,0.1]$) and pole is moving to the left ($\dot{\theta} \in [-0.2,0]$).
To balance the pole, the neural network controller needs to determine  ``pushing left''.
We apply Algorithm \ref{alg:verify} to verify the quantitative property, where $\outset=\{y|y_1-y_2 \ge 0\}$ and the proportion $p=0.9$, with an iteration limit of 1000. The 
under-approximation algorithm  
takes 12.1s to reach the target proportion 90.0\%.}

\xy{\textit{Lunarlander.}
In the Lunarlander task, the objective of the neural networks controller is to achieve a safe landing of the lander. Consider a scenario where the lander is slightly to the left of the centre of the landing pad ($x \in [-1,0]$), the lander is above the landing pad sufficient for descent correction ($h \in [0,1]$), and it is moving to the right ($\dot{x} \in [1,2]$) but descending rapidly ($\dot{h} \in [-2,-1]$). To avoid a hard landing, the neural network controller needs to reduce the descent speed by taking the action ``fire main engine''. 
We formulate the quantitative property for this task, where $\outset=\{y\in \mathbb{R}^4| \wedge_{i \in \{1,3,4\}}  y_2 \ge y_i\}$ and the proportion  $p=0.9$.
To compute preimage under-approximation for this more complex task 
takes 429.480s to reach the target proportion 90.0\%.}

\section{Conclusion}
\label{sec:conclusion}
\xy{We present PREMAP, an efficient and unifying algorithm for preimage approximation of neural networks. 
Our \emph{anytime} method stems from the observation that linear relaxation can be used to efficiently produce approximations, in conjunction with custom-designed strategies for iteratively decomposing the problem to rapidly improve the approximation quality. 
We formulate the preimage approximation in each refinement iteration as an optimisation problem and propose a differentiable objective to derive tighter preimages via optimising over convex bounding parameters and Lagrange multipliers.} 
Unlike previous 
approaches, our method
is designed for, and scales to, both low and high-dimensional problems. Experimental evaluation on a range of benchmark tasks shows significant advantages in runtime efficiency and scalability, and the utility of our method for important applications in quantitative verification and robustness analysis.

\acks{This project received funding from the ERC under the European Union’s Horizon 2020 research and innovation programme (FUN2MODEL, grant agreement No.~834115) and ELSA: European Lighthouse on Secure and Safe AI project (grant agreement
No. 101070617 under UK guarantee). This
work was done in part while Benjie Wang was visiting the Simons Institute for the Theory of Computing.}

\newpage

\appendix
\section{Experiment Setup}
\label{app:exp_setup}
In this section, we present the detailed configuration of neural networks in the benchmark tasks.
\subsection{Vehicle Parking.} 
For the vehicle parking task, we train a neural network with one hidden layer of 20 neurons, which is computationally feasible for exact preimage computation for comparison.
We consider computing the preimage approximation with input region corresponding to the entire input space
$\indomain = \{\inpoint \in \mathbb{R}^2 | \inpoint \in [0,2]^2 \}$, and output sets $\outset_{k}$, which correspond to the neural network outputting label $k$:
$
\outset_k = \{\outpoint \in \mathbb{R}^4~| \bigwedge_{i \in \{1, 2, 3, 4\}\setminus k} \outpoint_k - \outpoint_i \geq 0\},  \quad k \in \{1, 2, 3, 4\}$.

\subsection{Aircraft Collision Avoidance}
The aircraft collision avoidance (VCAS) system  \citep{Julian19nncontrol}
is used
to provide advisory for collision avoidance between the ownship aircraft and the intruder.
VCAS uses four input features 
$(h, \dot{h_A},\dot{h_B}, t)$ 
representing the relative altitude of the aircrafts,
vertical climbing rates of the ownship and intruder aircrafts, respectively, and time to the loss of horizontal separation. 
VCAS is implemented by nine feed-forward neural networks built with a hidden layer of 21 neurons.
In our experiment, we use the 
following
input 
region
for the ownship and intruder aircraft as in \cite{Matoba20Exact}:  
$h \in [-8000, 8000]$, $\dot{h_A} \in [-100, 100]$, $\dot{h_B}=30$, and $t\in [0, 40]$.
In the training, the input configurations are normalized into a range of $[-1, 1]$.
We consider the output property $\outset = \{y\in \mathbb{R}^9~|\wedge_{i \in [2,9]}~y_1 \ge y_i\}$ and 
generate the preimage approximation for
the VCAS neural networks.

\subsection{Neural Network Controllers}
\subsubsection{Cartpole} The cartpole control problem considers balancing a pole atop a cart by controlling the movement of the cart.
The neural network controller has two hidden layers with 64 neurons, and uses four input variables 
representing the position and velocity of the cart, the angle and angular velocity of the pole. The controller outputs are pushing the cart left or right.
In the experiments, we set the following input region for the Cartpole task: (1) cart position $[-1,1]$, (2) cart velocity $[0,2]$, (3) angle of the pole $[-0.2,0]$, and (4) angular velocity of the pole $[-2,0]$ (with varied feature length in the evaluation). We consider the output property for the action \textit{pushing left}.

\subsubsection{Lunarlander} The Lunarlander problem considers the task of correct landing of a moon lander on a landing pad. 
 The neural network for Lunarlander has two hidden layers with 64 neurons, and eight input features addressing the lander's coordinate, orientation, velocities, and ground contact indicators. The outputs represent four actions.
 For the Lunarlander task, we set the input region as: (1) horizontal and vertical position $[-1, 0]$ $\times$ $[0,1]$, (2) horizontal and vertical velocity $[0,2]$ $\times$ $[-2,0]$ (with varied feature length for evaluation), (3) angle and angular velocity $[-1,0] \times [-0.1,0.1]$, (4) left and right leg contact $[0.9,1]^2$. We consider the output specification for the action ``fire main engine'', i.e., $\{y \in \mathbb{R}^4~| \wedge_{i \in \{1,3,4\}} y_2 \ge y_i\}$.
\subsubsection{Dubinsrejoin.} The Dubinsrejoin problem considers guiding a wingman craft to a certain radius around a lead aircraft. 
The neural network controller has two hidden layers with 256 neurons.
The input space of the neural network controller is eight-dimensional, with the input variables capturing the position, heading, velocity of the lead and wingman crafts, respectively.
The outputs are also eight dimensional representing controlling actions of the wingman. 
Note that the eight neural network outputs are processed further as tuples of actuators (rudder, throttle) for controlling the wingman where each actuator has 4 options.
The control action tuple is decided by taking the action with the maximum output value among the first four network outputs (the first actuator options) and the action with the maximum value among the second four network outputs (the second actuator options).
In the experiments, we set the following input region: 
(1) horizontal and vertical position $[-0.2, 0]\times [0, 0.5]$, (2) heading and velocity $[-1,0]\times [0, 0.2]$ for the lead aircraft,
and (3) horizontal and vertical position $[0.4, 0.6]\times [-0.3,0.3]$ (with varied feature length for evaluation), (4) heading and velocity $[0.2,0.5] \times [-0.5, 0.5]$ for the wingman aircraft.
We consider the output property that  both actuators (rudder, throttle) take the first option, i.e., $\{y\in \mathbb{R}^8~|\wedge_{i \in \{2, 3, 4\}} ~ y_1 \ge y_i \bigwedge \wedge_{i \in \{6, 7, 8\}}~y_5 \ge y_i\}$. 

\subsection{MNIST Classification}
We use the trained neural network from VNN-COMP 2022 \citep{vnn22} for digit image classification. 
The neural network has six layers with a hidden neuron size of 100 for each hidden layer.
We consider
two types of image attacks: $l_{\infty}$ and patch attack.
For $L_{\infty}$ attack, a perturbation is applied to all pixels of the image.
For the patch attack, it applies arbitrary perturbations to the patch area, i.e., the perturbation noise covers the whole valid range $[0,1]$, for which we set the patch area at the centre and (upper-left) corner of the image with different sizes.

\section{Proofs}\label{app:proofs}
We present the propositions and proofs on guaranteed polytope volume improvement with each refinement iteration, noting that these propositions are valid without stochastic optimisation. 
Subsequently, we provide proofs for
Propositions \ref{prop:sound}
and \ref{prop:complete}. 
\begin{restatable}{proposition}{propPriority} \textbf{\label{prop:volume_guarantee}}
    Given any subregion $\indomain_{sub}$ with polytope under-approximation $\polytope_{\indomain_{sub}}(\outset)$, and its children $\indomain_{sub}^l, \indomain_{sub}^u$ with polytope under-approximations $\polytope_{\indomain_{sub}^l}(\outset), \polytope_{\indomain_{sub}^u}(\outset)$ respectively, it holds that:
    \begin{equation} \label{eqn:no_fragmentation}
        \polytope_{\indomain_{sub}^l}(\outset) \cup \polytope_{\indomain_{sub}^u}(\outset) \supseteq \polytope_{\indomain_{sub}}(\outset)
    \end{equation}
\end{restatable}

\begin{proof}
    We define $\polytope_{\indomain_{sub}}(\outset)|_{l}, \polytope_{\indomain_{sub}}(\outset)|_{r}$ to be the restrictions of $\polytope_{\indomain_{sub}}(\outset)$ to $\indomain_{sub}^l$ and $\indomain_{sub}^r$ respectively, that is:

    \begin{equation}
        \polytope_{\indomain_{sub}}(\outset)|_{l} = \{\inpoint \in \mathbb{R}^{\indim}| \bigwedge_{i =1}^{\specnum} (\underline{\spec_{i}}(\inpoint)  \geq 0 )\wedge (\inpoint \in \indomain_{sub}^l) \}
    \end{equation}
    \begin{equation}
        \polytope_{\indomain_{sub}}(\outset)|_{r} = \{\inpoint \in \mathbb{R}^{\indim}| \bigwedge_{i =1}^{\specnum} (\underline{\spec_{i}}(\inpoint)  \geq 0 )\wedge (\inpoint \in \indomain_{sub}^r) \}
    \end{equation}
    where we have replaced the constraint $\inpoint \in \indomain_{sub}$ with $\inpoint \in \indomain_{sub}^l$ (resp. $\inpoint \in \indomain_{sub}^r$), and $\underline{\spec_{i}}(\inpoint)$ is the LiRPA lower bound for the $i^{\textnormal{th}}$ specification on the input region $\indomain_{sub}$.

    On the other hand, we also have:
    \begin{equation}
        \polytope_{\indomain_{sub}^l}(\outset) = \{\inpoint \in \mathbb{R}^{\indim}| \bigwedge_{i =1}^{\specnum} (\underline{\spec_{l, i}}(\inpoint)  \geq 0 )\wedge (\inpoint \in \indomain_{sub}^l) \}
    \end{equation}
    \begin{equation}
        \polytope_{\indomain_{sub}^r}(\outset) = \{\inpoint \in \mathbb{R}^{\indim}| \bigwedge_{i =1}^{\specnum} (\underline{\spec_{r, i}}(\inpoint)  \geq 0 )\wedge (\inpoint \in \indomain_{sub}^r) \}
    \end{equation}
    where $\underline{\spec_{l, i}}(\inpoint)$ (resp. $\underline{\spec_{r, i}}(\inpoint)$) is the LiRPA lower bound for the $i^{\textnormal{th}}$ specification on the input region $\indomain_{sub}^l$ (resp. $\indomain_{sub}^r$). Now, 
    it is sufficient to show that $\polytope_{\indomain_{sub}^l}(\outset) \supseteq \polytope_{\indomain_{sub}}(\outset)|_{l}$ and $\polytope_{\indomain_{sub}^r}(\outset) \supseteq \polytope_{\indomain_{sub}}(\outset)|_{r}$ to prove Equation \ref{eqn:no_fragmentation}. We will now show that $\polytope_{\indomain_{sub}^l}(\outset) \supseteq \polytope_{\indomain_{sub}}(\outset)|_{l}$ (the proof for $\polytope_{\indomain_{sub}^r}(\outset) \supseteq \polytope_{\indomain_{sub}}(\outset)|_{r}$ is entirely similar).

    Before proving this result in full, we outline the approach and a sketch proof. It suffices to prove (for all $i$) that $\underline{\spec_{l, i}}(\inpoint)$ is a tighter bound than $\underline{\spec_{i}}(\inpoint)$ on $\indomain_{sub}^l$. That is, to show that $\underline{\spec_{l, i}}(\inpoint) \geq \underline{\spec_{i}}(\inpoint)$ for inputs $\inpoint$ in $\indomain_{sub}^l$, as then $\underline{\spec_{i}}(\inpoint)  \geq 0 \implies \underline{\spec_{l, i}}(\inpoint)  \geq 0$ for inputs $\inpoint$ in $\indomain_{sub}^l$, and so $\polytope_{\indomain_{sub}^l}(\outset) \supseteq \polytope_{\indomain_{sub}}(\outset)|_{l}$. The bound $\underline{\spec_{l, i}}(\inpoint)$ is tighter than $\underline{\spec_{i}}(\inpoint)$ because the input region for LiRPA is smaller for $\underline{\spec_{l, i}}(\inpoint)$, leading to tighter concrete neuron bounds, and thus tighter bound propagation through each layer of the neural network $\spec_i$.
    We present the formal proof of greater bound tightness for input and ReLU splitting in the following.

   \textbf{Input split:}
    We show $\underline{\spec_{l, i}}(\inpoint) \geq \underline{\spec_{i}}(\inpoint)$ for all $\inpoint \in \indomain_{sub}^l$ by induction (dropping the index $i$ in the following as it is not important). Recall that LiRPA generates symbolic upper and lower bounds on the pre-activation values of each layer in terms of the input (i.e. treating that layer as output), which can then be converted into concrete bounds.
    \begin{equation} \label{eqn:layer_bound}
        \lowerweight^{(j)} \inpoint + \lowerbias^{(j)} \leq \preact^{(j)}(\inpoint) \leq \upperweight^{(j)} \inpoint + \upperbias^{(j)}
    \end{equation}
    \begin{equation}  \label{eqn:layer_bound_left}
        \lowerweight^{(l, j)} \inpoint + \lowerbias^{(l, j)} \leq \preact^{(j)}(\inpoint) \leq \upperweight^{(l, j)} \inpoint + \upperbias^{(l, j)}
    \end{equation}
    where $\preact^{(j)}(\inpoint)$ are the pre-activation values for the $j^{\textnormal{th}}$ layer of the network $\spec_i$, and $\lowerweight^{(j)}, \lowerbias^{(j)}, \upperweight^{(j)}, \upperbias^{(j)}$ (resp. $\lowerweight^{(l, j)}, \lowerbias^{(l, j)}, \upperweight^{(l, j)}, \upperbias^{(l, j)}$) are the linear bound coefficients, for input regions $\indomain_{sub}$ (resp. $\indomain_{sub}^l$). 

    \textit{Inductive Hypothesis} For all layers $j = 1, ..., \numlayers$ in the network, and for all $\inpoint \in \indomain_{sub}^l$, it holds that:
    \begin{equation} 
    \lowerweight^{(j)} \inpoint + \lowerbias^{(j)} \leq
    \lowerweight^{(l, j)} \inpoint + \lowerbias^{(l, j)}  \leq \upperweight^{(l, j)} \inpoint + \upperbias^{(l, j)} \leq \upperweight^{(j)} \inpoint + \upperbias^{(j)}
    \end{equation}

    \textit{Base Case} For the input layer, we have the trivial bounds $\textbf{I}\inpoint \leq \inpoint \leq \textbf{I}\inpoint$ for both regions.

    \textit{Inductive Step} Suppose that the inductive hypothesis is true for layer $j - 1 < \numlayers$. Using the symbolic bounds in Equations \ref{eqn:layer_bound}, \ref{eqn:layer_bound_left}, we can derive concrete bounds $\concretelower^{(j - 1)} \leq \preact^{(j - 1)}(\inpoint) \leq \concreteupper^{(j - 1)}$ and $\concretelower^{(l, j - 1)} \leq \preact^{(l, j - 1)}(\inpoint) \leq \concreteupper^{(l, j - 1)}$ on the values of the pre-activation layer. By the inductive hypothesis, the bounds for region $\indomain_{sub}^l$ will be tighter, i.e. $\concretelower^{(j - 1)} \leq \concretelower^{(l, j - 1)} \leq \concreteupper^{(l, j - 1)} \leq \concreteupper^{(j - 1)}$. Now, consider the backward bounding procedure for layer $j$ as output. We begin by encoding the linear layer from post-activation layer $j - 1$ to pre-activation layer $j$ as:
    \begin{equation} \label{eqn:apx_linear}
        \weight^{(j)} \postact^{(j - 1)}(\inpoint) + \bias^{(j)} \leq \preact^{(j)}(\inpoint) \leq \weight^{(j)} \postact^{(j - 1)}(\inpoint) + \bias^{(j)}
    \end{equation}
    Then, we bound $\postact^{(j - 1)}(\inpoint)$ in terms of $\preact^{(j - 1)}(\inpoint)$ using linear relaxation. Consider the three cases in Figure \ref{fig:linear_relaxation} (reproduced from main paper), where we have a bound $\underline{c} \preact^{(j - 1)}_k(\inpoint) + \underline{d} \leq \postact^{(j - 1)}_k(\inpoint) \leq \overline{c} \preact^{(j - 1)}_k(\inpoint) + \overline{d}$, for some scalars $\underline{c}, \underline{d}, \overline{c}, \overline{d}$. If the concrete bounds (horizontal axis) are tightened, then an unstable neuron may become inactive or active, but not vice versa. It can thus be seen that the new linear upper and lower bounds on $\preact^{(j - 1)}_k(\inpoint)$ will also be tighter. 

    Substituting the linear relaxation bounds in Equation \ref{eqn:apx_linear} as in \cite{xu2021fast}, we obtain bounds of the form
    \begin{equation} 
        \lowerweight^{(j)}_j \preact^{(j - 1)}(\inpoint) + \lowerbias^{(j)}_j \leq \preact^{(j)}(\inpoint) \leq \upperweight^{(j)}_j \preact^{(j - 1)}(\inpoint) + \upperbias^{(j)}_j
    \end{equation}
    \begin{equation} 
        \lowerweight^{(l, j)}_j \preact^{(j - 1)}(\inpoint) + \lowerbias^{(l, j)}_j \leq \preact^{(j)}(\inpoint) \leq \upperweight^{(l, j)}_j \preact^{(j - 1)}(\inpoint) + \upperbias^{(l, j)}_j
    \end{equation}
    such that $\lowerweight^{(j)}_j \preact^{(j - 1)}(\inpoint) + \lowerbias^{(j)}_j \leq \lowerweight^{(l, j)}_j \preact^{(j - 1)}(\inpoint) + \lowerbias^{(l, j)}_j \leq \upperweight^{(l, j)}_j \preact^{(j - 1)}(\inpoint) + \upperbias^{(l, j)}_j \leq \upperweight^{(j)}_j \preact^{(j - 1)}(\inpoint) + \upperbias^{(j)}_j$ for all $\concretelower^{(l, j - 1)} \leq \preact^{(j - 1)}(\inpoint) \leq \concretelower^{(l, j - 1)}$, by the fact that the concrete bounds are tighter for $\indomain_{sub}^{l}$. 
    
    Finally, substituting the bounds in Equations \ref{eqn:layer_bound} and \ref{eqn:layer_bound_left} (for $\preact^{(j - 1)}$), and using the tightness result in the inductive hypothesis for $j - 1$, we obtain linear bounds for $\preact^{(j)}(\inpoint)$ in terms of of the input $\inpoint$, such that the inductive hypothesis for $j$ holds.

\textbf{ReLU split:}
We use $\indomain_{sub}^l$ and $\indomain_{sub}^r$ to denote the input subregions when fixing unstable ReLU neuron $z^{(j-1)}_k=\preact^{(j-1)}_k(\inpoint)$, i.e.,
$\indomain_{sub}^l=\{\inpoint~|~\preact^{(j-1)}_k(\inpoint) \ge 0\}$
and $\indomain_{sub}^r=\{\inpoint~|~\preact^{(j-1)}_k(\inpoint) < 0\}$.

In the following, we prove that $\underline{\spec_{l, i}}(\inpoint) \geq \underline{\spec_{i}}(\inpoint)$ for all $\inpoint \in \indomain_{sub}^l$.
Assume we fix one unstable ReLU neuron of layer $j-1$, then for all layers $1 \le m \le j-1$, for all   $\inpoint \in \indomain_{sub}^l$, it holds that:
\begin{equation} 
\lowerweight^{(m)} \inpoint + \lowerbias^{(m)} \leq
\lowerweight^{(l, m)} \inpoint + \lowerbias^{(l, m)}  \leq \upperweight^{(l, m)} \inpoint + \upperbias^{(l, m)} \leq \upperweight^{(m)} \inpoint + \upperbias^{(m)}
\end{equation}
where $\lowerweight^{(m)}=\lowerweight^{(l, m)}$, $\lowerbias^{(m)}=\lowerbias^{(l, m)}$ and same for the upper bounding parameters.

Now consider the bounding procedure for layer $j$.
The linear layer from post-activation layer $j - 1$ to pre-activation layer $j$ can be encoded as:
    \begin{equation}\label{eq:apx_linear_relu}
        \weight^{(j)} \postact^{(j - 1)}(\inpoint) + \bias^{(j)} \leq \preact^{(j)}(\inpoint) \leq \weight^{(j)} \postact^{(j - 1)}(\inpoint) + \bias^{(j)}
    \end{equation}
Consider the post activation function $\postact^{(j - 1)}(\inpoint)$ of the unstable neuron $z^{(j-1)}_k$, before splitting we have 
$\underline{c} \preact^{(j - 1)}_k(\inpoint) + \underline{d} \leq \postact^{(j - 1)}_k(\inpoint) \leq \overline{c} \preact^{(j - 1)}_k(\inpoint) + \overline{d}$, for some scalars $\underline{c}, \underline{d}, \overline{c}, \overline{d}$.
After splitting, we now have $\postact^{(j - 1)}_k(\inpoint) = \preact^{(j - 1)}_k(\inpoint)$ for $\indomain_{sub}^l$ where $\underline{c}=\overline{c}=1$, $\underline{d}=\overline{d}=0$, since the unstable neuron is fixed to be active.
By substituting the linear relaxation bounds before and after splitting in Equation \ref{eqn:apx_linear}, we obtain the bounding functions with regard to $\preact^{(j - 1)}(\inpoint)$ in the following form:
    \begin{equation} 
        \lowerweight^{(j)}_j \preact^{(j - 1)}(\inpoint) + \lowerbias^{(j)}_j \leq \preact^{(j)}(\inpoint) \leq \upperweight^{(j)}_j \preact^{(j - 1)}(\inpoint) + \upperbias^{(j)}_j
    \end{equation}
    \begin{equation} 
        \lowerweight^{(l, j)}_j \preact^{(j - 1)}(\inpoint) + \lowerbias^{(l, j)}_j \leq \preact^{(j)}(\inpoint) \leq \upperweight^{(l, j)}_j \preact^{(j - 1)}(\inpoint) + \upperbias^{(l, j)}_j
    \end{equation}
By the fact the relaxation is fixed to be exact for $\postact^{(j - 1)}_k(\inpoint)$, it holds that
$\lowerweight^{(j)}_j \preact^{(j - 1)}(\inpoint) + \lowerbias^{(j)}_j \leq \lowerweight^{(l, j)}_j \preact^{(j - 1)}(\inpoint) + \lowerbias^{(l, j)}_j \leq \upperweight^{(l, j)}_j \preact^{(j - 1)}(\inpoint) + \upperbias^{(l, j)}_j \leq \upperweight^{(j)}_j \preact^{(j - 1)}(\inpoint) + \upperbias^{(j)}_j$ for $\indomain_{sub}^{l}$.

Finally, for the bound propagation procedure of layer $L$, substituting the tightened bounding for $\preact^{(j-1)}(\inpoint)$, we obtain that $\underline{\spec_{l, i}}(\inpoint) = \lowerweight^{(l,L)} \inpoint + \lowerbias^{(l,L)}\geq \lowerweight^{(L)} \inpoint + \lowerbias^{(L)}=\underline{\spec_{i}}(\inpoint)$.
\end{proof}
\begin{restatable}{corollary}{corVolume}
    In each refinement iteration,
    the volume of the polytope under-approximation $\polytopeset_{Dom}$ does not decrease.
\end{restatable}

\begin{proof}
    In each iteration of Algorithm \ref{alg:main}, we replace the polytope $\polytope_{\indomain_{sub}}(\outset)$ in a leaf subregion with two polytopes $\polytope_{\indomain_{sub}^l}(\outset), \polytope_{\indomain_{sub}^r}(\outset)$ in the DUP under-approximation. By Proposition \ref{prop:volume_guarantee}, the total volume of the two new polytopes is at least that of the removed polytope. Thus the volume of the DUP approximation does not decrease.
    
    Similarly, for ReLU splitting, we replace the polytope $\polytope_{\indomain_{sub}}(\outset)$ in a leaf subregion with two polytopes $\polytope_{\indomain_{sub}^l}(\outset), \polytope_{\indomain_{sub}^r}(\outset)$ where the relaxed bounding functions for one unstable neuron are replaced with exact linear functions, i.e., $\underline{c} \preact^{(i)}_j(\inpoint) + \underline{d} \leq \postact^{(i)}_j(\inpoint) \leq \overline{c} \preact^{(i)}_j(\inpoint) + \overline{d}$ is replaced with the exact linear function $\postact^{(i)}_j(\inpoint) = \preact^{(i)}_j(\inpoint)$ and $\postact^{(i)}_j(\inpoint) = 0$, respectively, as shown in Figure \ref{fig:linear_relaxation} (from unstable to stable).
     By Proposition \ref{prop:volume_guarantee}, the total volume of the two new polytopes is at least that of the removed polytope. Thus the volume of the DUP approximation does not decrease.
\end{proof}

\propSound*

\begin{proof}
    Algorithm \ref{alg:verify} outputs True only if, at some iteration, we have that the exact volume $\volume(\polytopeset) \geq \proportion \times \volume(\inset)$. Since $\polytopeset$ is an under-approximation to the restricted preimage $\preimage_{\inset}(\outset)$, we have that $\frac{\volume(\preimage_{\inset}(\outset))}{\volume(\inset)} \geq \frac{\volume(\polytopeset)}{\volume(\inset)} \geq \proportion$, i.e. the quantitative property $(\inset, \outset, \proportion)$ holds.
\end{proof}

\propComplete*
\begin{proof}
The proof for the soundness of Algorithm \ref{alg:verify} with ReLU splitting is similar to input splitting.
Regarding the completeness,
when all unstable neurons are fixed with one activation status,
for each subregion $\indomain_{sub}$, we have $\underline{\spec_{i}}(\inpoint)=\spec_{i}(\inpoint)$.
It then holds that for any $\indomain_{sub} \subset \indomain$ where $\bigcup \indomain_{sub}=\indomain$, $(\underline{\spec_i}(\inpoint)  \geq 0 )\wedge \inpoint \in \indomain_{sub} \iff  (\spec_i(\inpoint)  \geq 0 )\wedge \inpoint \in \indomain_{sub}$, i.e., the polytope is the exact preimage.
Hence, when all unstable neurons are fixed to an activation status, we have $\polytopeset=\preimage_{\inset}(\outset)$.
Algorithm \ref{alg:verify} returns False only if 
the volume of the exact preimage 
$\frac{\volume(\preimage_{\inset}(\outset))}{\volume(\inset)} = \frac{\volume(\polytopeset)}{\volume(\inset)} < \proportion$.

\end{proof}

\vskip 0.2in
\newpage
\bibliography{references}

\end{document}